\documentclass[twoside]{article}

\usepackage[accepted]{aistats2026}
\usepackage[utf8]{inputenc} 
\usepackage[T1]{fontenc}    
\usepackage{hyperref}       
\usepackage{url}            
\usepackage{booktabs}       
\usepackage{amsfonts}       
\usepackage{nicefrac}       
\usepackage{microtype}      
\usepackage{xcolor}         
\usepackage[round]{natbib}
\usepackage{wrapfig}
\usepackage{subcaption}

\usepackage{amsfonts,amsmath,amssymb}
\usepackage{xcolor}
\usepackage{hyperref}
\usepackage{srcltx}
\usepackage{etoolbox}
\usepackage{nameref}
\usepackage{comment}
\usepackage{mathrsfs}
\usepackage{enumerate}
\usepackage[shortlabels]{enumitem}
\usepackage{amsfonts} 
\usepackage{nicefrac} 
\usepackage{mathtools}
\usepackage{dsfont,mathrsfs}
\usepackage{cleveref}
\usepackage{tikz-cd} 
\usepackage{amsthm}

\usepackage{multirow}
\usepackage{colortbl}
\definecolor{bgcolor}{rgb}{0.8,1,1}
\definecolor{bgcolor2}{rgb}{0.8,1,0.8}
\definecolor{niceblue}{rgb}{0.0,0.19,0.56}
\usepackage{threeparttable}

\usepackage{tcolorbox}
\usepackage{pifont}
\definecolor{mydarkgreen}{RGB}{39,130,67}
\definecolor{mydarkred}{RGB}{192,47,25}

\usepackage{xcolor}
\definecolor{dmorange500}{HTML}{FF5F19}
\definecolor{dmblue300}{HTML}{2267EB}
\definecolor{dmred300}{HTML}{FF617B}
\hypersetup{
	colorlinks,
	citecolor=dmblue300
}

\usepackage{xargs}
\usepackage{multirow}

\usepackage{aliascnt}
\usepackage{autonum}
\usepackage{upgreek}

\newtheorem{assum}{A\hspace{-2pt}}

\newtheorem{theorem}{Theorem}
\crefname{theorem}{theorem}{Theorems}
\Crefname{theorem}{Theorem}{Theorems}

\newcounter{lemma}
\renewcommand{\thelemma}{\arabic{lemma}}

\newenvironment{lemma}[1][]{
    \refstepcounter{lemma}
    \noindent\textbf{Lemma~\thelemma. #1}\itshape
}{\par}

\crefname{lemma}{lemma}{lemmas}
\Crefname{lemma}{Lemma}{Lemmas}

\newcounter{definition}
\renewcommand{\thedefinition}{\arabic{definition}}

\newenvironment{definition}[1][]{
    \refstepcounter{definition}
    \noindent\textbf{Definition~\thedefinition. #1}\itshape
}{\par}

\crefname{definition}{definition}{definitions}
\Crefname{definition}{Definition}{Definitions}

\newcounter{remark}
\renewcommand{\theremark}{\arabic{remark}}

\newenvironment{remark}[1][]{
    \refstepcounter{remark}
    \noindent\textbf{Remark~\theremark. #1}\itshape
}{\par}

\crefname{remark}{remark}{remarks}
\Crefname{remark}{Remark}{Remarks}

\newcounter{corollary}
\renewcommand{\thecorollary}{\arabic{corollary}}

\newenvironment{corollary}[1][]{
    \refstepcounter{corollary}
    \noindent\textbf{Corollary~\thecorollary. #1}\itshape
}{\par}

\crefname{corollary}{corollary}{corollaries}
\Crefname{corollary}{Corollary}{Corollaries}

\newcounter{proposition}
\renewcommand{\theproposition}{\arabic{proposition}}

\newenvironment{proposition}[1][]{
    \refstepcounter{proposition}
    \newline\textbf{Proposition~\theproposition. #1}\itshape
}{\par}

\crefname{proposition}{proposition}{propositions}
\Crefname{proposition}{Proposition}{Propositions}

\crefname{definition}{definition}{definitions}
\Crefname{Definition}{Definition}{Definitions}

\crefname{example}{example}{examples}
\Crefname{Example}{Example}{Examples}

\crefname{figure}{figure}{figures}
\Crefname{Figure}{Figure}{Figures}

\crefname{table}{table}{tables}
\Crefname{Table}{Table}{Tables}

\crefname{algorithm}{algorithm}{algorithms}
\Crefname{Algorithm}{Algorithm}{Algorithms}

\crefname{assum}{A\hspace{-3pt}}{A\hspace{-3pt}}
\crefname{assumb}{B\hspace{-2pt}}{B\hspace{-2pt}}
\crefname{assumUGE}{UGE\hspace{-1pt}}{UGE\hspace{-1pt}}
\crefname{assumID}{IND\hspace{-1pt}}{IND\hspace{-1pt}}
\crefname{assumUE}{UE\hspace{-1pt}}{UE\hspace{-1pt}}
\crefname{assumM}{M\hspace{-1pt}}{M\hspace{-1pt}}

\newlist{renumerate}{enumerate}{3}
\setlist[renumerate]{wide, labelwidth=!, labelindent=0pt,label=(\roman*)}

\newlist{aenumerate}{enumerate}{3}
\setlist[aenumerate]{wide, labelwidth=!, labelindent=0pt,label=(\arabic*)}

\newlist{aaenumerate}{enumerate}{3}
\setlist[aaenumerate]{wide, labelwidth=!, labelindent=0pt,label=(\alph*)}

\newlist{aenumerateSpace}{enumerate}{3}
\setlist[aenumerateSpace]{wide, labelwidth=!,label=(\arabic*)}

\newlist{benumerate}{enumerate}{3}
\setlist[benumerate]{wide, labelwidth=!, labelindent=0pt,label=$\bullet$}

\newtheorem{assumprime}{\textbf{A'}\hspace{-1pt}}
\Crefname{assumprime}{\textbf{A'}\hspace{-1pt}}{\textbf{A'}\hspace{-1pt}}
\crefname{assumprime}{\textbf{A'}}{\textbf{A'}}
\newcommand{\dto}{\overset{d}{\to}}

\newcommand{\PE}{\mathbb{E}}
\newcommand{\var}{\operatorname{Var}}
\newcommand{\PP}{\mathbb{P}}
\newcommandx{\genericb}[1][1=]{b_{#1}}
\newcommandx{\Constros}[1][1=]{\operatorname{C}_{\operatorname{Ros},#1}}
\newcommandx{\Constburk}[1][1=]{\operatorname{C}_{\operatorname{Burk}}}
\newcommandx{\driftW}[1][1=]{W_{#1}}

\newcommandx{\metricd}[1][1=]{\mathsf{d}_{#1}}

\newcommandx\invmeasure[1][1=]{\Pi_{#1}}
\newcommandx{\PPjoint}[1][1=]{\PP^{\MKjoint[#1]}}
\newcommandx{\PEjoint}[1][1=]{\PE^{\MKjoint[#1]}}
\newcommandx{\PEMID}[1][1=\alpha]{\PE^{\MK[#1]}}
\newcommandx{\PPMID}[1][1=\alpha]{\PP^{\MK[#1]}}

\newcommandx{\MKjoint}[1][1=]{\bar{\operatorname{P}}_{#1}}
\newcommandx\costw[1][1=]{\mathsf{c}_{#1}}

\newcommandx\Intergrdist[1][1=]{\mathbb{M}_{1}(#1)}

\newcommandx{\mmarkov}[1][1=0]{m^{(\Markov)}_{#1}}

\def\Conv{\mathsf{C}}

\def\F{\mathcal{F}}

\def\Xset{\mathsf{X}}
\def\Xsigma{\mathcal{X}}

\def\Zset{\mathsf{Z}}
\def\Zsigma{\mathcal{Z}}

\def\rset{\mathbb{R}}

\def\nset{\ensuremath{\mathbb{N}}}
\def\nsets{\ensuremath{\mathbb{N}^*}}

\newcommand{\SG}{\operatorname{SG}}

\newcommandx\sequence[4][2=,3=,4=]
{\ifthenelse{\equal{#3}{}}{\ensuremath{\{ #1_{#2 #4}\}}}{\ensuremath{\{ #1_{#2 #4} \}_{#2 \in #3}}}}

\newcommandx\sequenceD[2][2=]
{\ifthenelse{\equal{#2}{}}{\ensuremath{\{ #1\}}}{\ensuremath{\{ #1\!~:\!~#2  \}}}}
\newcommandx\sequenceDouble[4][3=,4=]
{\ifthenelse{\equal{#3}{}}{\ensuremath{\{ (#1_{#3},#2_{#3})\}}}{\ensuremath{\{ (#1_{#3},#2_{#3})\}_{#3 \in #4}}}}

\newcommandx{\sequencen}[2][2=n\in\nset]{\ensuremath{\{ #1, \eqsp #2 \}}}
\newcommandx\sequencens[2][2=n]
{\ensuremath{\{ #1_{#2} \!~:\!~#2\in\nsets\}}}
\newcommandx\sequencet[4]
{\ensuremath{\{ #1{#2_{#3}} \, : \, \eqsp #3 \in #4 \}}}
\def\PE{\mathbb{E}}
\def\P{\mathbb{P}}

\newcommandx{\PVar}[1][1=]{\ensuremath{\operatorname{Var}_{#1}}}
\newcommandx\conststab[1][1=p]{\varkappa_{#1}}

\def\noisecov{\Sigma_\xi}

\newcommandx{\MK}[1][1=\alpha]{\mathrm{P}_{#1}}
\newcommandx\MKK[1][1=\alpha]{\mathrm{K}_{#1}}

\newcommandx{\PEtilde}[1][1=]{\PE^{\mathrm{K}_{#1}}}
\newcommandx{\PPtilde}[1][1=]{\PP^{\mathrm{K}_{#1}}}

\newcommandx{\norm}[2][2=]{\Vert#1 \Vert_{{#2}}}
\newcommandx{\normLigne}[2][2=]{\Vert#1 \Vert_{{#2}}}
\newcommandx{\normLine}[2][2=]{\Vert#1 \Vert_{{#2}}}
\newcommandx{\normop}[2][2=]{\Vert{#1}\Vert_{{#2}}}
\newcommandx{\normopLigne}[2][2=]{\Vert{#1}\Vert_{{#2}}}
\newcommandx{\normopLine}[2][2=]{\Vert{#1}\Vert_{{#2}}}
\newcommandx{\osc}[2][1=]{\mathrm{osc}_{#1}(#2)}
\newcommandx{\normopadapt}[2][2=]{\left\Vert{#1}\right\Vert_{{#2}}}

\newcommandx{\CPP}[3][1=]
{\ifthenelse{\equal{#1}{}}{{\mathbb P}\left(\left. #2 \, \right| #3 \right)}{{\mathbb P}_{#1}\left(\left. #2 \, \right | #3 \right)}}
\newcommandx{\CPPtilde}[3][1=]
{\ifthenelse{\equal{#1}{}}{{\tilde{\mathbb P}}\left(\left. #2 \, \right| #3 \right)}{{\tilde{\mathbb P}}_{#1}\left(\left. #2 \, \right | #3 \right)}}

\def\iid{i.i.d.}

\newcommandx{\as}[1][1=\PP]{\ensuremath{#1\, -\mathrm{a.s.}}}

\newcommand{\eqsp}{\;}

\newcommand{\Id}{\mathrm{I}}
\def\prtheta{\bar{\theta}}

\newcommand{\ConstC}{\operatorname{C}}

\newcommandx{\boundmetric}[1][1=]{\kappa_{\MKK[#1]}}

\newcommandx{\Nnorm}[2][1=V]{[ #2]_{#1}}
\newcommandx{\lipnorm}[2][1=g]{[ #1]_{#2}}

\newcommandx{\CPE}[3][1=]{{\mathbb E}^{#3}_{#1}\left[#2\right]}
\newcommandx{\CPEext}[3][1=]{\tilde{\mathbb E}^{#3}_{#1}\left[#2\right]}
\newcommandx{\CPEtilde}[3][1=]{{\tilde{\mathbb E}}^{#3}_{#1}\left[#2\right]}
\newcommandx{\CPEs}[3][1=]{{\mathbb E}^{#3}_{#1}[#2]}

\def\thetalim{\theta^\star}

\newcommand{\rme}{\mathrm{e}}
\newcommand{\rmd}{\mathrm{d}}
\def\funcAw{\mathbf{A}}

\def\funcbw{\mathbf{b}}

\newcommandx{\zmfuncA}[2][1=]{\tilde{\funcAw}^{#1}(#2)}
\newcommandx{\zmfuncAw}[1][1=]{\tilde{\funcAw}_{#1}}
\newcommandx{\zmfuncb}[2][1=]{\tilde{\funcbw}^{#1}(#2)}

\newcommandx{\funcct}[2][1=]{\funcctilde^{#1}(#2)}

\newcommand{\frobnorm}[1]{\left\Vert #1 \right\Vert_{\mathrm{F}}}

\newcommand{\pscal}[2]{\langle#1\,,\,#2\rangle}

\newcommandx{\CovC}[1][1=u]{\operatorname{C}_{#1}}

\usepackage[textwidth=2.25cm, textsize=footnotesize]{todonotes}

\newcommand{\som}[1]{\todo[color=green!20]{\textbf{SS:} #1}}

\DeclareMathAlphabet{\mathpzc}{OT1}{pzc}{m}{it}

\def\lyapW{\mathpzc{W}}

\newcommandx{\bias}[1][1=\alpha]{\operatorname{B}_{#1}}

\newcommandx\probaMarkovTilde[2][2=]
{\ifthenelse{\equal{#2}{}}{{\widetilde{\mathbb{P}}_{#1}}}{\widetilde{\mathbb{P}}_{#1}\left[ #2\right]}}

\def\thetas{\thetalim}

\def\funcctilde{\tilde{c}_u}

\newcommandx{\driftb}[1][1=p]{\bar{b}_{#1}}

\newcommandx{\boldb}[1][1={q}]{\mathsf{b}_{#1}}

\newcommandx{\ConstGW}[1][1={n,\lyapW}]{\operatorname{G}_{#1}}

\newcommandx{\ConstMW}[1][1={n,\lyapW}]{\operatorname{M}_{#1}}

\Crefname{assumTD}{\textbf{TD}\hspace{-1pt}}{\textbf{TD}\hspace{-1pt}}
\crefname{assumTD}{\textbf{TD}}{\textbf{TD}}

\Crefname{assumptionC}{\textbf{C}\hspace{-1pt}}{\textbf{C}\hspace{-1pt}}
\crefname{assumptionC}{\textbf{C}}{\textbf{C}}

\Crefname{assumptionM}{\textbf{UGE}\hspace{-1pt}}{\textbf{UGE}\hspace{-1pt}}
\crefname{assumptionM}{\textbf{UGE}}{\textbf{UGE}}

\def\distance{\mathsf{d}}

\newcommandx{\vartconstwas}[1][1=V]{c_{#1}}

\newcommandx{\deltawas}[1][1=*]{\delta_{#1}}

\newcommandx{\wasser}[4][1=\distance,4=]{\mathbf{W}_{#1}^{#4}\left(#2,#3\right)}
\newcommandx{\covcoeff}[2]{\rho_{#1}^{(#2)}}

\newcommand{\dobrush}{\mathsf{\Delta}}
\newcommandx{\dobru}[3][1=,3=]{\dobrush_{#1}^{#3}( #2)}  

\def\Markov{\mathrm{M}}

\def\gauss{\mathcal{N}}

\newcommandx{\dlim}[1]{\ensuremath{\stackrel{#1}{\Longrightarrow}}}

\def\boot{\mathsf{b}}
\newcommand{\PPb}{\mathbb{P}^\boot}
\newcommand{\PEb}{\mathbb{E}^\boot}

\def\kolmogorov{\mathsf{d}_{\Conv}}

\begin{document}

%

%
\runningauthor{ M. Sheshukova,  S. Samsonov,  D. Belomestny, E. Moulines, Q. Shao, Z. Zhang, A. Naumov}

\twocolumn[

\aistatstitle{Gaussian Approximation and Multiplier Bootstrap for Stochastic Gradient Descent}

\aistatsauthor{Marina Sheshukova \textsuperscript{1} \And  Sergey Samsonov \textsuperscript{1} \And Denis Belomestny \textsuperscript{2,1} \And Éric Moulines \textsuperscript{3,4}}
\aistatsauthor{Qi-Man Shao\textsuperscript{5} \And Zhuo-Song Zhang  \textsuperscript{5} \And Alexey Naumov \textsuperscript{1,6}}

\aistatsaddress{
\textsuperscript{1}HSE University, 
\textsuperscript{2}Duisburg-Essen University
\textsuperscript{3}LRE EPITA
\textsuperscript{4}MBZUAI CMS Division\\
\textsuperscript{5}Southern University of Science and Technology
\textsuperscript{6}Steklov Mathematical Institute of Russian Academy of Sciences
} 

]

\begin{abstract}
In this paper, we establish the non-asymptotic validity of the multiplier bootstrap procedure for constructing the confidence sets using the Stochastic Gradient Descent (SGD) algorithm. Under appropriate regularity conditions, our approach avoids the need to approximate the limiting covariance of Polyak-Ruppert SGD iterates, which allows us to derive approximation rates in convex distance of order up to $1/\sqrt{n}$. Notably, this rate can be faster than the one that can be proven in the Polyak-Juditsky central limit theorem. To our knowledge, this provides the first fully non-asymptotic bound on the accuracy of bootstrap approximations in SGD algorithms. Our analysis builds on the Gaussian approximation results for nonlinear statistics of independent random variables.
\end{abstract}

\section{INTRODUCTION}
\label{sec:intro}
Stochastic Gradient Descent (SGD) is a widely used first-order optimization method well suited for large datasets and online learning. The algorithm has attracted significant attention; see, e.g. \citep{polyak1992acceleration,nemirovski2009robust,moulines2011non}. SGD aims to solve the optimization problem:
\begin{equation}
\label{eq:stoch_minimization}
f(\theta) \to \min_{\theta \in \rset^{d}}\eqsp, \qquad \nabla f(\theta) = \PE_{\xi \sim \PP_{\xi}}[F(\theta,\xi)] \eqsp, 
\end{equation}
where $\xi$ is a random variable defined on a measurable space $(\Zset,\Zsigma)$. Instead of the exact gradient $\nabla f(\theta)$, the algorithm accesses only unbiased stochastic estimates $F(\theta,\xi)$. 
\par 
Throughout this work, we focus on strongly convex objective functions and denote by $\thetas$ the unique minimizer of \eqref{eq:stoch_minimization}. The iterates $\theta_k$, $k \in \nset$, generated by SGD follow the recursive update:
\begin{equation}
\label{eq:sgd_recursion_main}
\theta_{k+1} = \theta_{k} - \alpha_{k+1} F(\theta_k,\xi_{k+1})\eqsp, \quad \theta_0 \in \rset^{d}\eqsp, 
\end{equation}
where $\{\alpha_k\}_{k \in \nset}$ is a sequence of step sizes (or learning rates), which may be either diminishing or constant, and $\{\xi_k\}_{k \in \nset}$ is an \iid\, sequence sampled from $\PP_{\xi}$. Theoretical properties of SGD, particularly in the convex and strongly convex settings, have been extensively studied; see, e.g., \citep{nesterov:2004, moulines2011non, bubeck2015convex, lan_2020}. Many optimization algorithms build on the recurrence \eqref{eq:sgd_recursion_main} to accelerate the convergence of the sequence $\theta_k$ to $\thetas$. Notable examples include momentum acceleration \citep{qian1999momentum}, variance reduction techniques \citep{defazio2014saga, schmidt2017minimizing}, and averaging methods. In this work, we focus on Polyak-Ruppert averaging, originally proposed in \citep{ruppert1988efficient} and \citep{polyak1992acceleration}, which improves convergence by averaging the SGD iterates in \eqref{eq:sgd_recursion_main}. Specifically, the estimator is defined as
\begin{equation}
\label{eq:PR_estimate}
\textstyle 
\bar \theta_n = \frac{1}{n}\sum_{i=0}^{n-1}\theta_i\eqsp, \quad n \in \nset\eqsp.
\end{equation}
It has been established (see \citep[Theorem 3]{polyak1992acceleration}) that, under appropriate conditions on the objective function $f$, the noisy gradient estimates $F$, and the step sizes $\alpha_k$, the sequence of averaged iterates $\{\bar{\theta}_{n}\}_{n \in \nset}$ is asymptotically normal:
\begin{equation}
\label{eq:CLT_fort}
\textstyle 
\sqrt{n}(\bar{\theta}_{n} - \thetas) \dto \gauss(0,\Sigma_\infty)\eqsp,
\end{equation}
where $\dto$ denotes convergence in distribution, and $\gauss(0,\Sigma_\infty)$ is a zero-mean Gaussian distribution with covariance matrix $\Sigma_\infty$, defined later in \Cref{sec:gaus_approximation}.  This result raises two key questions:
\begin{enumerate}[(i),noitemsep,nolistsep]
\item What is the rate of convergence in \eqref{eq:CLT_fort}?
\item How can \eqref{eq:CLT_fort} be leveraged to construct confidence sets for $\thetas$, given that $\Sigma_\infty$ is unknown in practice?
\end{enumerate}
In our paper, we aim to answer both questions. To quantify convergence rates in \eqref{eq:CLT_fort}, we use convex distance, defined for random vectors $X, Y \in \rset^d$ as
\begin{equation}
\label{eq:convex-distance-definiton}
\kolmogorov(X, Y) = \sup_{B \in \Conv(\rset^{d})}\left|\P\bigl(X \in B\bigr) - \P(Y \in B)\right|\eqsp,
\end{equation}
where $\Conv(\rset^{d})$ denotes the collection of convex subsets of $\rset^{d}$. 
The supremum in \eqref{eq:convex-distance-definiton} can be taken over different classes of sets, leading to various probability metrics. In particular, more restrictive classes such as rectangles give rise to metrics with logarithmic dependence on the dimension $d$ (see, e.g., \citep{kojevnikov2022berry, Chernozhukov2013}). The choice of the class of sets is often driven by the needs of a particular application and may affect the dependence of the resulting bounds on the dimension $d$. In particular, even for normal approximation of linear statistics, the dimensional dependence may vary depending on the chosen class. Results for convex distance can be found in \citep{bentkus2004}, while results for rectangles are available in \citep{kojevnikov2022berry, Chernozhukov2013}.

 \citet{shao2022berry} derive Berry-Esseen-type bounds for $\kolmogorov(\sqrt{n}(\bar{\theta}_{n} - \thetas), \gauss(0,\Sigma_{n}))$, where
$\Sigma_n$ is the covariance matrix of the linearized counterpart of \eqref{eq:sgd_recursion_main}; see the precise definition below in \eqref{eq:sigma_n_def}. We complement this result with the rates of convergence in \eqref{eq:CLT_fort}. We also establish a lower bound on the convex distance 
$\kolmogorov(\sqrt{n}(\bar{\theta}_{n} - \thetas), \gauss(0,\Sigma_{\infty}))$. 
This result shows that, for certain step-size sequences $\alpha_k$ in 
\eqref{eq:sgd_recursion_main}, the Gaussian approximation with covariance 
$\Sigma_{\infty}$ is less accurate than an approximation based on another 
covariance, in particular $\Sigma_n$. A similar phenomenon has been reported in 
the bootstrap literature for \iid\ data outside the context of gradient methods; 
see \citep[Theorem~3.11]{shao1995jackknife}.

One popular approach for solving (ii) is based on the \emph{plug-in} methods \citep{chen2020aos, chen2021statistical}, which aim to construct an estimator $\hat{\Sigma}_n$ of $\Sigma_{\infty}$ directly. Theoretical guarantees for these methods typically focus on non-asymptotic bounds for how close $\hat{\Sigma}_n$ is to $\Sigma_{\infty}$, often in terms of $\PE[\norm{\hat{\Sigma}_n - \Sigma_{\infty}}]$. At the same time, the analysis of these methods bypasses item (i) and the issues related to the rate of convergence in \eqref{eq:CLT_fort}. In our paper, we present, to our knowledge, the first fully non-asymptotic analysis of a procedure for constructing confidence intervals based on the bootstrap approach \citep{efron1992bootstrap, JMLR:v19:17-370}, which avoids direct approximation of $\Sigma_{\infty}$. Moreover, theoretical analysis of the underlying procedure, together with results on normal approximation with $\mathcal{N}(0,\Sigma_{\infty})$ from (i), shows that the same approximation rate cannot be achieved by plug-in methods, at least for a certain range of step sizes $\alpha_k$ in \eqref{eq:sgd_recursion_main}. Our key contributions are as follows:
\begin{itemize}[noitemsep, nolistsep]
    \item We establish the non-asymptotic validity of the multiplier bootstrap procedure introduced in \citep{JMLR:v19:17-370}. Under suitable regularity conditions, our bounds show that the distribution of $\sqrt{n}(\bar{\theta}_{n} - \thetas)$ can be approximated, up to logarithmic factors, at rate $1/n^{\gamma-1/2}$ for step sizes of the form $\alpha_k = c_0/(k+k_0)^{\gamma}$ with $\gamma \in (1/2,1)$. To our knowledge, this is the first bound on the accuracy of bootstrap approximations in SGD. Notably, this rate can be faster than the one obtained in \eqref{eq:CLT_fort}. Our results improve upon recent works \citep{samsonov2024gaussian, wu2024statistical}, which addressed the convergence rate in similar procedures for the LSA algorithm and TD learning, respectively.
    
    \item  Our analysis of the multiplier bootstrap procedure reveals an important distinction: unlike plug-in estimators, the validity of the bootstrap method does not depend on approximating $\sqrt{n}(\bar{\theta}_{n} - \thetas)$ by $\gauss(0,\Sigma_\infty)$. Instead, it requires approximation by $\gauss(0, \Sigma_n)$ with the matrix $\Sigma_n$ being the covariance matrix of the linearized counterpart of \eqref{eq:sgd_recursion_main}. The structure of $\Sigma_n$ is central to our analysis, both for the rate in \eqref{eq:CLT_fort} and for the non-asymptotic bootstrap validity. Precise definitions are provided in \Cref{sec:gaus_approximation}.
    
    \item We analyze the Polyak–Ruppert averaged SGD iterates \eqref{eq:PR_estimate}  for strongly convex minimization problems and establish Gaussian approximation rates in \eqref{eq:CLT_fort} with respect to the convex distance. Specifically, we show that $\kolmogorov(\sqrt{n}(\bar{\theta}_{n} - \thetas), \gauss(0,\Sigma_\infty))$ is of order $n^{-1/4}$ for step sizes $\alpha_k = c_0/(k+k_0)^{3/4}$ with appropriately chosen $c_0$ and $k_0$. Our proof relies on the techniques of \citep{shao2022berry} and \citep{wu2024statistical}. We further provide a lower bound showing that this rate of normal approximation with $\mathcal{N}(0,\Sigma_{\infty})$ is tight in the regime $\alpha_k = c_0/(k+k_0)^{\gamma}$ with $\gamma \geq 3/4$.
\end{itemize}

\vspace{-3mm}
\paragraph{Notations.} Throughout this paper, we use the following notations. For a matrix $A \in \rset^{d \times d}$ and a vector $x \in \rset^{d}$, we denote by $\norm{A}$ and $\norm{x}$ their spectral norm and Euclidean norm, respectively. We also write $\frobnorm{A}$ for the Frobenius norm of matrix $A$. Given a function $f: \rset^{d} \to \rset$, we write $\nabla f(\theta)$ and $\nabla^2 f(\theta)$ for its gradient and Hessian at a point $\theta$. We use the standard abbreviations "i.i.d." for "independent and identically distributed" and "w.r.t." for "with respect to".

\paragraph{Literature review.} 
The asymptotic behavior of SGD, including the asymptotic normality of the estimator $\bar{\theta}_n$ and its almost sure convergence, has been studied for smooth and strongly convex minimization problems \citep{polyak1992acceleration, kushner2003stochastic, benveniste2012adaptive}. Optimal mean-squared error (MSE) bounds for $\theta_n - \thetas$ and $\bar{\theta}_n - \thetas$ were first derived in \citep{nemirovski2009robust} for smooth and strongly convex objectives, and later refined in \citep{moulines2011non}. The constant step size regime for strongly convex problems has been analyzed in \citep{durmus2020biassgd, li2025statistical}. High-probability bounds for SGD iterates were established in \citep{rakhlin2012making} and later extended in \citep{harvey2019tight}, both addressing non-smooth and strongly convex minimization. 
\par 
It is important to note that the results discussed above do not directly yield convergence rates for $\sqrt{n}(\bar{\theta}_{n} - \thetas)$ to $\gauss(0,\Sigma_\infty)$ in terms of convex or Wasserstein distance. Among the relevant contributions in this direction, we highlight recent works \citep{srikant2024rates, samsonov2024gaussian, wu2024statistical}, which provide quantitative bounds on the convergence rate in \eqref{eq:CLT_fort} for iterates of temporal difference (TD) learning and general linear stochastic approximation (LSA) schemes. These algorithms, however, do not necessarily reduce to SGD with a quadratic objective $f$, since the system matrix in LSA is not required to be symmetric. Convergence rates of order up to $1/\sqrt{n}$ were established in \citep{pmlr-v99-anastasiou19a} for a class of smooth test functions. The recent work \citep{agrawalla2023high} establishes Berry–Esseen bounds of order up to $n^{-1/4}$ for the last iterate of SGD in the linear regression setting.
\par 
Bootstrap methods for i.i.d.\ observations were first introduced in \citep{efron1992bootstrap}. In the context of SGD, \citet{JMLR:v19:17-370} proposed the multiplier bootstrap procedure for constructing confidence intervals for $\thetas$ and established its asymptotic validity. The same method was later analyzed in \citep{samsonov2024gaussian} for the LSA algorithm, where a rate of $n^{-1/4}$ was obtained for the accuracy of approximating the distribution of $\sqrt{n}(\bar{\theta}_n - \thetas)$ by the
distribution of its bootstrap-world counterpart in terms of convex distance.

\par 
A naive approach for constructing confidence intervals is to perform multiple independent runs of the algorithm to empirically acquire distributional information and then construct a confidence interval, as discussed in \citep{zhu2024high}.
A popular class of methods for constructing confidence sets for $\thetas$ relies on estimating the asymptotic covariance matrix $\Sigma_{\infty}$. Plug-in and batch-mean estimators of $\Sigma_{\infty}$ have attracted a lot of attention \citep{chen2020aos,chen2021statistical,chen2022online}, especially in the setting when the stochastic estimates of Hessian are available. Estimators of $\Sigma_{\infty}$ based on the batch-mean method and its online variant were studied in \citep{chen2020aos} and \citep{zhu2023online_cov_matr}. \citet{pmlr-v178-li22b} studied the asymptotic validity of the plug-in estimator for $\Sigma_{\infty}$ in the local SGD framework. The analysis in \citep{zhong2023online} refined guarantees for both the multiplier bootstrap and batch-mean estimators of $\Sigma_{\infty}$ in nonconvex problems. However, these contributions typically establish recovery rates for $\Sigma_{\infty}$ but only prove the asymptotic validity of the resulting confidence intervals. A notable exception is the recent work \citep{wu2024statistical}, which studied the temporal-difference (TD) learning algorithm. There, the authors provided a fully non-asymptotic analysis, obtaining the approximation rate $n^{-1/3}$ for the distribution of $\sqrt{n}(\bar{\theta}_n - \thetas)$ in terms of convex distance.

\section{MAIN RESULTS}
\label{sec:bootstrap}
This section establishes the non-asymptotic validity of the multiplier bootstrap method proposed in \citep{JMLR:v19:17-370}. We restrict attention to smooth and strongly convex minimization problems, following the frameworks of \citep{moulines2011non}, \citep{pmlr-v99-anastasiou19a}, and \citep{shao2022berry}. The procedure is based on perturbing the trajectory \eqref{eq:sgd_recursion_main}. Let 
\(\mathcal{W}^{n-1} = \{w_{\ell}\}_{1 \leq \ell \leq n-1}\)
be i.i.d.\ random variables with distribution $\PP_w$, each satisfying \(\PE[w_1] = 1\) and \(\var[w_1] = 1\). We assume that \(\mathcal{W}^{n-1}\) is independent of
\(\Xi^{n-1} = \{\xi_{\ell}\}_{1 \leq \ell \leq n-1}\).
We then use \(\mathcal{W}^{n-1}\) to construct randomly perturbed trajectories of the SGD dynamics \eqref{eq:sgd_recursion_main}:
\begin{equation}
\label{eq:sgd_bootstrap}
\begin{split}
\theta_{k}^{\boot} 
&= \theta_{k-1}^{\boot} - \alpha_{k} w_k F(\theta_{k-1}^\boot,\xi_{k}) \eqsp, 
\quad \theta_{0}^{\boot} = \theta_{0} \in \rset^{d} \eqsp, \\
\prtheta_{n}^{\boot} 
&=  n^{-1} \sum\nolimits_{k=0}^{n-1} \theta_k^{\boot} \eqsp, 
\quad n \geq 1 \eqsp.
\end{split}
\end{equation}
When generating different weights $w_{k}$, we obtain samples from the conditional distribution of $\bar{\theta}_{n}^\boot$ given the data $\Xi^{n-1}$. We further denote
\begin{align}
\label{def: PPb}
\PPb = \PP(\cdot \mid \Xi^{n-1}) \eqsp, 
\quad \PEb = \PE(\cdot \mid \Xi^{n-1}) \eqsp. 
\end{align}
The core principle behind the bootstrap procedure \eqref{eq:sgd_bootstrap} is that the "bootstrap world" probabilities $\PPb\big(\sqrt{n} (\bar{\theta}_{n}^\boot - \bar{\theta}_n) \in B\big)$ are close to $\PP\big(\sqrt{n} (\bar{\theta}_n - \thetas) \in B\big)$ for $B \in \Conv(\rset^{d})$. Formally, we say that the procedure \eqref{eq:sgd_bootstrap} is asymptotically valid if 
\begin{equation}
\sup_{B \in \Conv(\rset^{d})} \Big| \PPb\big(\sqrt{n} (\bar{\theta}_{n}^\boot - \bar{\theta}_n) \in B \big) - \PP\big(\sqrt{n} (\bar{\theta}_n - \thetas) \in B\big) \Big|
\end{equation}
converges to $0$ in $\PP$-probability as $n \to \infty$. This result was obtained in \citep{JMLR:v19:17-370} under assumptions close to the original paper \citep{polyak1992acceleration}.
While an analytical expression for \(\PPb(\sqrt{n} (\bar{\theta}_{n}^\boot - \bar{\theta}_n) \in B)\) is unavailable, it can be approximated via Monte Carlo simulations by generating \(M\) perturbed trajectories according to \eqref{eq:sgd_bootstrap}. Standard arguments (see, e.g., \citep[Section~5.1]{shao2003mathematical}) suggest that the accuracy of this Monte Carlo approximation scales as \(\mathcal{O}(M^{-1/2})\) when generating $M$ parallel perturbed trajectories in \eqref{eq:sgd_bootstrap}.
\par 
\paragraph{Assumptions.} We impose the following regularity conditions on the objective function $f$:  
\begin{assum}
\label{ass:L-smooth}
The function $f$ is twice continuously differentiable and $L_{1}$-smooth on $\rset^{d}$; that is, there exists a constant $L_{1} > 0$ such that for any $\theta,\theta' \in \rset^{d}$,
\begin{equation}
\label{eq:L-smooth}
\norm{\nabla f(\theta) - \nabla f(\theta')} \leq L_{1} \norm{\theta - \theta'} \eqsp.
\end{equation}
Moreover, $f$ is assumed to be $\mu$-strongly convex on $\rset^{d}$; that is, there exists a constant $\mu > 0$ such that for any $\theta,\theta' \in \rset^{d}$,
\begin{equation}
\label{eq:strong_convex}
(\mu/2)\norm{\theta-\theta'}^2 \leq f(\theta) - f(\theta') - \langle \nabla f(\theta'), \theta - \theta' \rangle \eqsp.
\end{equation}
\end{assum}
\Cref{ass:L-smooth} implies the following bound on the Hessian of $f$:
\[
\textstyle 
\mu \Id_d \preceq \nabla^2 f(\theta) \preceq L_1 \Id_d,
\]
for all $\theta \in \rset^{d}$. We next state the assumptions on $F(\theta,\xi)$. Specifically, we write
\begin{equation}
\label{eq: gradF decompositon}
\textstyle 
F(\theta_{k-1}, \xi_k) = \nabla f(\theta_{k-1}) + \zeta_k \eqsp, \quad 
\zeta_k := \zeta(\theta_{k-1}, \xi_k) \eqsp,
\end{equation}
so that $\{\zeta_k\}_{k \in \nset}$ is a sequence of $d$-dimensional random vectors whose distribution may depend on $\theta_{k-1}$. The SGD recursion \eqref{eq:sgd_recursion_main} then takes the form  
\begin{equation}
\label{eq:sgd_recursion_main_new}
\textstyle
\theta_{k} = \theta_{k-1} - \alpha_{k}\big(\nabla f(\theta_{k-1}) + \zeta_k\big) \eqsp, 
\quad \theta_0 \in \rset^{d} \eqsp.
\end{equation}  
We impose the following assumption on the noise sequence $\zeta_k$:
\begin{assum}
\label{ass:bound_noise}
For each $k \geq 1$, $\zeta_k$ admits the decomposition $\zeta_k = \eta(\xi_k) + g(\theta_{k-1}, \xi_k)$, where
\begin{enumerate}[(i),noitemsep,nolistsep]
    \item $\{\xi_k\}_{k=1}^{n-1}$ is a sequence of i.i.d.\ random variables on $(\Zset,\Zsigma)$ with distribution $\PP_{\xi}$. The function $\eta: \Zset \to \rset^d$ satisfies $\PE[\eta(\xi_1)]=0$ and $\PE[\eta(\xi_1)\eta(\xi_1)^\top]=\Sigma_{\xi}$, with $\lambda_{\min}(\Sigma_{\xi}) > 0$.
    
    \item The function $g: \rset^d \times \Zset \to \rset^d$ satisfies $\PE[g(\theta,\xi_1)]=0$ for all $\theta \in \rset^{d}$. Moreover, there exists $L_2 > 0$ such that for any $\theta,\theta' \in \rset^{d}$ and  $2 \leq p\leq \log n$,
    \begin{equation}
    \label{eq: g_bound_norm}
    \textstyle 
    \PE^{1/p}[\norm{g(\theta,\xi)-g(\theta',\xi)}^p]\leq L_2 \norm{\theta - \theta'}\eqsp,
    \end{equation}
    and $g(\thetas,z)=0$ for all $z \in \Zset$.
    \item There exists $C_{\xi}>0$ such that for any $\theta \in \rset^{d}$, the random vector $g(\theta,\xi)+\eta(\xi)$ is sub-Gaussian with variance proxy $C_{\xi}^2$; that is, for any $v \in \rset^{d}$,
    \[ 
    \PE\!\left[\exp\{\langle g(\theta,\xi)+\eta(\xi), v \rangle\}\right] 
    \leq \exp\{\norm{v}^2 C_{\xi}^2/2\} \eqsp. 
    \]
\end{enumerate}
\end{assum}
\textbf{Discussion.} As an example of a sequence $\zeta_k$ satisfying conditions (i) and (ii) in \Cref{ass:bound_noise}, consider the case where the stochastic estimates $F(\theta,\xi)$ satisfy:
\begin{enumerate}[noitemsep,nolistsep]
    \item $\PE[F(\theta,\xi)] = \nabla f(\theta)$ for all $\theta \in \rset^{d}$;
    \item $\PE^{1/p}[\norm{F(\theta,\xi) - F(\theta',\xi)}^p] \leq L \norm{\theta - \theta'}$.
\end{enumerate}
In this case, (i) and (ii) in \Cref{ass:bound_noise} hold with $\eta(\xi) = F(\thetas,\xi)$ and $g(\theta,\xi) = F(\theta,\xi) - F(\thetas,\xi) - \nabla f(\theta)$. Condition (ii) in \Cref{ass:bound_noise} is often imposed when studying averaged iterates; see \citep[Assumption~H2]{moulines2011non} and \citep{durmus2020biassgd,SheshukovaICLR}. It is also possible to adapt our arguments to the weaker assumption
    \begin{equation}
    \textstyle 
    \PE^{1/p}[\norm{g(\theta,\xi)-g(\theta',\xi)}^p]\leq L_2 \norm{\theta - \theta'}^\beta\eqsp,
    \end{equation}
for some $1/2<\beta < 1$, at the cost of a slower approximation rate. The boundary case $\beta=1/2$, which arises for example in quantile regression, is not covered by our analysis and requires different techniques. We refer to \citep{chen2023recursive, chen2025smoothed, caitime} for non-asymptotic results in this setting. Furthermore, our proof strategy remains applicable if only a fixed number of moments $p \ge 4$ (rather than $p = \log n$) is available, at the cost of a slower rate with respect to $n$ in the final estimates. Similarly, one can allow $L_2$ in \eqref{eq: g_bound_norm} to depend polynomially on $p$. In this case, our proof approach remains valid, but yields additional polynomial dependence on $\log n$ in the final rate of \Cref{th:bootstrap_validity}. 

The assumption \Cref{ass:bound_noise}-(iii) is crucial to establish high-order moment bounds:
\begin{equation}
\label{eq:hpd_bound_preliminary}
\textstyle 
\PE^{1/p}\!\left[\norm{\theta_k - \thetas}^{p}\right] 
\quad \text{and} \quad 
\PE^{1/p}\!\left[\norm{\theta_k^{\boot} - \thetas}^{p}\right] \eqsp,
\end{equation}
see \Cref{lem: high_prob_last_iter} in the Appendix. Our proof generalizes the argument of \citep[Theorem~4.1]{harvey2019tight}, which requires the noise variables $\zeta_k$ in \eqref{eq:sgd_recursion_main_new} to be almost surely bounded. This argument was later generalized in \citep{madden2024high} to the setting where $\zeta_k$ is conditionally sub-Gaussian given $\F_{k-1} = \sigma(\theta_i, i \leq k-1)$, with variance proxy uniformly (in $\theta$) bounded by a constant factor. This is exactly the setting considered in \Cref{ass:bound_noise}-(iii). Such an assumption is widely used in the literature; see \citep{nemirovski2009robust,hazan2014beyond} and the remarks in \citep{harvey2019tight}. We believe that this assumption can be generalized for sub-Weibull noise setting (see \citep[Definition~5]{madden2024high}), at the expense of additional technicalities (see \citep[Theorem~9]{madden2024high}). Some of the authors considering bounds of type \eqref{eq:hpd_bound_preliminary}, such as \citet{rakhlin2012making}, imposed the stronger assumption that $\sup_{\theta \in \rset^{d}}\norm{F(\theta,\xi)}$ is almost surely bounded. A less restrictive approach might be to consider updates involving gradient clipping; see, e.g., \citep{sadiev2023high}, or a version of SGD algorithm with projections as in \citep{rakhlin2012making}. However, both approaches change the limiting covariance matrix $\Sigma_\infty$ and the leading term in Gaussian approximation for $\sqrt{n} (\bar{\theta}_{n}^\boot - \bar{\theta}_n)$. Both approaches would introduce an additional level of technical difficulties, so we leave a detailed study of these schemes for the future work. 

We further impose the following condition on the Hessian matrix $\nabla^2 f(\theta)$ at $\thetas$:
\begin{assum}
\label{ass:hessian_Lipschitz_ball}
There exist constants $L_3,\beta > 0$ such that for all $\theta$ with $\norm{\theta - \thetas} \leq \beta$, we have
\begin{equation}
\textstyle 
\norm{\nabla^2 f(\theta) - \nabla^2 f(\thetas)} \leq L_3 \norm{\theta - \thetas} \eqsp.
\end{equation}
\end{assum}
\Cref{ass:hessian_Lipschitz_ball} ensures that the Hessian of $f$ is Lipschitz continuous in a neighborhood of $\thetas$. 
Similar assumptions have been considered in \citep{shao2022berry,pmlr-v99-anastasiou19a}, as well as in other works on first-order optimization methods; see, e.g., \citep{li2022root}. Several studies on the non-asymptotic analysis of SGD impose stronger smoothness conditions, such as bounded derivatives of $f$ up to order $4$, as in \citep{durmus2020biassgd}. We also impose the following assumption on the bootstrap weights $w_i$ used in the algorithm:
\begin{assum}
\label{ass:bound_bootstap_weights}
There exist constants $0 < W_{\min} < W_{\max} < +\infty$ such that $W_{\min} \leq w_1 \leq W_{\max}$ a.s. 
\end{assum}

The original work \citep{JMLR:v19:17-370} also required positive bootstrap weights $w_i$. We impose boundedness of $w_i$ in order to derive the high-probability bound in \Cref{lem: high_prob_last_iter}. An explicit example of a distribution satisfying \Cref{ass:bound_bootstap_weights} and the constraints $\PE[w_1]=1$ and $\PVar[][w_1]=1$ is given in \Cref{sec:example_distribution}. Finally, we impose the following condition on the step sizes $\alpha_k$ and sample size $n$:
\begin{assum}
\label{ass:step_size_new_boot}
Let $\alpha_k = c_0 (k_0 + k)^{-\gamma}$, where $\gamma \in (1/2,1)$ and $c_0,k_0$ satisfy $2 c_0 W_{\max}^2 L_1^2 \leq 1$, 
\[
k_0 \geq \max\!\left\{ \left(\tfrac{2\gamma}{\mu c_0 W_{\min}}\right)^{\!1/(1-\gamma)}, \; \left(\tfrac{1}{\mu W_{\min}}\right)^{\!1/\gamma} \right\} \eqsp.
\] 
\end{assum}
\begin{assum}
\label{ass:n_lower_bound}
The number of observations $n$ is assumed to be sufficiently large. Precise expressions for $n$ are provided in \Cref{sec: bootstrap validity section} (see \Cref{ass:n_lower_bound_full}).
\end{assum}

The particular bound on $k_0$ in \Cref{ass:step_size_new_boot} arises from the high-order moment bounds (see \Cref{lem: high_prob_last_iter} in the Appendix). We further discuss the lower bound on the number of observations imposed in \Cref{ass:n_lower_bound} in the proof of \Cref{th:bootstrap_validity}.

\paragraph{Discussion of assumptions}
 Most of the theoretical assumptions used in our analysis (namely, \Cref{ass:L-smooth}, \Cref{ass:bound_noise}(i,ii), \Cref{ass:hessian_Lipschitz_ball}) are standard in the stochastic approximation literature and are of the same nature as the conditions appearing in the classical Polyak--Juditsky CLT \citep{polyak1992acceleration}. Similar assumptions are also routinely used in several recent works on non-asymptotic analysis of stochastic approximation \citep{li2022root, shao2022berry, SheshukovaICLR}. The additional assumptions \Cref{ass:bound_noise}(iii), \Cref{ass:bound_bootstap_weights}, and \Cref{ass:step_size_new_boot} are required to obtain high-probability guarantees for the last SGD iterate, while \Cref{ass:n_lower_bound} ensures that $\lambda_{\min}(\Sigma_n^{\boot})$ is separated from zero with high probability.
\subsection{Non-asymptotic multiplier bootstrap validity}

\begin{theorem}
\label{th:bootstrap_validity}
Assume \Cref{ass:L-smooth}--\Cref{ass:n_lower_bound}. Then, with $\PP$-probability at least $1 - 2/n$, we have
\begin{multline}
\label{eq:gaussian_approximation_bootstrap}
\sup_{B \in \Conv(\rset^{d})} 
\Big| \PPb\big(\sqrt{n}(\bar{\theta}_{n}^\boot - \bar{\theta}_n) \in B\big) 
- \PP\big(\sqrt{n}(\bar{\theta}_n - \thetas) \in B\big) \Big|  \\
\leq \frac{\ConstC_1 \sqrt{\log n}}{n^{1/2}} 
+ \frac{\ConstC_2 \log n}{n^{\gamma - 1/2}} 
+ \frac{\ConstC_3 (\log n)^{3/2}}{n^{\gamma/2}} \eqsp,
\end{multline}
where $\ConstC_1,\ConstC_2,\ConstC_3$ are given in \Cref{app:proof_th_main_1}, equation~\eqref{eq: bootsrap constants def}.
\end{theorem}

\begin{figure*}[ht!]
  \centering
  \begin{tikzcd}[column sep=7.2em, row sep=3.0em]
  {\textbf{Real world: }\textstyle \sqrt{n}\,(\bar{\theta}_n-\thetas)}
    \arrow[r, leftrightarrow,
      "{\shortstack{Gaussian approximation\\ \scriptsize Th.~\ref{th:bound_kolmogorov_dist_pr_sigma_n}}}"]
  & {\mathcal N(0,\Sigma)}
    \arrow[d, leftrightarrow, "{\shortstack{Gaussian comparison\\ \scriptsize \citet{BarUly86},\\ \scriptsize \citet{Devroye2018}}}"] \\
  {\textbf{Bootstrap world: }\textstyle \sqrt{n}\,(\bar{\theta}_n^\boot-\bar{\theta}_n)}
    \arrow[r, leftrightarrow,
      "{\shortstack{Gaussian approximation\\ \scriptsize Th.~\ref{GAR bootstrap_main}}}"]
  & {\mathcal N(0,\Sigma^\boot)}
  \end{tikzcd}
  \caption{Gaussian approximations in the real and bootstrap worlds, and their comparison.}
  \label{fig:gaussian_diagram}
\end{figure*}

\begin{remark}
The constants $\ConstC_1,\ConstC_2,\ConstC_3$ in \Cref{th:bootstrap_validity} depend on the problem dimension $d$, as well as on $\norm{\theta_0 - \thetas}$. To make the dependence on $d$ explicit, we assume the natural scaling 
$\norm{\theta_0 - \thetas} \lesssim \sqrt{d}$. These dimension dependence aligns with the one assumed in \citep{shao2022berry}. We also assume that $C_{\xi}$ from \Cref{ass:bound_noise}-(iii) is dimension-free. Then \Cref{th:bootstrap_validity} yields
\begin{multline}
 \sup_{B \in \Conv(\rset^{d})} 
 \Big| \PPb\big(\sqrt{n}(\bar{\theta}_{n}^\boot - \bar{\theta}_n) \in B\big) 
 - \PP\big(\sqrt{n}(\bar{\theta}_n - \thetas) \in B\big) \Big| \\
 \lesssim \frac{d^2 \sqrt{\log n}}{n^{1/2}}
 + \frac{d^{3/2} \log n}{n^{\gamma - 1/2}} 
 + \frac{d^{3/2} (\log n)^{3/2}}{n^{\gamma/2}} \eqsp.
\end{multline}
Here, the notation $\lesssim$ indicates that the inequality holds up to constants independent of $n$ and $d$.
\end{remark}

\begin{remark}
\label{rem:step_size_remark}
The result of \Cref{th:bootstrap_validity} can also be established for the step size $\alpha_k = c_0/(k+k_0)$. The required Gaussian approximation with covariance matrix $\Sigma_n$ was proved in \citep{shao2022berry}. The only differences compared to \Cref{th:bootstrap_validity} are the additional $\log n$ factors in the bound and modified conditions on $c_0$ and $k_0$ in \Cref{ass:step_size_new_boot}.
\end{remark}

\begin{remark}
\label{rem:experiments}
In \Cref{sec:experiments} we provide numerical results on logistic and linear regressions, illustrating the coverage probabilities achieved by the multiplier bootstrap approach and the batch-mean approach \citep{roy2023online}.
\end{remark}

\begin{proof}[Proof sketch of Theorem \ref{th:bootstrap_validity}] The proof of non-asymptotic bootstrap validity is based on the Gaussian approximation performed both in the "real" world and bootstrap world together with an appropriate Gaussian comparison inequality:
Here, $\Sigma$ and $\Sigma^\boot$ are covariance matrices to be specified later. To understand the origin of the Gaussian approximation, we linearize the statistics $\sqrt{n}(\bar{\theta}_{n} - \thetas)$ and $\sqrt{n}(\bar{\theta}_{n}^\boot - \bar{\theta}_n)$. We detail the derivation for $\sqrt{n}(\bar{\theta}_{n} - \thetas)$, while analogous arguments for $\sqrt{n}(\bar{\theta}_{n}^\boot - \bar{\theta}_n)$ are provided in \Cref{GAR boot world sec}.  

Let $G = \nabla^2 f(\thetas)$. We expand $\sqrt{n}(\bar{\theta}_{n} - \thetas)$ into a weighted sum of independent random vectors plus lower-order terms. By the Newton–Leibniz formula,  
\begin{align}
\label{eq:H_theta_def}
\textstyle
\nabla f(\theta) = G(\theta - \thetas) + H(\theta),
\end{align}
where 
\[
H(\theta) = \int_{0}^1 \!\big(\nabla^2 f(\thetas + t(\theta - \thetas)) - G\big)(\theta - \thetas) \,\rmd t \eqsp.
\]
Note that $H(\theta)$ is of order $\norm{\theta - \thetas}^2$ (see \Cref{lem:H_theta_bound}). The recursion for the SGD error \eqref{eq:sgd_recursion_main_new} writes as  
\begin{multline}
\label{eq: sgd_reccurence_with_remainder_main}
\textstyle
\theta_k - \thetas 
= (\Id_d - \alpha_k G)(\theta_{k-1} - \thetas) 
\\ - \alpha_k\big(\eta(\xi_k) + g(\theta_{k-1},\xi_k) + H(\theta_{k-1})\big) \eqsp.
\end{multline}

For $i \in \{0,\ldots,n-1\}$, define
\begin{equation}
\label{eq:Q_i_def}
\textstyle
Q_i = \alpha_i \sum_{j=i}^{n-1} \prod_{k=i+1}^{j} (\Id_d - \alpha_k G) \eqsp,
\end{equation}
with the convention that empty products equal $\Id_d$. Averaging \eqref{eq: sgd_reccurence_with_remainder_main} and rearranging yields  $\sqrt{n}(\bar{\theta}_n - \thetas) 
= W + D $ with
\begin{equation}
\label{eq:linear and nonlinear terms}
\textstyle
W = -\frac{1}{\sqrt{n}} \sum_{i=1}^{n-1} Q_i \eta(\xi_i), \, 
D = \sqrt{n}(\bar{\theta}_n - \thetas) - W \eqsp.
\end{equation}
\som{Fix textsyle}

Here, $W$ is a weighted sum of i.i.d.\ mean-zero random vectors with covariance  
\begin{equation}
\label{eq:sigma_n_def} 
\textstyle
\Sigma_n = n^{-1} \sum_{k=1}^{n-1} Q_k \noisecov Q_k^\top \eqsp,
\end{equation}
and $D$ is the remainder term, defined in \Cref{sec:Gar_real_world_proof}, equation~\eqref{eq:linear and nonlinear terms appendix}. Moreover, in \Cref{sec: proof of difference between cov} we show that $Q_i$ can be approximated by $G^{-\top}$ and $\Sigma_n$ approximates  
\begin{equation}
\label{eq:Sigma_infty_def}
\textstyle
\Sigma_\infty = G^{-1} \Sigma_{\xi} G^{-\top} \eqsp.
\end{equation}

We expect $D$ not to affect the asymptotic distribution of the linear statistic $W$, which should be Gaussian by the central limit theorem. A key issue is the choice of the approximating Gaussian distribution $\mathcal{N}(0,\Sigma)$, with $\Sigma = \Sigma_n$ or $\Sigma_\infty$, and its bootstrap analogue $\Sigma^\boot$. This choice does not alter the bootstrap recursion \eqref{eq:sgd_bootstrap}, but only influences the rates of bootstrap approximation.  

\citet{JMLR:v19:17-370} selected $\mathcal{N}(0,\Sigma_\infty)$ for their asymptotic analysis, and a similar approach was used in \citep[Theorem~3]{samsonov2024gaussian} for the LSA setting. However, as shown in \Cref{cor:berry-esseen}, this choice implies that the rate of normal approximation in bootstrap world is not faster than $n^{-1/4}$. In contrast, \Cref{th:bound_kolmogorov_dist_pr_sigma_n} and \Cref{GAR bootstrap_main} demonstrate that rates up to $n^{-1/2}$ are attainable by taking $\Sigma = \Sigma_n$ in diagram~\ref{fig:gaussian_diagram} and using the corresponding bootstrap analogue. To conclude the proof, we apply the Gaussian comparison inequality; see \Cref{app:proof_th_main_1}. A full proof of \Cref{th:bootstrap_validity} is given in \Cref{sec: bootstrap validity section}.
\end{proof}
\paragraph{Discussion.} 
In \citep{samsonov2024gaussian}, a counterpart of \Cref{th:bootstrap_validity} was established for the LSA algorithm with an approximation rate of order $n^{-1/4}$, up to logarithmic factors. This rate is suboptimal, since the authors employed $\mathcal{N}(0,\Sigma_{\infty})$ as the Gaussian approximation when proving bootstrap validity. A more recent work \citep{wu2024statistical} improved the rate to $n^{-1/3}$ for the temporal-difference (TD) learning procedure with linear function approximation. Their approach relies on a direct estimate of $\Sigma_{\infty}$ and yields a rate of order $n^{-1/3}$ when approximating the distribution of $\sqrt{n}(\bar{\theta}_n - \thetas)$ by $\mathcal{N}(0,\hat{\Sigma}_n)$ with a suitably constructed estimator $\hat{\Sigma}_n$; see \citep[Theorems~3.4 and~3.5]{wu2024statistical}. \citet{chen2020aos} proposed a plug-in estimator $\hat{\Sigma}_n$ of $\Sigma_{\infty}$ and showed that
\[
\textstyle 
\PE\!\bigl[\norm{\hat{\Sigma}_n - \Sigma_{\infty}}\bigr] \lesssim C n^{-\gamma/2}, 
\quad \gamma \in (1/2,1),
\]
under weaker assumptions than those adopted in the present section. However, this bound is insufficient to establish an analogue of the Gaussian comparison result \Cref{lem:bound_kolmogorov_dist_sigma_n_sigma_infty} for $\mathcal{N}(0,\hat{\Sigma}_n)$ and $\mathcal{N}(0,\Sigma_{\infty})$ on a set of large $\PP$-probability. Achieving such a result would require high-probability bounds on $\norm{\hat{\Sigma}_n - \Sigma_{\infty}}$, which are more delicate and may necessitate additional assumptions beyond those in \citep{chen2020aos}.

Furthermore, controlling the error between the distribution of $\sqrt{n}(\bar{\theta}_n - \thetas)$ and $\mathcal{N}(0,\hat{\Sigma}_n)$ requires a Berry–Esseen type bound on the distance between $\sqrt{n}(\bar{\theta}_n - \thetas)$ and $\mathcal{N}(0,\Sigma_\infty)$. As shown in \Cref{cor:berry-esseen}, the approximation rate in this problem deteriorates as $\gamma \to 1$, creating an additional trade-off in the analysis of plug-in procedures based on estimating $\Sigma_\infty$. This phenomenon highlights the fundamental difference between the multiplier bootstrap approach and the plug-in method of \citep{chen2020aos}.

\subsection{Gaussian approximation in the real world}
\label{sec:gaus_approximation}
In this section, we relax the assumptions \Cref{ass:bound_noise} and \Cref{ass:step_size_new_boot}. We introduce a family of assumptions, denoted \Cref{ass:noise_decomposition}($p$) with $p \geq 2$, on the noise sequence $\zeta_k$, and \Cref{ass:step_size} on the step sizes $\alpha_k$:
\begin{assum}[$p$]
\label{ass:noise_decomposition}
Let $p \geq 2$.
Conditions (i) and (ii) in \Cref{ass:bound_noise} hold,
except that the moment inequality in condition (ii)
is required only at order $p$, rather than for all
orders in $[2,\log n]$.
In addition, there exists $\sigma_p > 0$ such that
\[
    \PE^{1/p}\!\left[\|\eta(\xi_1)\|^p\right]
    \leq \sigma_p
    \eqsp.
\]
\end{assum}

\begin{assum}
\label{ass:step_size}
$\alpha_k = c_0 / (k_0 + k)^{\gamma}$, where $\gamma \in (1/2,1)$, $c_0 > 0, k_0 \geq 1$, and  satisfies $2\alpha_0 L_1 \leq 1$.
\end{assum}

Clearly, \Cref{ass:step_size_new_boot} implies \Cref{ass:step_size} and  \Cref{ass:bound_noise} implies
\Cref{ass:noise_decomposition}($p$)
for every $2 \leq p \leq \log n$, with
$\sigma_p \lesssim C_{\xi}(\sqrt{d}+\sqrt{p})$.
For the results in this section, it suffices to impose
\Cref{ass:noise_decomposition}($4$).
In particular, this assumption follows from
\Cref{ass:bound_noise} whenever $n \geq e^4$.

The main result of this section establishes a Gaussian approximation for $\sqrt{n}(\bar{\theta}_n - \thetas)$ with $\mathcal{N}(0,\Sigma_n)$. This refines the bounds of \citep[Theorem~3.4]{shao2022berry} and serves as a key step toward analyzing normal approximation with $\mathcal{N}(0,\Sigma_{\infty})$ in \Cref{sec:shao_self-normalized}.

\begin{theorem}    
\label{th:bound_kolmogorov_dist_pr_sigma_n}
Assume \Cref{ass:L-smooth}, \Cref{ass:hessian_Lipschitz_ball}, \Cref{ass:noise_decomposition}($4$), and \Cref{ass:step_size}. Then, for $Y \sim \mathcal{N}(0,\Id_d)$,
\begin{equation}
\label{eq:kolm-dist}
\kolmogorov\!\left(\sqrt{n}\,\Sigma_{n}^{-1/2}(\bar{\theta}_n - \thetas),\, Y\right) 
\leq \frac{\ConstC_4}{\sqrt{n}} 
+ \frac{\ConstC_5}{n^{\gamma - 1/2}} 
+ \frac{\ConstC_6}{n^{\gamma/2}} \eqsp,
\end{equation}
where $\ConstC_4,\ConstC_5,\ConstC_6$ are defined in \Cref{sec:Gar_real_world_proof}, equation~\eqref{eq:def_Const_M_3_i}. Moreover, since $\Sigma_n$ is non-degenerate and the image of a convex set under a non-degenerate linear transformation is convex,
\begin{multline}
\kolmogorov\!\left(\sqrt{n}\,\Sigma_n^{-1/2}(\bar{\theta}_n - \thetas),\, Y\right) 
\\= \kolmogorov\!\left(\sqrt{n}(\bar{\theta}_n - \thetas),\, \Sigma_n^{1/2} Y\right).
\end{multline}
\end{theorem}
\begin{remark}
The constants $\ConstC_4,\ConstC_5,\ConstC_6$ in \Cref{th:bound_kolmogorov_dist_pr_sigma_n} depend on the problem dimension $d$ and on the parameters specified in \Cref{ass:L-smooth}--\Cref{ass:noise_decomposition}($4$)--\Cref{ass:hessian_Lipschitz_ball}--\Cref{ass:step_size}. Moreover, $\ConstC_4$ depends on $\norm{\theta_0 - \thetas}$. To make the dependence on $d$ explicit, we assume the natural scaling $\sigma_2,\ \sigma_4,\ \norm{\theta_0 - \thetas} \lesssim \sqrt{d}$. Under this assumption, \Cref{th:bound_kolmogorov_dist_pr_sigma_n} implies
\begin{equation}
   \kolmogorov\!\left(\sqrt{n}\,\Sigma_{n}^{-1/2}(\bar{\theta}_n - \thetas),\, Y\right)  
   \lesssim \frac{d^2}{\sqrt{n}} + \frac{d^{3/2}}{n^{\gamma - 1/2}} + \frac{d^{3/2}}{n^{\gamma/2}} \eqsp.
\end{equation}
Here, the notation $\lesssim$ indicates that the inequality holds up to constants independent of $n$ and $d$.
\end{remark}

\begin{remark} 
When $\gamma \to 1$, the correction terms in \eqref{eq:kolm-dist} scale as $\mathcal{O}(1/\sqrt{n})$, yielding an overall approximation rate approaching $1/\sqrt{n}$. For $\gamma \in (0,1)$, we have $1/n^{\gamma/2} < 1/n^{\gamma - 1/2}$, so the term $\ConstC_5/n^{\gamma - 1/2}$ dominates. We retain both terms in \eqref{eq:kolm-dist}, as they correspond to the moments of the statistics
\[
\frac{1}{\sqrt{n}} \sum_{i=1}^{n-1} Q_i H(\theta_{i-1}) 
\quad \text{and} \quad 
\frac{1}{\sqrt{n}} \sum_{i=1}^{n-1} Q_i g(\theta_{i-1}, \xi_i).
\]
 Moreover, the first of these has nonzero mean, since $H(\theta_{i-1})$ is quadratic in $\norm{\theta_{i-1} - \thetas}^2$. In the case of constant step-size SGD, this term can be corrected using the Richardson–Romberg technique \citep{durmus2020biassgd,SheshukovaICLR}. However, it is unclear if this technique can be extended to the diminishing step-size regime.
\end{remark}
\begin{proof}[Proof sketch of \Cref{th:bound_kolmogorov_dist_pr_sigma_n}] 
The starting point is the decomposition \eqref{eq:linear and nonlinear terms}, which expresses 
$\sqrt{n}(\bar{\theta}_n - \thetas) = W + D$ as the sum of a linear term $W$ and a nonlinear remainder $D$. Establishing a Gaussian approximation for this statistic therefore reduces to analyzing the interplay between the leading linear part and the error term. Our analysis follows the framework of \citep{shao2022berry}. Consider independent random variables $X_1,\ldots,X_n$ taking values in a measurable space $\mathcal{X}$, and a $d$-dimensional statistic $T=T(X_1,\ldots,X_n)$ that admits the decomposition
\begin{align}
\textstyle
W = \sum_{\ell=1}^n Z_\ell, \quad D= T - W,   \label{eq:W-def}
\end{align}
with $Z_\ell = r_\ell(X_\ell)$ for measurable maps $r_\ell : \mathcal{X} \to \rset^d$. The component $W$ captures the linear structure, while $D$ accounts for the nonlinear correction which is treated as an error term, assumed to be ``small'' relative to $W$ in an appropriate sense.  

Suppose $\PE[Z_\ell]=0$ and $\sum_{\ell=1}^n \PE[Z_\ell Z_\ell^\top] = \Id_d$. Let  
$\Upsilon_n = \sum_{\ell=1}^n \PE[\|Z_\ell\|^3]$.
Then, for $Y \sim \mathcal{N}(0,\Id_d)$, we have
\begin{multline}
\label{eq:shao_zhang_bound_main}
\kolmogorov(T, Y) 
\leq 259 d^{1/2} \Upsilon_n 
+ 2 \PE[\|W\| \|D\|] \\
+ 2 \sum_{\ell=1}^n \PE[\|Z_\ell\| \,\|D - D^{(\ell)}\|],
\end{multline}
where $D^{(\ell)} = D(X_1,\ldots,X_{\ell-1},X_\ell',X_{\ell+1},\ldots,X_n)$ and $X_\ell'$ is an independent copy of $X_\ell$. This result follows from \citep[Theorem~2.1]{shao2022berry}. The bound \eqref{eq:shao_zhang_bound_main} extends to the case $\sum_{\ell=1}^n \PE[Z_\ell Z_\ell^\top] = \Sigma \succ 0$, as shown in \citep[Corollary~2.3]{shao2022berry}. To apply this result, we take $X_i = \xi_i$, $Z_\ell = h(X_\ell)$, and let $\xi_i'$ be an i.i.d.\ copy of $\xi_i$. The key step is to bound $\PE^{1/2}[\|D(\xi_1,\ldots,\xi_{n-1})\|^2]$ and $\PE^{1/2}[\|D - D_i'\|^2]$. Detailed proof is given in \Cref{sec:Gar_real_world_proof}.
\end{proof}

\subsection{Gaussian approximation in the bootstrap world} 
\label{GAR boot world sec}
In the main result of this section, we study the Gaussian approximation for $\sqrt{n}(\bar{\theta}_{n}^\boot - \bar{\theta}_n)$ w.r.t. $\PPb$, where the target distribution is an appropriately chosen normal law. Although this result is similar in its nature to \Cref{th:bound_kolmogorov_dist_pr_sigma_n}, it requires to address additional challenges that arise in the "bootstrap world". Our first steps are the same as in \eqref{eq: sgd_reccurence_with_remainder_main} and \eqref{eq:sgd_bootstrap}:
\begin{equation}
\begin{split}
&\theta_{k}^{\boot}-\theta_{k}  = (\Id-\alpha_{k}G)(\theta_{k-1}^{\boot}-\theta_{k-1}) \\
&- \alpha_k\bigl(H(\theta_{k-1}^{\boot}) +g(\theta_{k-1}^{\boot}, \xi_k) \\& \qquad \qquad \qquad-  H(\theta_{k-1}) -g(\theta_{k-1}, \xi_k)\bigr)\\&- \alpha_k(w_k-1)(G(\theta_{k-1}^{\boot}-\thetas)+\eta(\xi_k) \\& \qquad \qquad \qquad+ g(\theta_{k-1}^{\boot}, \xi_k) + H(\theta_{k-1}^{\boot}))\eqsp.   
\end{split}
\end{equation}
Taking average of the above identity and rearranging the terms, we obtain a counterpart of \eqref{eq:linear and nonlinear terms}:
\begin{equation}
\label{eq:D_boot-def}
\begin{split}
&\sqrt{n}(\prtheta_n^{\boot}-\prtheta_{n}) = W^\boot + D^\boot \eqsp, \\  
&W^\boot = -\frac{1}{\sqrt{n}} \sum_{i=1}^{n-1}(w_i-1) Q_{i} \eta(\xi_i) \eqsp, \\   
&D^\boot = \sqrt{n}(\prtheta_n^{\boot}-\prtheta_{n}) - W^\boot \eqsp. 
\end{split}
\end{equation}
Here $W^\boot$ is a weighted sum of i.i.d. random variables $\Xi^{n-1}$, such that  $\PEb[W^\boot] = 0$, $\PEb[W^\boot \{W^\boot\}^{\top}] = \Sigma_{n}^\boot$, where
\begin{equation}
\label{Sigman boot def}
\Sigma_{n}^\boot = n^{-1} \sum_{i=1}^{n-1} Q_i \eta (\xi_i) \eta(\xi_i)^\top Q_i^\top\eqsp,
\end{equation}
and $D^\boot$ is a non-linear statistic of $\Xi^{n-1}$. The main difficulty arises when analyzing the conditional distribution of $\sqrt{n}(\bar{\theta}_{n}^\boot - \bar{\theta}_n)$ given the data $\Xi^{n-1}$. The approach of \citep{shao2022berry} requires controlling the second moments of $D^\boot$ and $D^\boot - \{D^\boot\}^{(i)}$ with respect to the bootstrap measure $\PPb$, on a high-probability event under the original measure $\PP$. At the same time, the martingale structure of the summands in $D^\boot$ is lost unless we condition on an extended filtration.
\begin{equation}
\label{eq: extended filtration}
\textstyle 
\widetilde{\F}_i = \sigma(w_1,\ldots w_i, \xi_1,\ldots \xi_i) \eqsp, 1 \le i \le n-1 \eqsp.
\end{equation} 
Therefore, it is not clear whether the approach of \citep{shao2022berry}, discussed in \Cref{sec:gaus_approximation}, can be applied directly. Instead, we rely on a linearization method that exploits high-order moment bounds for the remainder term $D^\boot$; see \Cref{prop:prob-D-boot-bound} in \Cref{subsec:GAR in the bootstrap world}. This necessity motivates uniform sub-Gaussian assumption \Cref{ass:bound_noise}. We now state the main result of this section:
\begin{theorem}
\label{GAR bootstrap_main}
Assume \Cref{ass:L-smooth} - \Cref{ass:n_lower_bound}. Then with $\PP$ - probability at least $1 - 2/n$, it holds
\begin{multline}
\sup_{B \in \Conv(\rset^d)}|\PPb( \sqrt n\{\Sigma_n^\boot\}^{-\frac{1}{2}}   (\bar{\theta}_{n}^\boot - \bar{\theta}_n) \in B) - \PPb(Y^\boot \in B)| 
\\ \leq  \frac{M_{3,1}^\boot}{n^{1/2}}  +  \frac{M_{3,2}^\boot  \log n}{n^{\gamma-1/2}}  + \frac{M_{3,3}^\boot \log^{3/2} n}{n^{\gamma/2}} \eqsp,
\end{multline}
where $Y^\boot \sim \mathcal{N}(0, \Id)$ under $\PPb$ and $\{M_{3,i}^\boot\}_{i=1}^3$ are defined in Appendix \ref{subsec:GAR in the bootstrap world}, equation \eqref{M_3i boot def}. 
\end{theorem}
\begin{proof}[Proof sketch of \Cref{GAR bootstrap_main}] We apply the bound
\begin{align}
&\textstyle  \underset{B \in \Conv(\rset^d)}{\sup}|\PPb(\{\Sigma_n^\boot\}^{-\tfrac{1}{2}} (W^\boot + D^\boot) \in B) - \PPb(Y^\boot \in B)| \nonumber  \\ 
&\textstyle \leq \sup_{B \in \Conv(\rset^d)}|\PPb(\{\Sigma_n^\boot\}^{-\tfrac{1}{2}} W^\boot \in B) - \PPb(Y^\boot \in B)| \nonumber \\
&\textstyle \hspace{50pt}+ 2 c_d (\PEb[\norm{\{\Sigma_n^\boot\}^{-\tfrac{1}{2}} D^\boot}^p])^{\tfrac{1}{1+p}}\eqsp, \label{crude bound for nonlinear stat} 
\end{align}
where $c_d \le 4 d^{1/4}$ is the isoperimetric constant of the class of convex sets, see e.g. \citep{bentkus2004}. The proof of \eqref{crude bound for nonlinear stat} is provided in \Cref{prop:prob-D-boot-bound} in \Cref{subsec:GAR in the bootstrap world}. 
We first control $\bar{\PE}[\norm{D^\boot}^p]$ using Burkholder's inequality, where $\bar{\PE}$ denotes expectation with respect to the product measure $\PP_\xi^{\otimes n} \otimes \PP_w^{\otimes n}$. Applying Markov's inequality then yields $\PP$--high-probability bounds for the behavior of $\PEb[\norm{D^\boot}^p]$. This step requires bounds on $\bar{\PE}^{1/p}\!\left[\norm{\theta_k - \thetas}^{p}\right]$  and $\bar{\PE}^{1/p}\!\left[\norm{\theta_k^{\boot} - \thetas}^{p}\right]$,
$k \in \{1,\ldots,n-1\}$, where $p = c \log n$ for some absolute constant $c$, with polynomial dependence on $p$.

To control the first term on the right-hand side of \eqref{crude bound for nonlinear stat}, we note that $\Sigma_n^\boot$ concentrates around $\Sigma_n$ by the matrix Bernstein inequality (see \Cref{lem:matrix_bernstein}). Hence, there exists a set $\Omega_1$ with $\PP(\Omega_1) \geq 1 - 1/n$ such that $\lambda_{\min}(\Sigma_n^\boot) > 0$ on $\Omega_1$. On this event, one can apply Berry–Esseen type bounds for non-i.i.d.\ sums of random vectors. The full proof is presented in \Cref{subsec:GAR in the bootstrap world}.
\end{proof}

\begin{remark}
The constants $M_{3,1}^\boot, M_{3,2}^\boot, M_{3,3}^\boot$ in \Cref{GAR bootstrap_main} depend on the problem dimension $d$, as well as $\|\theta_0 - \thetas\|$. 
To make the dependence on $d$ explicit, we assume the natural scaling $\|\theta_0 - \thetas\| \lesssim \sqrt{d}$. We also assume that $C_{\xi}$ from \Cref{ass:bound_noise}-(iii) is dimension-free. Under this assumption, \Cref{GAR bootstrap_main} implies
\begin{multline}
\sup_{B \in \Conv(\rset^d)}
\Big| \PPb\!\big( \sqrt{n}\{\Sigma_n^\boot\}^{-\frac{1}{2}}(\bar{\theta}_{n}^\boot - \bar{\theta}_n) \in B \big) 
- \PPb(Y^\boot \in B) \Big|
\\ \lesssim\ 
\frac{d^2}{\sqrt{n}}
+ \frac{d^{5/4}\log n}{n^{\gamma - 1/2}}
+ \frac{d^{3/4}(\log n)^{3/2}}{n^{\gamma/2}} \eqsp.
\end{multline}
Here, the notation $\lesssim$ indicates that the inequality holds up to constants independent of $n$ and $d$.
\end{remark}

\subsection{Rate of convergence in the Polyak--Juditsky central limit theorem}
\label{sec:shao_self-normalized}
We finally address the change from $\Sigma_n$ to $\Sigma_\infty$ and derive convergence rates in the Polyak--Juditsky result \eqref{eq:CLT_fort}. Our argument builds on \Cref{th:bound_kolmogorov_dist_pr_sigma_n} together with the following lemma. 
\par 
\begin{lemma}
\label{lem:bound_kolmogorov_dist_sigma_n_sigma_infty}
Assume \Cref{ass:L-smooth}, \Cref{ass:step_size} and suppose that $\lambda_{\min}(\Sigma_\xi)>0$. Let $Y \sim \mathcal{N}(0,\Id_d)$. Then the convex distance between the distributions of $\Sigma_n^{1/2} Y$ and $\Sigma_\infty^{1/2} Y$ satisfies
\begin{equation}
\kolmogorov(\Sigma_n^{1/2} Y, \Sigma_\infty^{1/2} Y) 
\leq C_\infty n^{\gamma - 1} \eqsp,
\end{equation}
where the constant $C_\infty$ is defined in \eqref{eq:def_C_infty}.
\end{lemma}
Combining \Cref{th:bound_kolmogorov_dist_pr_sigma_n} with \Cref{lem:bound_kolmogorov_dist_sigma_n_sigma_infty} and applying the triangle inequality yields the following result on the closeness of $\sqrt{n}(\bar{\theta}_n - \thetas)$ to $\mathcal{N}(0,\Sigma_\infty)$.
\begin{theorem}
\label{cor:berry-esseen}
Assume \Cref{ass:L-smooth}, \Cref{ass:hessian_Lipschitz_ball}, \Cref{ass:noise_decomposition}($4$), \Cref{ass:step_size}. Then, with $Y \sim \mathcal{N}(0,\Id_d)$ it holds that 
    \begin{equation}
    \label{eq:Berry-Esseen_Bound_Sigma_infty}
\textstyle        \kolmogorov(\sqrt{n}(\bar\theta_n -\thetas), \Sigma_\infty^{1/2} Y)  \leq 
        \frac{\ConstC_4}{\sqrt{n}} + \frac{\ConstC_5 + \ConstC_{6}}{n^{\gamma - 1/2}} + \frac{C_\infty}{ n^{1 - \gamma}} \eqsp,
   \end{equation}
   where $\ConstC_4, \ConstC_5$ and $\ConstC_6$ are given in Theorem \ref{th:bound_kolmogorov_dist_pr_sigma_n}. 
\end{theorem}

\vspace{-2mm}
\paragraph{Discussion.} 
\Cref{th:bound_kolmogorov_dist_pr_sigma_n} shows that the normal approximation with $\mathcal{N}(0,\Sigma_n)$ improves as the step sizes $\alpha_k$ become less aggressive, i.e., as $\gamma \to 1$. At the same time, \Cref{cor:berry-esseen} highlights a trade-off: the rate at which $\Sigma_n$ converges to $\Sigma_\infty$ also influences the overall approximation quality. Optimizing the bound in \eqref{eq:Berry-Esseen_Bound_Sigma_infty} with respect to $\gamma$ yields the optimal choice $\gamma = 3/4$, which leads to the approximation rate
\begin{multline}
\label{eq:optimal_Berry-Esseen_Bound_Sigma_infty} 
\kolmogorov\!\left(\sqrt{n}(\bar{\theta}_n - \thetas), \Sigma_\infty^{1/2} Y\right) \\ 
\leq \frac{\ConstC_1^{\prime}}{n^{1/4}} 
+ \frac{\ConstC_2^{\prime}}{\sqrt{n}}\big(\norm{\theta_0-\thetas}+\norm{\theta_0-\thetas}^2\big) \eqsp,
\end{multline}
where $\ConstC_1^{\prime}$ and $\ConstC_2^{\prime}$ are instance-dependent constants (independent of $\norm{\theta_0-\thetas}$) that can be derived from \Cref{cor:berry-esseen}. This result enables a non-asymptotic assessment of methods for constructing confidence intervals based on direct estimation of $\Sigma_\infty$, such as those studied in \citep{chen2020aos,zhu2023online_cov_matr}.
\vspace{-2mm}
\paragraph{Lower bounds}
We provide a lower bound showing that the bound in \Cref{cor:berry-esseen} is tight in certain regimes of step size decay power $\gamma \in (1/2, 1)$. To this end, we consider the minimization problem \eqref{eq:stoch_minimization} with $f(\theta) = \theta^2/2$ and $\theta_0 = 0$. In this case, $\thetas = 0$. We use an additive noise model, where the stochastic gradient oracles $F(\theta, \xi)$ are given by $F(\theta, \xi) = \theta + \xi$, $\xi \sim \mathcal{N}(0, 1)$. Unrolling \eqref{eq:sgd_recursion_main}, we get
\begin{equation}
\label{eq:theta_k_recursion}
\textstyle 
\sqrt{n}\bar\theta_{n} = -n^{-1/2} \sum_{j=1}^{n-1}Q_{j} \xi_{j}
\end{equation}
with $Q_{j} = \alpha_{j}\sum_{k=j}^{n-1}\prod_{\ell=j+1}^{k}(1-\alpha_{\ell})$, showing that
$\sqrt{n}(\bar\theta_{n} - \thetas) \sim \mathcal{N}(0,\sigma^2_{n,\gamma})$ with $\sigma^2_{n,\gamma} = n^{-1} \sum_{j=1}^{n-1}Q_{j}^2$. 
\Cref{lem:bound_kolmogorov_dist_sigma_n_sigma_infty} (see also \eqref{eq:difference_sigma_n_infty_appendix} in the Appendix), we have $G = 1, \Sigma_{\infty} = 1$, and $\sigma^2_{n,\gamma} \to 1$ as $n \to \infty\eqsp$. Moreover, the following lower bound holds:
\begin{theorem}
\label{prop:lower_bounds}
Consider the sequence $\{\theta_k\}_{k \in \nset}$ defined by the recurrence \eqref{eq:theta_k_recursion} with $\alpha_{j} = c_0/(1+j)^{\gamma}$. Then it holds, for the number of observations $n$ sufficiently large, that 
\begin{equation}
\label{eq:variance_lower_bound}
\textstyle
|\sigma^2_{n,\gamma} - 1| > \frac{C_1(\gamma,c_0)}{n^{1-\gamma}}\eqsp, 
\end{equation}
where the constant $C_1(\gamma,c_0)$ depends only upon $c_0$ and $\gamma$. Moreover, for $n$ large enough
\begin{equation}
\label{eq:Kolmogorov_lower_bound}
\textstyle
\kolmogorov(\sqrt{n}(\bar\theta_{n} - \thetas), \mathcal{N}(0,1)) > \frac{C_2(\gamma,c_0)}{n^{1-\gamma}}\eqsp.
\end{equation}
\end{theorem}
\textbf{Discussion.} Proof of \Cref{prop:lower_bounds} is provided in \Cref{sec:proof_lower_bound}, along with simple numerical simulations indicating the tightness of the lower bound \eqref{eq:variance_lower_bound}. Note that the bound \eqref{eq:Kolmogorov_lower_bound} shows that the distribution of $\sqrt{n}(\bar\theta_{n} - \thetas)$ cannot be approximated by $\mathcal{N}(0,\Sigma_{\infty})$ at a rate faster than $1/n^{1-\gamma}$. Moreover, it shows that the rate of normal approximation in \Cref{cor:berry-esseen} cannot be improved when $\gamma \in [3/4, 1)$. This fact is extremely important when considering the bootstrap validity result in \Cref{th:bootstrap_validity} and the normal approximation in \Cref{th:bound_kolmogorov_dist_pr_sigma_n}. Indeed, both results suggest that normal approximation rates of order up to $1/\sqrt{n}$ can be achieved as $\gamma \to 1$, but they require using a different covariance matrix $\Sigma_n$, corresponding to the linearized recurrence in \eqref{eq:sigma_n_def}. At the same time, in the regime $\gamma \to 1$, the approximation by $\mathcal{N}(0, \Sigma_{\infty})$ can be too slow. It is an interesting and, to the best of our knowledge, open question to provide lower bounds analogous to \Cref{prop:lower_bounds} which show the tightness of other summands in \Cref{cor:berry-esseen} in the regime $1/2 < \gamma < 3/4$.
\section{CONCLUSION}
In our paper, we performed the fully non-asymptotic analysis of the multiplier bootstrap procedure for SGD applied to strongly convex minimization problems. We showed that the algorithm can achieve approximation rates in convex distances of order up to $1/\sqrt{n}$. We highlight the fact that the validity of the multiplier bootstrap procedure does not require one to consider Berry-Esseen bounds with the asymptotic covariance matrix $\Sigma_{\infty}$, which is in sharp contrast to the methods that require direct estimation of $\Sigma_{\infty}$.

\section*{Acknowledgment}
This work is an output of a research project HSE-BR-2025-019 implemented as part of the Basic Research Program at HSE University. This research was supported in part through computational resources of HPC facilities at HSE University \citep{kostenetskiy2021hpc}.

\clearpage
\newpage

\bibliographystyle{plainnat}
\bibliography{references}

\section*{Checklist}


\begin{enumerate}

  \item For all models and algorithms presented, check if you include:
  \begin{enumerate}
    \item A clear description of the mathematical setting, assumptions, algorithm, and/or model. Answer: Yes. 
    All main results of this submission are supported with rigorous assumptions and proofs. Assumptions are presented in \Cref{sec:bootstrap} and \Cref{sec:gaus_approximation}, as well as the statements of the main theorems. Proofs are provided in Appendix.
    
    \item An analysis of the properties and complexity (time, space, sample size) of any algorithm. Answer: Yes. The paper is devoted to the analysis of SGD algorithm and rates of normal approximation for the SGD algorithm. All theoretical results are mathematically justified. At the same time, the paper does not provide a new algorithm.
    
    \item (Optional) Anonymized source code, with specification of all dependencies, including external libraries. Answer: Yes. Code to reproduce experiments will be presented at the anonymous github.
  \end{enumerate}

  \item For any theoretical claim, check if you include:
  \begin{enumerate}
    \item Statements of the full set of assumptions of all theoretical results. Answer: Yes. All theoretical results are stated with explicit pointers to the underlying assumptions (see \Cref{sec:bootstrap} and \Cref{sec:gaus_approximation})
    \item Complete proofs of all theoretical results. Answer: Yes. Detailed proof of each theorem stated in the main text is provided in the appendix with clear references to the relevant sections. 
    \item Clear explanations of any assumptions. Answer: Yes. A detailed discussion of each of the assumptions is provided in \Cref{sec:bootstrap}.   
  \end{enumerate}

  \item For all figures and tables that present empirical results, check if you include:
  \begin{enumerate}
    \item The code, data, and instructions needed to reproduce the main experimental results (either in the supplemental material or as a URL). Answer: Yes. All code is open source, link to an anonymous GitHub repository is included. 
    \item All the training details (e.g., data splits, hyperparameters, how they were chosen). Answer: Yes. Numerical results are stated with a complete description of the environments that are used, as well as the precise sets of hyperparameters that we used. The code (in Python) is provided as supplementary with the paper, making it easy for one to reproduce our numerical experiments.
    \item A clear definition of the specific measure or statistics and error bars (e.g., with respect to the random seed after running experiments multiple times).  Answer: Yes.
    \item A description of the computing infrastructure used. (e.g., type of GPUs, internal cluster, or cloud provider). Answer: Yes. All necessary information to reproduce experiments is provided in  Appendix.
  \end{enumerate}

  \item If you are using existing assets (e.g., code, data, models) or curating/releasing new assets, check if you include:
  \begin{enumerate}
    \item Citations of the creator If your work uses existing assets. Answer: Not applicable. 
    \item The license information of the assets, if applicable. Answer: Not applicable.
    \item New assets either in the supplemental material or as a URL, if applicable. Answer: Not applicable.
    \item Information about consent from data providers/curators. Answer: Not Applicable.
    \item Discussion of sensible content if applicable, e.g., personally identifiable information or offensive content. Answer: Not applicable.
  \end{enumerate}

  \item If you used crowdsourcing or conducted research with human subjects, check if you include:
  \begin{enumerate}
    \item The full text of instructions given to participants and screenshots. Answer: Not applicable.
    \item Descriptions of potential participant risks, with links to Institutional Review Board (IRB) approvals if applicable. Answer: Not applicable.
    \item The estimated hourly wage paid to participants and the total amount spent on participant compensation. Answer: Not applicable.
  \end{enumerate}

\end{enumerate}

\clearpage
\appendix
\thispagestyle{empty}

\onecolumn
\tableofcontents

\section{Proof of Theorem \ref{th:bootstrap_validity}}
\label{sec: bootstrap validity section}
We begin this section with explicit bound on the number of observations $n$ stated in \Cref{ass:n_lower_bound}. 
\setcounter{assumprime}{5}
\begin{assumprime}
\label{ass:n_lower_bound_full}
Number of observations $n$ satisfies $n \geq \rme^4$ and $\frac{n}{d\log(9) +\log(2n)}\geq \max(1, ((32\sqrt{2}+16)C_{Q, \xi}C_{\Sigma}^2)^2)$, where the constants $C_{Q, \xi}$ and $C_\Sigma$ are defined in \eqref{eq:const_C_Q_xi_def} and \eqref{eq:def_C_Sigma}, respectively.
\end{assumprime}

\subsection{Proof of Theorem \ref{th:bootstrap_validity}}
\label{app:proof_th_main_1}

Recall that our proof of non-asymptotic bootstrap validity is based on the Gaussian approximation performed in both the “real” world and the bootstrap world, combined with the  Gaussian comparison inequality, that is:
\begin{figure}[h]
\centering
\begin{tikzcd}[column sep = 110pt]
\label{eq:diagram_appendix}
    \text{Real world: \quad\quad} \sqrt n  (\bar{\theta}_{n} - \thetas) \arrow[<->]{r}{\text{Gaussian approximation, Th. \ref{th:bound_kolmogorov_dist_pr_sigma_n}}}  & \mathcal N(0,  \Sigma_n)  \arrow[<->]{d}{\text{Gaussian comparison} } 
\\
\text{Bootstrap world: } \sqrt n  (\bar{\theta}_{n}^\boot - \bar{\theta}_n) \arrow[<->]{r}{\text{Gaussian approximation, Th. \ref{GAR bootstrap_main}}} & \mathcal N(0,   \Sigma_n^\boot )\eqsp, \\ 
\end{tikzcd}
\end{figure}
\\
where $\Sigma_n$ and $\Sigma_n^\boot$ are defined in \ref{eq:sigma_n_def} and \ref{Sigman boot def} respectively.
Recall that $Y \sim \mathcal{N}(0, \Id)$ under $\PP$ and $Y^\boot \sim \mathcal{N}(0, \Id)$ under $\PPb$.

To formalize the scheme described above, we consider the following decomposition
\begin{equation}
    \sup_{B \in \Conv(\rset^{d})} |\PPb(\sqrt n (\bar{\theta}_{n}^\boot - \bar{\theta}_n) \in B ) - \PP(\sqrt n (\bar{\theta}_n - \thetas) \in B)| \leq T_1 + T_2 + T_3 \eqsp,
\end{equation}
where 
\begin{align}
    &T_1 =  \sup_{B \in \Conv(\rset^{d})} |\PP(\sqrt n (\bar{\theta}_n - \thetas)\in B)- \PP(\Sigma_n^{1/2}Y\in B)|\eqsp,\\
    &T_2 = \sup_{B \in \Conv(\rset^d)}|\PPb( \sqrt n   (\bar{\theta}_{n}^\boot - \bar{\theta}_n) \in B) - \PPb(\{\Sigma_n^\boot\}^{1/2}Y^\boot \in B)|\eqsp,\\
    & T_3 =  \sup_{B \in \Conv(\rset^{d})} |\PPb(\{\Sigma_n^\boot\}^{1/2}Y^\boot \in B)-\PP(\Sigma_n^{1/2}Y\in B)|\eqsp.
\end{align}
Hence, from \Cref{th:bound_kolmogorov_dist_pr_sigma_n}  (see the proof in \ref{sec:Gar_real_world_proof})
it follows 
\begin{equation}
    T_1 \leq \frac{\ConstC_4}{\sqrt{n}}   +  \frac{\ConstC_5}{n^{\gamma - 1/2}}+ \frac{\ConstC_6}{n^{\gamma/2}}\eqsp,
\end{equation}
where $\ConstC_4, \ConstC_5, \ConstC_6$ are given in equation \eqref{eq:def_Const_M_3_i}.  Note that $\ConstC_4, \ConstC_5, \ConstC_6$ depend on the constants $\sigma_2$ and $\sigma_4$ from \Cref{ass:noise_decomposition}. However, under \Cref{ass:bound_noise}(iii), this condition holds with $ \sigma_2, \sigma_4 \lesssim C_{\xi}\sqrt{d}.$

Recall that 
\begin{equation}
    \sqrt n   (\bar{\theta}_{n}^\boot - \bar{\theta}_n) = W^\boot + D^\boot\eqsp,
\end{equation}
where $W^\boot$ is a linear statistic of $\Xi^{n-1}$ with $\PEb[W^\boot\{W^\boot\}^{\top}] = \Sigma_n^b$ and $D^\boot$ is a non-linear statistic of $\Xi^{n-1}$. 

To control $T_2$ and $T_3$, set $p=\log n-1$ and consider the events
\begin{align}
\Omega_1
&=\left\{
\|\Sigma_n^b-\Sigma_n\|
\leq (16\sqrt{2}+8)C_{Q,\xi}
\sqrt{\frac{d\log 9+\log(2n)}{n}}
\right\},\\
\Omega_\eta
&=\left\{
\frac1n\sum_{i=1}^{n-1}\|\eta(\xi_i)\|^3
\leq C_{1,\xi}^3
\right\},\\
\Omega_D
&=\left\{
\bigl(\mathbb E^b[\|D^b\|^p]\bigr)^{1/p}
\leq (2n)^{1/p}\rme^{1/p}
\left(
M_{1,1}^b p^{3/2}n^{-\gamma/2}
+M_{2,1}^b p n^{1/2-\gamma}
\right)
\right\}.
\end{align}
Here $C_{Q,\xi}$ is defined in \eqref{eq:const_C_Q_xi_def}, $C_{1,\xi}$ is defined
in \eqref{def:C_1_xi} and $M_{1,1}^b,M_{2,1}^b$ are given in \eqref{eq:def_M_11_M_21_boot}.

By \Cref{lem:matrix_bernstein}, $\mathbb P(\Omega_1^c)\leq 1/n$.
The bounds established in the proof of \Cref{GAR bootstrap_main} (see \ref{subsec:GAR in the bootstrap world}) give
$\mathbb P(\Omega_\eta^c)\leq 1/(2n)$ and
$\mathbb P(\Omega_D^c)\leq 1/(2n)$. Hence
\[
\Omega=\Omega_1\cap\Omega_\eta\cap\Omega_D
\qquad\text{satisfies}\qquad
\mathbb P(\Omega)\geq 1-\frac{2}{n}.
\]

On $\Omega$, the matrix $\Sigma_n^b$ is positive definite
by \Cref{lambda min boot}. Since invertible linear maps preserve convexity,
\Cref{GAR bootstrap_main} yields
\[
T_2\leq \frac{M_{3,1}^b}{\sqrt n}
       +\frac{M_{3,2}^b\log n}{n^{\gamma-1/2}}
       +\frac{M_{3,3}^b(\log n)^{3/2}}{n^{\gamma/2}},
\]
where $M_{3,1}^b,M_{3,2}^b,M_{3,3}^b$ are defined in \eqref{M_3i boot def}.

To complete the proof, we need to bound the convex distance between the two Gaussians.

Since $\Omega\subseteq\Omega_1$,  \Cref{Pinsker}, \Cref{lem:bound_Q_i_and_Sigma_n} and \Cref{lem:matrix_bernstein} yield
\begin{align}
T_3
&\leq \frac32
 \bigl\|\Sigma_n^{-1/2}(\Sigma_n^b-\Sigma_n)
                 \Sigma_n^{-1/2}\bigr\|_F\\
&\leq \frac32\sqrt d\,C_\Sigma^2\|\Sigma_n^b-\Sigma_n\|\\
&\leq (24\sqrt2+12)C_{Q,\xi}C_\Sigma^2
       \sqrt{\frac{d\{d\log 9+\log(2n)\}}{n}}.
\end{align}

Collecting previous bounds and using triangle inequality we get that on the set $\Omega$ with $\PP(\Omega)\geq 1-2/n$, it holds:
\begin{multline}
  \sup_{B \in \Conv(\rset^{d})} |\PPb(\sqrt n (\bar{\theta}_{n}^\boot - \bar{\theta}_n) \in B ) - \PP(\sqrt n (\bar{\theta}_n - \thetas) \in B)| \le   \frac{\ConstC_1 \sqrt{\log n}}{n^{1/2}} + \frac{\ConstC_2 \log n}{n^{\gamma - 1/2} } +  \frac{\ConstC_3 \log^{3/2} n}{n^{\gamma/2}} \eqsp, 
\end{multline}
where 
\begin{align}
\label{eq: bootsrap constants def}
C_1
&=C_4+M_{3,1}^b
 +(24\sqrt2+12)C_{Q,\xi}C_\Sigma^2
   \sqrt{d\left(1+\frac{d\log9+\log2}{4}\right)},\\
C_2&=C_5+M_{3,2}^b,\\
C_3&=C_6+M_{3,3}^b.
\end{align}

\subsection{Matrix Bernstein inequality for $
\Sigma_n^\boot$ and Gaussian comparison}

\begin{lemma}
\label{lem:matrix_bernstein}
Under assumptions \Cref{ass:L-smooth}, \Cref{ass:bound_noise}, \Cref{ass:step_size_new_boot}, \Cref{ass:n_lower_bound}, there is a set $\Omega_1 \in \F_{n-1}$, such that $\PP(\Omega_1) \geq 1 - 1/n$ and on $\Omega_1$ it holds that
\begin{equation}
\norm{\Sigma_n^b-\Sigma_n}
\leq
(16\sqrt{2}+8)C_{Q,\xi}
\sqrt{\frac{d\log 9+\log(2n)}{n}}.
\end{equation}
where
\begin{equation}
\label{eq:const_C_Q_xi_def}
    C_{Q,\xi} = C_Q^2C_\xi^2
\end{equation}
\end{lemma}
\begin{proof}
Let
$$
M_n:=\Sigma_n^b-\Sigma_n
=
\frac1n\sum_{i=1}^{n-1}
Q_i\bigl(\eta(\xi_i)\eta(\xi_i)^\top-\Sigma_\xi\bigr)Q_i^\top.
$$

Since \(M_n\) is symmetric, we first control its quadratic forms.

Fix \(u\in \mathbb S^{d-1}\) and define

$$
Z_i:=\langle Q_i^\top u,\eta(\xi_i)\rangle,
\qquad i=1,\ldots,n-1.
$$

By Assumption~\Cref{ass:bound_noise}, evaluated at \(\theta=\theta^\star\), and using
\(g(\theta^\star,\xi)=0\), we have for every \(\lambda\in\mathbb R\),

$$
\PE\exp\{\lambda Z_i\}
\le
\exp\left\{
\frac{\lambda^2 C_\xi^2\|Q_i^\top u\|^2}{2}
\right\}.
$$

Moreover, by the bound \(\|Q_i\|\le C_Q\),

$$
\PE\exp\{\lambda Z_i\}
\le
\exp\left\{
\frac{\lambda^2 C_Q^2C_\xi^2}{2}
\right\}.
$$

Thus \(Z_i\) is sub-Gaussian with variance proxy

$$
\sigma^2:=C_Q^2C_\xi^2=C_{Q,\xi}.
$$

In particular,

$$
\PP(|Z_i|\ge t)
\le
2\exp\left(-\frac{t^2}{2\sigma^2}\right),
\qquad t>0.
$$

Hence, for every integer \(k\ge1\),

$$
\begin{aligned}
\PE|Z_i|^{2k}
&=
2k\int_0^\infty
t^{2k-1}\PP(|Z_i|\ge t)\,dt\\
&\le
4k\int_0^\infty
t^{2k-1}
\exp\left(-\frac{t^2}{2\sigma^2}\right)\,dt\\
&=
2^{k+1}k!\sigma^{2k}.
\end{aligned}
$$

Set

$$
Y_i:=Z_i^2-\PE Z_i^2.
$$

Then \(\PE Y_i=0\), and for every integer \(k\ge2\),

$$
\begin{aligned}
\PE|Y_i|^k
&\le
2^{k-1}
\left(
\PE|Z_i|^{2k}+(\PE Z_i^2)^k
\right)\\
&\le
2^k\PE|Z_i|^{2k}\\
&\le
2^{2k+1}k!\sigma^{2k}\\
&=
\frac{k!}{2}\,
(64\sigma^4)\,
(4\sigma^2)^{k-2}.
\end{aligned}
$$

Therefore the Bernstein moment condition holds with

$$
v_i^2=64\sigma^4,
\qquad
c=4\sigma^2.
$$

Applying Bernstein's inequality from \cite{concent_ineq}[Theorem 2.10], we get for every \(x>0\),

$$
\PP\left(
\left|
\sum_{i=1}^{n-1}Y_i
\right|
\ge
\sqrt{2Vx}+cx
\right)
\le
2e^{-x},
$$

where

$$
V:=\sum_{i=1}^{n-1}v_i^2
\le
64n\sigma^4.
$$

Consequently,

$$
\PP\left(
|u^\top M_nu|
\ge
8\sqrt{2}\,\sigma^2\sqrt{\frac{x}{n}}
+
4\sigma^2\frac{x}{n}
\right)
\le
2e^{-x}.
$$

Let \(\mathcal N\) be a \(1/4\)-net of \(\mathbb S^{d-1}\). It can be chosen such that

$$
|\mathcal N|\le 9^d.
$$

For every symmetric matrix \(A\),

$$
\|A\|
\le
2\max_{u\in\mathcal N}|u^\top Au|.
$$

Therefore, by the union bound,

$$
\PP\left(
\|M_n\|
\ge
16\sqrt{2}\,\sigma^2\sqrt{\frac{x}{n}}
+
8\sigma^2\frac{x}{n}
\right)
\le
2\cdot 9^d e^{-x}.
$$

Now choose

$$
x=d\log 9+\log(2n).
$$

Then

$$
2\cdot 9^d e^{-x}=\frac1n.
$$

Hence there exists an event \(\Omega_1\in\F_{n-1}\) such that

$$
\PP(\Omega_1)\ge 1-\frac1n
$$

and, on \(\Omega_1\),

$$
\|M_n\|
\le
16\sqrt{2}\,\sigma^2
\sqrt{\frac{d\log 9+\log(2n)}{n}}
+
8\sigma^2
\frac{d\log 9+\log(2n)}{n}.
$$

By Assumption~\Cref{ass:n_lower_bound},

$$
\frac{d\log 9+\log(2n)}{n}\le 1,
$$

and therefore

$$
\frac{d\log 9+\log(2n)}{n}
\le
\sqrt{\frac{d\log 9+\log(2n)}{n}}.
$$

Using \(\sigma^2=C_{Q,\xi}\), we conclude that, on \(\Omega_1\),

$$
\|\Sigma_n^b-\Sigma_n\|
\le
(16\sqrt{2}+8)C_{Q,\xi}
\sqrt{\frac{d\log 9+\log(2n)}{n}}.
$$

\end{proof}

\begin{corollary}
\label{lambda min boot}
Under assumptions \Cref{ass:L-smooth}, \Cref{ass:bound_noise}, \Cref{ass:step_size_new_boot}, \Cref{ass:n_lower_bound}, on $\Omega_1$ it holds that 
\begin{equation}
    \lambda_{\min}(\Sigma_n^\boot) \geq \frac{1}{2 C_\Sigma^2} \eqsp. 
\end{equation}
\end{corollary}
\begin{proof}
Using eigenvalue stability (Lidski's) inequality, we obtain 
\begin{equation}
    \lambda_{\min}(\Sigma_n^\boot) \geq \lambda_{\min}(\Sigma_n)-\norm{\Sigma_n-\Sigma_n^\boot}\eqsp.
\end{equation}
Note that on $\Omega_1$, we have 
\begin{equation}
    \norm{\Sigma_n-\Sigma_n^\boot} \leq  (16\sqrt{2}+8)C_{Q,\xi}
\sqrt{\frac{d\log 9+\log(2n)}{n}}\leq \frac{1}{2C_\Sigma^2}\eqsp,
\end{equation}
where in the last inequality we use \Cref{ass:n_lower_bound}.
\end{proof}


\subsection{Example of distribution satisfying \Cref{ass:bound_bootstap_weights}}
\label{sec:example_distribution}
To construct examples of distributions satisfying the above assumption, one can use the beta distribution, which is defined on \([0, 1]\), and then shift and scale it. Set \( W = a + bX \)  where \( X \sim \text{Beta}(\alpha, \beta) \) and \(a,b>0.\) We have \( \mathbb{E}[X] = \frac{\alpha}{\alpha + \beta}, \)  \( \text{Var}(X) = \frac{\alpha \beta}{(\alpha + \beta)^2 (\alpha + \beta + 1)} \) and \(a\leq W\leq a+b \) a.s.  By solving (for \( a \) and \( b \)) the equations \( \mathbb{E}[W] = a + b\mathbb{E}[X] = 1 \) and \( \text{Var}(W) = b^2\text{Var}(X) = 1, \) we derive \(b=1/\sqrt{\text{Var}(X)}\) and \(a=1-\mathbb{E}[X]/\sqrt{\text{Var}(X)}.\) Note that $a>0$ provided $\alpha+\beta+1<\beta/\alpha$.

\section{Proof of Theorem \ref{th:bound_kolmogorov_dist_pr_sigma_n}}
\label{sec:Gar_real_world_proof}
We first provide details of the expansion \eqref{eq:linear and nonlinear terms}. Recall that the error of SGD approximation may be rewritten as follows
\begin{equation}
    \label{eq: sgd_reccurence_with_remainder}
    \theta_k-\thetas = (\Id - \alpha_k G) (\theta_{k-1}-\thetas) - \alpha_k ( H(\theta_{k-1}) + \eta(\xi_k) + g(\theta_{k-1}, \xi_k))\eqsp.
\end{equation}
Iteratively spinning this expression out we get
\begin{equation}
    \label{eq: sgd_reccurence_with_remainder_2}
    \theta_k-\thetas = \prod_{j=1}^k(\Id - \alpha_j G) (\theta_{0}-\thetas) - \sum_{j=1}^k \alpha_j \prod_{i = j+1}^k (\Id - \alpha_i G) ( H(\theta_{j-1}) + \eta(\xi_j) + g(\theta_{j-1}, \xi_j))\eqsp.
\end{equation}
Taking average of \eqref{eq: sgd_reccurence_with_remainder} and changing the order of summation, we obtain 
\begin{equation}
    \sqrt{n}(\bar\theta_n -\thetas) = \frac{1} {\sqrt{n}\alpha_0}Q_0(\theta_0-\thetas) -\frac{1}{\sqrt{n}}\sum_{i=1}^{n-1}Q_i(H(\theta_{i-1}) + \eta(\xi_i) + g(\theta_{i-1}, \xi_i)),
\end{equation}
where $Q_i$ is defined in \eqref{eq:Q_i_def}. Finally, we obtain 
\begin{align}
\label{eq:linear and nonlinear terms appendix}
    &\sqrt{n}(\bar\theta_n -\thetas) = W + D,\\ &D = \frac{1}{\sqrt{n}\alpha_0}Q_0(\theta_0-\thetas)-\frac{1}{\sqrt{n}}\sum_{i=1}^{n-1}Q_ig(\theta_{i-1}, \xi_i) - \frac{1}{\sqrt{n}}\sum_{i=1}^{n-1}Q_iH(\theta_{i-1})\eqsp,\\
    &W = - \frac{1}{\sqrt{n}}\sum_{i=1}^{n-1}Q_i\eta(\xi_i)\eqsp.
\end{align}

Under suitable assumptions on the step size and the Hessian, we can show that the matrices $Q_i$, defined in \eqref{eq:Q_i_def}, are uniformly bounded for all $i$. This fact will be used in the proof of \Cref{th:bound_kolmogorov_dist_pr_sigma_n}. Moreover, we also establish that $\lambda_{\min}(Q_i)$ is bounded away from zero uniformly for $1\le i\le\lfloor n/2\rfloor$, which in turn implies a lower bound on $\lambda_{\min}(\Sigma_n)$.

\begin{lemma}
\label{lem:bound_Q_i_and_Sigma_n}
Assume \Cref{ass:L-smooth} \Cref{ass:step_size} and suppose that $\lambda_{\min}(\Sigma_\xi)>0$. Then for any $i \in \{0, \ldots, n-1\}$ it holds that
\begin{equation}
\lambda_{\max}(Q_i) \leq C_Q\eqsp,
\end{equation}
where the constant $C_Q$ is defined in \eqref{eq:const_C_Q}
Moreover,
\begin{equation}
\lambda_{\min}(Q_i) \geq C_Q^{\min}\eqsp, \text{ and } \norm{\Sigma_n^{-1/2}} \leq C_{\Sigma}\eqsp,
\end{equation}
where the matrix $\Sigma_n$ is defined in \eqref{eq:sigma_n_def}, and $ C_Q^{\min}, C_{\Sigma}$ are defined in  \eqref{def:C_Q_min} and \eqref{eq:def_C_Sigma} respectively.
\end{lemma}
The version of this lemma with proof and  with explicit constants is given in \ref{subsec:proof_bound_Q_i_Sigma_n}.

\begin{proof}[Proof of Theorem \ref{th:bound_kolmogorov_dist_pr_sigma_n}]
We normalize the both parts of  \eqref{eq:linear and nonlinear terms} by $\Sigma_n^{1/2}$ and obtain
\begin{equation}
    \label{eq: norm_PR_decomposition}
    \sqrt{n}\Sigma_n^{-\frac{1}{2}}(\bar\theta_n -\thetas) = \sum_{i=1}^{n-1}\underbrace{-\frac{\Sigma_n^{-\frac{1}{2}}}{\sqrt{n}}Q_i\eta(\xi_i)}_{w_i} + D_{n,1} + D_{n,2} + D_{n,3}\eqsp,
\end{equation}
where we have set 
\begin{equation}
\label{eq:D_n_1_3_def}
\begin{split}
D_{n,1} &= \frac{\Sigma_n^{-\frac{1}{2}}}{\sqrt{n}\alpha_0}Q_0(\theta_0-\thetas)\eqsp, \\
D_{n,2} &= -\frac{\Sigma_n^{-\frac{1}{2}}}{\sqrt{n}}\sum_{i=1}^{n-1}Q_i H(\theta_{i-1})\eqsp, \\
D_{n,3} &= -\frac{\Sigma_n^{-\frac{1}{2}}}{\sqrt{n}}\sum_{i=1}^{n-1}Q_i g(\theta_{i-1}, \xi_i)) \eqsp.
\end{split}
\end{equation}
Also, for any $1 \leq i \leq n-1$ we construct
\begin{align}
    &D_{n,1}^{(i)} = \frac{\Sigma_n^{-1/2}}{\sqrt{n}\alpha_0}Q_0(\theta_0^{(i)}-\thetas) \eqsp,\\& 
    D_{n,2}^{(i)} = -\frac{\Sigma_n^{-1/2}}{\sqrt{n}}\sum_{j=1}^{n-1}Q_jH(\theta_{j-1}^{(i)})\eqsp,\\&
    D_{n,3}^{(i)} = -\frac{\Sigma_n^{-1/2}}{\sqrt{n}}\sum_{j=1}^{n-1}Q_jg(\theta_{j-1}^{(i)}, \widetilde{\xi_j}^{(i)})),
\end{align}
where we set
\begin{equation}
    \widetilde{\xi_j}^{(i)} =
    \begin{cases}
        \xi_j \eqsp, &\text{if } j \neq i\\
        \xi_j' \eqsp, &\text{if } j = i \eqsp. 
    \end{cases}
\end{equation}
Define $D_n = D_{n,1} +  D_{n,2} +  D_{n,3}$, $D_n^{(i)}= D_{n,1}^{(i)} +  D_{n,2}^{(i)} +  D_{n,3}^{(i)}$, $W_n=\sum_{i=1}^{n-1}w_i$ and  $\Upsilon_{n} = \sum_{i=1}^{n-1}\PE[\norm{\omega_i}^3]$(we keep the same notations as in the unnormalized setting for simplicity). 
Let $Y \sim \mathcal{N}(0, I_d)$. Then, using \citep[Theorem 2.1]{shao2022berry}, we have
\begin{equation}
    \kolmogorov (\sqrt{n}\Sigma_n^{-1/2}(\bar\theta_n-\thetas), Y) \leq 259d^{1/2}\Upsilon_n + 2\underbrace{\PE\{\norm{W_n}\norm{D_n}\}}_{R_1} + 2\underbrace{\sum_{i=1}^{n-1}\PE[\norm{\omega_i}\norm{D_n-D_n^{(i)}}]}_{R_2}\eqsp.
\end{equation}
Applying H\"{o}lder's inequality, we get 
\begin{align}
    &R_1 \leq \PE^{1/2}[\norm{W_n}^2]\PE^{1/2}[\norm{D_n}^2]\eqsp,\\
    &R_2 \leq \sum_{i=1}^{n-1}\PE^{1/2}[\norm{\omega_i}^2]\PE^{1/2}[\norm{D_n-D_n^{(i)}}^2]\eqsp.
\end{align}
Note that $\PE^{1/2}[\norm{W_n}^2]=\sqrt{d}$. Applying \Cref{lem:bound_Q_i_and_Sigma_n},  we get $\PE^{1/2}\norm{w_i}^2\leq \frac{1}{\sqrt{n}}C_{\Sigma}C_Q\sigma_2$  and 
$$
\Upsilon_n \leq \frac{1}{\sqrt{n}}(C_{\Sigma}C_Q\sigma_4)^3 \eqsp.
$$
To complete the proof, it remains to bound $\PE^{1/2}[\|D_n\|^2]$ and $\sum_{i=1}^{n-1}\PE^{1/2}[\|D_n - D_n^{(i)}\|^2]$.
The first term can be bounded using \Cref{lem:bound_Dn}, while the second is controlled via \Cref{lem:bound_sum_Dn-Dni}.
Combining these results, we obtain the following bound:
\begin{equation}
\kolmogorov(\sqrt{n} \Sigma_{n}^{-1/2}(\bar\theta_n -\thetas), Y) \leq \frac{\sqrt{d}M_{3,1}}{\sqrt{n}} + \frac{M_{3,2}}{\sqrt{n}}(\norm{\theta_0-\thetas}+\norm{\theta_0-\thetas}^2+\sigma_2+ \sigma_4^2) + M_{3,3}n^{1/2-\gamma}+M_{3,4}n^{-\gamma/2}\eqsp,
\end{equation}
     where 
    \begin{equation}
    \begin{split}
        &M_{3,1} =259(C_{\Sigma}C_Q\sigma_4)^3\eqsp,\\
        &M_{3,2} = 2\sqrt{d}M_{1,1}  + 2C_{\Sigma}C_Q\sigma_2M_{2,1}\eqsp,\\
        &M_{3,3} = 2\sqrt{d}M_{1,2}\sigma_4^2\eqsp,\\
        &M_{3, 4} = 2\sqrt{d}M_{1,3}\sigma_2+2C_{\Sigma}C_Q\sigma_2( M_{2,2}\sigma_4^2+M_{2,3}\sigma_2)\eqsp.
    \end{split}
    \end{equation}
    Constants $M_{1,1}, M_{1,2}, M_{1,3}$ are defined in \eqref{eq:def_const_M_1} and $M_{2,1}, M_{2,2}, M_{2,3}$ are defined in \eqref{eq:def_const_M_2}.
    We simplify the last inequality and get the statement of the theorem with
\begin{equation}
\label{eq:def_Const_M_3_i}
\begin{aligned}
\ConstC_4
&=
\sqrt d\,M_{3,1}
+
M_{3,2}
\bigl(
\|\theta_0-\theta^\star\|
+\|\theta_0-\theta^\star\|^2
+\sigma_2+\sigma_4^2
\bigr),\\
\ConstC_5&=M_{3,3},\\
\ConstC_6&=M_{3,4}.
\end{aligned}
\end{equation}  
\end{proof}

\subsection{Bounds for $D_n$}
For simplicity of notations we define 
\begin{equation}
\label{def:const_T_1_T_2}
\begin{split}
     &T_1(A) = 1 + \frac{1}{A^{1/(1-\gamma)}(1-\gamma)}\Gamma(\frac{1}{1-\gamma})\eqsp,\\
      &T_2(A) = 1 + \max\biggl(\exp\biggl\{\frac{1}{1-\gamma}\biggr\}\frac{1}{A^{1/(1-\gamma)}(1-\gamma)}\Gamma(\frac{1}{1-\gamma}), \frac{1}{A(1-\gamma)^2}\biggr)\eqsp.
\end{split}
\end{equation}

\begin{lemma}
\label{lem:bound_Dn}
    Assume \Cref{ass:L-smooth}, \Cref{ass:hessian_Lipschitz_ball}, \Cref{ass:noise_decomposition}($4$) and \Cref{ass:step_size}. Then it holds that 
    \begin{equation}
        \PE^{1/2}[\norm{D_n}^2]\leq \frac{M_{1,1}}{\sqrt{n}}(\norm{\theta_0-\thetas}+\norm{\theta_0-\thetas}^2+\sigma_2+ \sigma_4^2)+M_{1,2}\sigma_4^2n^{1/2-\gamma}+M_{1,3}\sigma_2n^{-\gamma/2},
    \end{equation}
    where 
    \begin{equation}
    \label{eq:def_const_M_1}
        \begin{split}         &M_{1,1}=C_{\Sigma}C_Q\biggl(T_1(\frac{\mu c_0}{4})(L_2+L_H)\max(\sqrt{C_{4,1}},\sqrt{C_1}) + k_0^\gamma/c_0\biggr)\\&M_{1,2}=C_{\Sigma}C_QL_H\sqrt{C_{4,2}}c_0\frac{1}{1-\gamma}\\&M_{1,3}=C_{\Sigma}C_QL_2\sqrt{C_{2}}\sqrt{c_0}\sqrt{\frac{1}{1-\gamma}}\eqsp,
        \end{split}
    \end{equation}
where $C_{4,1}$ and $C_{4,2}$ are defined in \Cref{cor:fourth_moment_bound_last_iterate}, $C_1$ and $C_2$ are defined in \Cref{lem:bound_last_iter_second_moment} and $T_1(\cdot)$ is defined in \cref{def:const_T_1_T_2}.
\end{lemma}
\begin{proof}
Using Minkowski's inequality and the definition of $D_n$, we obtain 
\begin{equation}
    \PE^{1/2}[\norm{D_n}^2] \leq 
    \PE^{1/2}[\norm{D_{n,1}}^2] +\PE^{1/2}[\norm{D_{n,2}}^2] + \PE^{1/2}[\norm{D_{n,3}}^2]\eqsp,
\end{equation}
and consider each of the terms $D_{n,1}, D_{n,2}, D_{n,3}$ separately. Applying \Cref{lem:bound_Q_i_and_Sigma_n}, we get
\begin{equation}
     \PE^{1/2}[\norm{D_{n,1}}^2] \leq \frac{C_{\Sigma}C_{Q}k_0^\gamma}{\sqrt{n}c_0}\norm{\theta_0-\thetas}\eqsp.
\end{equation}
Now we consider the term $D_{n,2}$. Applying Minkowski's inequality, \Cref{lem:bound_Q_i_and_Sigma_n} and \Cref{lem:H_theta_bound}, we have
\begin{equation}
    \PE^{1/2}[\norm{D_{n,2}}^2] \leq \frac{C_{\Sigma}C_{Q}}{\sqrt{n}}\sum_{i=1}^{n-1}\PE^{1/2}[\norm{H(\theta_{i-1})}^2]\leq \frac{C_{\Sigma}C_{Q}L_H}{\sqrt{n}}\sum_{i=1}^{n-1}\PE^{1/2}[\norm{\theta_{i-1}-\thetas}^4]\eqsp.
\end{equation}
For $D_{n,3}$ we note that $\{g(\theta_{i-1},\xi_i)\}_{i=1}^{n-1}$ is a  martingale difference with respect to $\F_i$. Hence, using  \Cref{lem:bound_Q_i_and_Sigma_n} and \Cref{ass:noise_decomposition}, we get 
\begin{equation}
\PE^{1/2}[\norm{D_{n,3}}^2] \leq \frac{C_{\Sigma}C_{Q}}{\sqrt{n}}\left(\sum_{i=1}^{n-1}\PE[\norm{g(\theta_{i-1},\xi_i)}^2]\right)^{1/2}
\leq \frac{C_{\Sigma}C_{Q}L_2}{\sqrt{n}}\left(\PE[\sum_{i=1}^{n-1}\norm{\theta_{i-1}-\thetas}^2]\right)^{1/2}\eqsp.
\end{equation}
Hence, it is enough to upper bound $\PE[\norm{\theta_{i}-\thetas}^{2p}]$ for $p = 1$ and $p = 2$ and $i \in \{0,\ldots,n-2\}$. Using \Cref{lem:bound_last_iter_second_moment} and \Cref{lem:bound_sum_exponent}, we obtain 
\begin{align}
&\left(\sum_{i=0}^{n-2}\PE[\norm{\theta_{i}-\thetas}^2]\right)^{1/2} \leq \left(\sum_{i=0}^{n-2} C_1\exp\biggl\{ -\frac{\mu c_0}{4}(i+k_0)^{1-\gamma}\biggr\}[\norm{\theta_0-\thetas}^2 + \sigma_2^2] + C_2\sigma_2^2\alpha_i\right)^{1/2}\\
&\qquad  \leq \sqrt{C_1}\sqrt{T_1\biggl(\frac{\mu c_0}{4}\biggr)}[\norm{\theta_0-\thetas} + \sigma_2] + \sqrt{C_2}\sigma_2\sqrt{c_0}\left(\frac{(n-2+k_0)^{1-\gamma} -(k_0-1)^{1-\gamma}}{1-\gamma}\right)^{1/2}, 
\end{align}
where $T_1(\cdot)$ is defined in \eqref{def:const_T_1_T_2}. Using \Cref{cor:fourth_moment_bound_last_iterate} and \Cref{lem:bound_sum_exponent}, we get 
\begin{align}
    &\sum_{i=0}^{n-2}\PE^{1/2}[\norm{\theta_{i}-\thetas}^4] \leq \sum_{i=0}^{n-2} \sqrt{C_{4,1}}\exp\biggl\{ -\frac{\mu c_0}{4}i^{1-\gamma}\biggr\}[\norm{\theta_0-\thetas}^2 + \sigma_4^2] + \sqrt{C_{4,2}}\sigma_4^2\alpha_i\\& \leq \sqrt{C_{4,1}}T_1\biggl(\frac{\mu c_0}{4}\biggr)[\norm{\theta_0-\thetas}^2 + \sigma_4^2] + \sqrt{C_{4,2}}\sigma_4^2 c_0\left(\frac{(n-2+k_0)^{1-\gamma}- (k_0-1)^{1-\gamma}}{1-\gamma}\right)\eqsp.
\end{align}
We finish the proof, using simple inequality $(n-2+k_0)^{1-\gamma}- (k_0-1)^{1-\gamma} \leq n^{1-\gamma}$
\end{proof}

\subsection{Bounds for $D_n-D_n^{(i)}$}
\begin{lemma}
\label{lem:bound_sum_Dn-Dni}
        Assume \Cref{ass:L-smooth}, \Cref{ass:hessian_Lipschitz_ball}, \Cref{ass:noise_decomposition}($4$) and \Cref{ass:step_size}. Then it holds that 
        \begin{equation}
            \sum_{i=1}^{n-1}\PE^{1/2}[\norm{D_n-D_n^{(i)}}^2] \leq \frac{M_{2,1}}{\sqrt{n}}(\norm{\theta_0-\thetas}+\norm{\theta_0-\thetas}^2+\sigma_2+ \sigma_4^2) + M_{2,2}\sigma_4^2n^{1/2-\gamma/2}+  M_{2,3}\sigma_2n^{1/2-\gamma/2},
        \end{equation}
            where 
    \begin{equation}
    \label{eq:def_const_M_2}
        \begin{split}         &M_{2,1}=C_{\Sigma}C_QT_1(\frac{\mu c_0}{8})T_2(\frac{\mu c_0}{2(1-\gamma)})(L_2+L_H)\max\biggl(\sqrt{2(C_1+c_0^2k_0^{-\gamma}R_1R_2)},\\&\qquad \qquad \qquad \qquad c_0R_{4,1}^{1/4}
\Bigl(
(C_{4,1}R_{4,2})^{1/4}
+(C_{4,1}R_{4,3})^{1/4}
+\sqrt{\alpha_0}(C_{4,2}R_{4,2})^{1/4}
\Bigr)\biggr)\\&M_{2,2}=C_{\Sigma}C_QL_HT_2(\frac{\mu c_0}{2(1-\gamma)})\frac{c_0^{3/2}(R_{4,1}C_{4,2}R_{4,3})^{1/4}}
{1-\gamma/2}\\&M_{2,3}=\sqrt{2c_0}C_{\Sigma}C_QL_2\sqrt{C_{2}+R_1R_3c_0T_2(\frac{\mu c_0}{1-\gamma})}\frac{1}{1-\gamma/2} \eqsp.
        \end{split}
    \end{equation}
    Constants $R_1, R_2, R_3$ are defined in \eqref{eq:def_R_1_R_2_R_3} and constants $R_{4,1}, R_{4,2}, R_{4,3}$ are defined \eqref{eq:def_R_41_R_42_R_43}.
\end{lemma}
\begin{proof}
Using Minkowski's inequality and the definition of $D_n$ and $D_n^{(i)}$, we obtain 
\begin{equation}
    \sum_{i=1}^{n-1}\PE^{1/2}[\norm{D_n-D_n^{(i)}}^2] \leq \sum_{i=1}^{n-1}\PE^{1/2}[\norm{D_{n,2}-D_{n,2}^{(i)}}^2] +\sum_{i=1}^{n-1}\PE^{1/2}[\norm{D_{n,3}-D_{n,3}^{(i)}}^2] 
\end{equation}
 Define $\F_j^{(i)} = \F_j$ if $j< i$ and $\F_j^{(i)} = \sigma(\F_j \vee \sigma(\xi_i'))$ otherwise. Then $\{g(\theta_{j-1},\xi_j)-g(\theta_{j-1}^{(i)},\widetilde{\xi}_j)\}_{j=1}^{n-1}$ is a martingale difference with respect to $\F_j^{(i)}$. Hence, we have, using \Cref{lem:bound_Q_i_and_Sigma_n} and the fact that $\theta_{j-1} = \theta_{j-1}^{(i)}$ for $j \leq i$, we obtain that 
 \begin{align}
&\PE[\norm{D_{n,3}-D_{n,3}^{(i)}}^2] = \PE\norm{\frac{\Sigma_n^{-1/2}}{\sqrt{n}}\sum_{j=1}^{n-1}Q_j(g(\theta_{j-1},\xi_j)-g(\theta_{j-1}^{(i)},\widetilde{\xi}_j))}^2 \\
&\qquad \leq \frac{C_{\Sigma}^2C_Q^2}{n}\PE[\norm{g(\theta_{i-1},\xi_i)-g(\theta_{i-1},\xi_i')}^2] +\frac{C_{\Sigma}^2C_Q^2}{n}\sum_{j=i+1}^{n-1}\PE[\norm{g(\theta_{j-1},\xi_j)-g(\theta_{j-1}^{(i)},\xi_j)}^2]\eqsp.
\end{align}
   Using \Cref{ass:noise_decomposition} and \Cref{lem:bound_Q_i_and_Sigma_n}, we get 
   \begin{equation}
       \PE[\norm{D_{n,3}-D_{n,3}^{(i)}}^2] \leq \frac{2C_{\Sigma}^2C_Q^2L_2^2}{n}\PE[\norm{\theta_{i-1}-\thetas}^2] +\frac{C_{\Sigma}^2C_Q^2L_2^2}{n}\sum_{j=i+1}^{n-1}\PE[\norm{\theta_{j-1}-\theta_{j-1}^{(i)}}^2]\eqsp.
   \end{equation}
   Using \Cref{lem:bound_second_moment_difference} and \Cref{lem:bound_sum_exponent}, we obtain 
   \begin{align}
       &\sum_{j=i+1}^{n-1}\PE[\norm{\theta_{j-1}-\theta_{j-1}^{(i)}}^2] \leq R_1R_2\exp\biggl\{-\frac{\mu c_0}{4}(i+k_0-1)^{1-\gamma}\biggr\}\alpha_i^2(\norm{\theta_{0}-\thetas}^2 + \sigma_2^2)T_2\biggl(\frac{\mu c_0}{1-\gamma}\biggr)(i+k_0)^{\gamma} \\ 
      & \qquad \qquad \qquad \qquad \qquad \qquad \qquad \qquad \qquad \qquad \qquad  + R_1R_3\sigma_2^2\alpha_i^2T_2\biggl(\frac{\mu c_0}{1-\gamma}\biggr)(i+k_0)^{\gamma} \\
      & \qquad \leq R_1R_3\sigma_2^2c_0T_2\biggl(\frac{\mu c_0}{1-\gamma}\biggr)\alpha_i + R_1R_2c_0^2k_0^{-\gamma}T_2\biggl(\frac{\mu c_0}{1-\gamma}\biggr)\exp\biggl\{-\frac{\mu c_0}{4}(i+k_0-1)^{1-\gamma}\biggr\}(\norm{\theta_{0}-\thetas}^2+\sigma_2^2)\eqsp.
   \end{align}
   Combining inequalities above, we get 
   \begin{align}
       \sum_{i=1}^{n-1}\PE^{1/2}[\norm{D_{n,3}-D_{n,3}^{(i)}}^2] &\leq \frac{\sqrt{2}C_{\Sigma}C_QL_2}{\sqrt{n}}\sqrt{C_1+ c_0^2k_0^{-\gamma}R_1R_2T_2\biggl(\frac{\mu c_0}{1-\gamma}\biggr)}T_1\biggl(\frac{\mu c_0}{8}\biggr)(\norm{\theta_0-\thetas}+\sigma_2)\\ &+\frac{\sqrt{2c_0}C_{\Sigma}C_QL_2}{\sqrt{n}}\sqrt{C_2+ R_1R_3c_0T_2\biggl(\frac{\mu c_0}{1-\gamma}\biggr)}\sigma_2\left(\frac{(n+k_0-2)^{1-\gamma/2}-(k_0-1)^{1-\gamma/2}}{1-\gamma/2}\right)\eqsp.
   \end{align}
   We now bound
$\sum_{i=1}^{n-1}
\PE^{1/2}[\|D_{n,2}-D_{n,2}^{(i)}\|^2]$.
By Minkowski's and H\"older's inequalities,
\Cref{lem:bound_Q_i_and_Sigma_n}, and the difference bound in \Cref{lem:H_theta_bound},
\[
\PE^{1/2}[\|D_{n,2}-D_{n,2}^{(i)}\|^2]
\leq
\frac{C_\Sigma C_Q L_H}{\sqrt n}
\sum_{j=i+1}^{n-1}
\PE^{1/4}[\|\theta_{j-1}-\theta^\star\|^4]\,
\PE^{1/4}[\|\theta_{j-1}-\theta_{j-1}^{(i)}\|^4],
\]
where we used that $\theta_{j-1}$ and
$\theta_{j-1}^{(i)}$ have the same distribution.

By \Cref{lem:bounds_on_sum_step_sizes}(b) and  \Cref{lem:bound_sum_exponent},
\[
\sum_{j=i+1}^{n-1}
\exp\left\{-\mu\sum_{\ell=i+1}^{j-1}\alpha_\ell\right\}
\leq
T_2\left(\frac{\mu c_0}{2(1-\gamma)}\right)
(i+k_0)^\gamma.
\]
Applying \Cref{cor:fourth_moment_bound_last_iterate} and \Cref{lem:bound_fourth_moment_difference}, and using
$\alpha_{j-1}\leq\alpha_i\leq\alpha_0$ for $j\geq i+1$,
$\alpha_i(i+k_0)^\gamma=c_0$, and $k_0\geq1$, we obtain
\begin{align}
&\sum_{j=i+1}^{n-1}
\PE^{1/4}[\|\theta_{j-1}-\theta^\star\|^4]\,
\PE^{1/4}[\|\theta_{j-1}-\theta_{j-1}^{(i)}\|^4]\\
&\quad\leq
c_0R_{4,1}^{1/4}
T_2\left(\frac{\mu c_0}{2(1-\gamma)}\right)
\Biggl[
\Bigl(
(C_{4,1}R_{4,2})^{1/4}
+(C_{4,1}R_{4,3})^{1/4}
+\sqrt{\alpha_0}(C_{4,2}R_{4,2})^{1/4}
\Bigr)\\
&\hspace{35mm}\times
\exp\left\{-\frac{\mu c_0}{8}i^{1-\gamma}\right\}
\bigl(\|\theta_0-\theta^\star\|^2+\sigma_4^2\bigr)
+
(C_{4,2}R_{4,3})^{1/4}\sigma_4^2\sqrt{\alpha_i}
\Biggr].
\end{align}
Here we also used
\[
(\|\theta_0-\theta^\star\|^4+\sigma_4^4)^{1/2}
\leq \|\theta_0-\theta^\star\|^2+\sigma_4^2
\]
and
\[
\sigma_4
(\|\theta_0-\theta^\star\|^4+\sigma_4^4)^{1/4}
\leq \|\theta_0-\theta^\star\|^2+\sigma_4^2.
\]

Finally,  \Cref{lem:bound_sum_exponent} and
\[
\sum_{i=1}^{n-1}\sqrt{\alpha_i}
\leq
\frac{\sqrt{c_0}}{1-\gamma/2}\,n^{1-\gamma/2}
\]
yield
\begin{align}
&\sum_{i=1}^{n-1}
\PE^{1/2}[\|D_{n,2}-D_{n,2}^{(i)}\|^2]\\
&\quad\leq
\frac{C_\Sigma C_Q L_Hc_0R_{4,1}^{1/4}}{\sqrt n}
T_2\left(\frac{\mu c_0}{2(1-\gamma)}\right)
T_1\left(\frac{\mu c_0}{8}\right)\\
&\qquad\times
\Bigl(
(C_{4,1}R_{4,2})^{1/4}
+(C_{4,1}R_{4,3})^{1/4}
+\sqrt{\alpha_0}(C_{4,2}R_{4,2})^{1/4}
\Bigr)
\bigl(\|\theta_0-\theta^\star\|^2+\sigma_4^2\bigr)\\
&\qquad+
\frac{C_\Sigma C_Q L_Hc_0^{3/2}
(R_{4,1}C_{4,2}R_{4,3})^{1/4}}{1-\gamma/2}
T_2\left(\frac{\mu c_0}{2(1-\gamma)}\right)
\sigma_4^2n^{1/2-\gamma/2}.
\end{align}
Combining this estimate with the bound for
$D_{n,3}-D_{n,3}^{(i)}$ proves the claim.
\end{proof}

\subsection{Bounds for $\theta_k^{(i)}-\theta_k$}
Let $(\xi_1',\ldots, \xi_{n-1}')$ be an independent copy of $(\xi_1,\ldots, \xi_{n-1})$. For each $1 \leq i \leq n-1$, we construct the sequence $\theta_k^{(i)}$, $1 \leq k \leq n-1$, as follows:
\begin{equation}
\label{eq:independent_replace_construct}
\theta_k^{(i)} = 
    \begin{cases}
     \theta_k\eqsp, \qquad &\text{if }k<i\\
     \theta_{k-1}^{(i)} - \alpha_k(\nabla f(\theta_{k-1}^{(i)}) + g(\theta_{k-1}^{(i)}, \xi_k')+ \eta(\xi_k')) \eqsp,\qquad&\text{if }k=i\\
     \theta_{k-1}^{(i)} - \alpha_k(\nabla f(\theta_{k-1}^{(i)}) + g(\theta_{k-1}^{(i)}, \xi_k)+\eta(\xi_k))\eqsp,\qquad&\text{if }k>i \eqsp.
    \end{cases}
\end{equation}

\begin{lemma}
\label{lem:bound_second_moment_difference}
   Assume \Cref{ass:L-smooth}, \Cref{ass:hessian_Lipschitz_ball}, \Cref{ass:noise_decomposition}($2$) and \Cref{ass:step_size}. Then for any $k \in \nset$ and $1 \leq i \leq n-1$ it holds  
    \begin{equation}
        \PE[\norm{\theta_k^{(i)}-\theta_k}^2] \leq  \alpha_i^2R_1\exp\biggl\{-2\mu\sum_{j=i+1}^{k}\alpha_j\biggr\}\biggl(R_2\exp\biggl\{-\frac{\mu c_0}{4}(i+k_0-1)^{1-\gamma}\biggr\}(\norm{\theta_{0}-\thetas}^2 + \sigma_2^2) + R_3\sigma_2^2
    \biggr),
    \end{equation}
    where we have set
    \begin{equation}
    \label{eq:def_R_1_R_2_R_3}
        R_1 = 4 \exp\biggl\{\frac{2 c_0^2(L_1+L_2)^2}{2\gamma-1}\biggr\}\eqsp, \qquad R_2 = L_2^2C_1, \qquad  R_3 =(1+C_2L_2^2\alpha_0) \eqsp.
    \end{equation}
    And constant $C_1$ and $C_2$ are defined in \Cref{lem:bound_last_iter_second_moment}.
    \end{lemma}
\begin{proof}
    By construction \eqref{eq:independent_replace_construct}, we have 
    \begin{equation}
    \theta_k^{(i)}-\theta_k = 
    \begin{cases}
     0\eqsp, \quad &\text{if }k<i\\
      - \alpha_k\bigl(g(\theta_{k-1}, \xi_k')+\eta(\xi_k')-g(\theta_{k-1}, \xi_k)-\eta(\xi_k)\bigr)\eqsp,\quad&\text{if }k=i\\
     \theta_{k-1}^{(i)}-\theta_{k-1}-\alpha_k\bigl(\nabla f(\theta_{k-1}^{(i)}) - \nabla f (\theta_{k-1}) +g(\theta_{k-1}^{(i)}, \xi_k)-g(\theta_{k-1}, \xi_k)\bigr) \eqsp,\quad&\text{if }k>i\\
    \end{cases}
    \end{equation}
    Since $\xi_i'$ is independent copy of $\xi_i$, we obtain 
    \begin{align}
        \PE[\norm{\theta_i^{(i)}-\theta_i}^2] &\overset{(a)}{\leq} 4\alpha_i^2(L_2^2\PE[\norm{\theta_{i-1}-\thetas}^2] + \sigma_2^2) \\&\overset{(b)}{\leq} 4\alpha_i^2\biggl(L_2^2C_1\exp\biggl\{-\frac{\mu c_0}{4}(i+k_0-1)^{1-\gamma}\biggr\}(\norm{\theta_{0}-\thetas}^2 + \sigma_2^2) + (1+C_2L_2^2\alpha_0)\sigma_2^2
    \biggr),
    \end{align}
    where in (a) we used \Cref{ass:noise_decomposition}, and in (b) we used \Cref{lem:bound_last_iter_second_moment} and $\alpha_{i-1}\leq\alpha_0$.
    For $k>i$, applying \cref{ass:noise_decomposition} and \cref{ass:L-smooth}, we have 
    \begin{align}
         \PE [\norm{\theta_k^{(i)}-\theta_k}^2|\F_{k-1}^{(i)}] &\leq \norm{\theta_{k-1}^{(i)}-\theta_{k-1}}^2 -2\alpha_k\langle\theta_{k-1}^{(i)}-\theta_{k-1}, \nabla f(\theta_{k-1}^{(i)}) - \nabla f (\theta_{k-1}) \rangle \\&\qquad + 2\alpha_k^2(L_1+L_2)^2\norm{\theta_{k-1}^{(i)}-\theta_{k-1}}^2\eqsp.
    \end{align}
    Taking expectation from both sides and applying \Cref{ass:L-smooth} with \Cref{lem:bounds_on_sum_step_sizes}\ref{eq:sum_alpha_k_p}, we obtain  
    \begin{align}
        \PE [\norm{\theta_k^{(i)}-\theta_k}^2] &\leq (1-2\alpha_k\mu + 2\alpha_k^2(L_1+L_2)^2)\PE[\norm{\theta_{k-1}^{(i)}-\theta_{k-1}}^2] \\&\leq \exp\biggl\{\frac{2c_0^2(L_1+L_2)^2}{2\gamma-1}\biggr\}\exp\biggl\{-2\mu\sum_{j=i+1}^{k}\alpha_j\biggr\}\PE[ \norm{\theta_i^{(i)}-\theta_i}^2]\eqsp.
    \end{align}
    Combining the above inequalities completes the proof.
\end{proof}

\begin{lemma}
\label{lem:bound_fourth_moment_difference}
   Assume \Cref{ass:L-smooth}, \Cref{ass:hessian_Lipschitz_ball}, \Cref{ass:noise_decomposition}($4$) and \Cref{ass:step_size}. Then for any $k \in \nset$ and $1 \leq i \leq n-1$ it holds  
    \begin{equation}
        \PE[\norm{\theta_k^{(i)}-\theta_k}^4] \leq\alpha_i^4R_{4,1}\exp\biggl\{ -4\mu\sum_{j=i+1}^k\alpha_j\biggr\}\biggl(R_{4,2}\exp\{-\frac{2\mu c_0}{4}(i+k_0-1)^{1-\gamma}\}(\norm{\theta_{0}-\thetas}^4 + \sigma_4^4) + R_{4,3}\sigma_4^4\biggr)
    \end{equation}
    where we have set
    \begin{equation}
    \label{eq:def_R_41_R_42_R_43}
    R_{4,1} = 64\exp\biggl\{ \frac{8c_0^2(L_1+L_2)^2(1+3c_0(L_1+L_2))^2)}{2\gamma-1}\biggr\}\eqsp, \quad R_{4,2} = L_2^4C_{4,1}\eqsp, \quad R_{4,3} =1+ L_2^4C_{4,2}\alpha_0^2\eqsp.
    \end{equation}
    And constant $C_{4,1}$, $C_{4,2}$ are defined in \Cref{cor:fourth_moment_bound_last_iterate}.
    \end{lemma}
\begin{proof}
Repeating the proof of the \Cref{lem:bound_second_moment_difference} for $k = i$, we get
\begin{align}
        \PE[\norm{\theta_i^{(i)}-\theta_i}^4] &\leq64\alpha_i^4(L_2^4\PE[\norm{\theta_{i-1}-\thetas}^4] + \sigma_4^4) \\& \leq 64\alpha_i^4\biggl(L_2^4C_{4,1}\exp\biggl\{-\frac{2\mu c_0}{4}(i+k_0-1)^{1-\gamma}\biggr\}(\norm{\theta_{0}-\thetas}^4 + \sigma_4^4) + (1+L_2^4C_{4,2}\alpha_0^2)\sigma_4^4\biggr)\eqsp.
\end{align}
For $k>i$ we denote $\delta_k^{(i)} = \norm{\theta_k^{(i)}-\theta_k}$, similar to \eqref{eq:coupling_recurence}, we obtain 
\begin{align}
E[\{\delta_k^{(i)}\}^4|\F_{k-1}^{(i)}] \leq (1 -4\mu\alpha_k + 8\alpha_k^2(L_1+L_2)^2(1+3c_0(L_1+L_2))^2)\{\delta_{k-1}^{(i)}\}^4\eqsp.
\end{align}
Using 
\Cref{lem:bounds_on_sum_step_sizes}\ref{eq:sum_alpha_k_p}, we obtain
\begin{equation}
    E[\{\delta_k^{(i)}\}^4]\leq \exp\biggl\{ \frac{8c_0^2(L_1+L_2)^2(1+3c_0(L_1+L_2))^2)}{2\gamma-1}\biggr\}\exp\biggl\{ -4\mu\sum_{j=i+1}^k\alpha_j\biggr\} \PE[\norm{\theta_i^{(i)}-\theta_i}^4]\eqsp.
\end{equation}
Combining the above inequalities completes the proof.
\end{proof}

\section{Proof of Theorem \ref{GAR bootstrap_main}}
\label{subsec:GAR in the bootstrap world}
Since the matrix $\Sigma_n^\boot$ concentrates around $\Sigma_n$ due to \Cref{lem:matrix_bernstein} , there is a set $\Omega_1$ such that $\PP(\Omega_1) \geq 1-1/n$ and  $\lambda_{\min}(\Sigma_n^\boot) > 0$ on $\Omega_1$. On $\Omega_1$, applying \Cref{nonlinearapprox} conditionally on
$\mathcal F_{n-1}$ with 
    $$
    X = \{\Sigma_n^\boot\}^{-1/2} W^\boot, \quad Y = \{\Sigma_n^\boot\}^{-1/2} D^\boot,
    $$
    we get
    \begin{equation}
\begin{split}
    &\sup_{B \in \Conv(\rset^d)}|\PPb( \sqrt n  \{\Sigma_n^\boot\}^{-1/2} (\bar{\theta}_{n}^\boot - \bar{\theta}_n) \in B) - \PPb(Y^\boot \in B)| \\
    & \qquad\qquad\le \sup_{B \in \Conv(\rset^d)}|\PPb(\{\Sigma_n^\boot\}^{-1/2} W^\boot \in B) - \PPb(Y^\boot \in B)|  + 2 c_d (\PEb[\|\{\Sigma_n^\boot\}^{-1/2} D^\boot\|^p])^{1/(1+p)} \eqsp. 
    \end{split}
\end{equation}
By \citep{shao2022berry} (with $D = 0$) we may estimate 
\begin{equation}
    \label{CLT boot lin stat proof}
    \begin{split}
    &\sup_{B \in \Conv(\rset^d)}|\PPb(\{\Sigma_n^\boot\}^{-1/2} W^\boot \in B) - \PPb(Y^\boot \in B)| \\
    &\qquad\qquad\qquad\qquad\le \frac{259 d^{1/2}}{n^{3/2}} \sum_{i = 1}^{n-1} \PEb[|w_i-1|^3] \| (\{\Sigma_n^\boot\}^{-1/2} Q_i 
    \eta(\xi_i)\|^3 \eqsp.
    \end{split}
\end{equation}


Set
\[
C_{1,\xi}^3
= 8(1+\sqrt{2})\sqrt{6\rme}\,d^{3/2}C_\xi^3.
\qquad
\Omega_\eta
=\left\{\frac1n\sum_{i=1}^{n-1}\|\eta(\xi_i)\|^3
          \leq C_{1,\xi}^3\right\}.
\]
Since $\eta(\xi_1)$ is sub-Gaussian with variance proxy $C_\xi^2$,
Lemma~19 gives
\[
\PE\exp\!\left\{\frac{\|\eta(\xi_1)\|^2}{4dC_\xi^2}\right\}
\leq \rme.
\]
Using $u^3\leq 6\rme^u$ for $u\geq 0$, we obtain
\[
\PE\|\eta(\xi_1)\|^6
\leq 6\rme(4dC_\xi^2)^3,
\qquad
(1+\sqrt{2})\bigl(\PE\|\eta(\xi_1)\|^6\bigr)^{1/2}
\leq C_{1,\xi}^3.
\]
Write $S_n=n^{-1}\sum_{i=1}^{n-1}\|\eta(\xi_i)\|^3$.
Since the observations are identically distributed, the
Cauchy--Schwarz inequality yields
\[
\mathbb E S_n
= \frac{n-1}{n}\,\mathbb E\|\eta(\xi_1)\|^3
\leq \bigl(\mathbb E\|\eta(\xi_1)\|^6\bigr)^{1/2}.
\]
Moreover, independence gives
\[
\operatorname{Var}(S_n)
= \frac{1}{n^2}\sum_{i=1}^{n-1}
  \operatorname{Var}\bigl(\|\eta(\xi_i)\|^3\bigr)
\leq \frac{n-1}{n^2}\,\mathbb E\|\eta(\xi_1)\|^6
\leq \frac{1}{n}\,\mathbb E\|\eta(\xi_1)\|^6.
\]
By the choice of $C_{1,\xi}$,
\[
C_{1,\xi}^3-\mathbb E S_n
\geq \sqrt{2}\,\bigl(\mathbb E\|\eta(\xi_1)\|^6\bigr)^{1/2}.
\]
Hence, by Chebyshev's inequality,
\[
\mathbb P(\Omega_\eta^c)
= \mathbb P(S_n>C_{1,\xi}^3)
\leq \frac{\operatorname{Var}(S_n)}
          {(C_{1,\xi}^3-\mathbb E S_n)^2}
\leq \frac{1}{2n}.
\]

On $\Omega_1\cap\Omega_\eta$, Lemma \ref{lem:bound_Q_i_and_Sigma_n} and Corollary \ref{lambda min boot}
\[
\|(\Sigma_n^b)^{-1/2}Q_i\eta(\xi_i)\|
\leq \sqrt{2}C_\Sigma C_Q\|\eta(\xi_i)\|.
\]
Since $\mathbb E^b|w_i-1|^3
\leq W_{\max}\mathbb E^b|w_i-1|^2=W_{\max}$,
the right-hand side of \eqref{CLT boot lin stat proof} is bounded by
\[
\frac{259\sqrt d\,W_{\max}(\sqrt{2}C_\Sigma C_Q)^3}{\sqrt n}
\frac1n\sum_{i=1}^{n-1}\|\eta(\xi_i)\|^3
\leq \frac{M_{3,1}^b}{\sqrt n}.
\]

Set $p=\log n-1$ and $q=p/(p+1)$. Define
\[
\Omega_D=
\left\{
(\mathbb E^b\|D^b\|^p)^{1/p}
\leq (2n)^{1/p}e^{1/p}
\left(M_{1,1}^b p^{3/2}n^{-\gamma/2}
     +M_{2,1}^b p n^{1/2-\gamma}\right)
\right\}.
\]
By Proposition \ref{prop:prob-D-boot-bound} and Markov's inequality,
$\mathbb P(\Omega_D^c)\leq 1/(2n)$. Consequently,
\[
\Omega=\Omega_1\cap\Omega_\eta\cap\Omega_D,
\qquad
\mathbb P(\Omega)\geq 1-\frac2n.
\]

On $\Omega$, Corollary \ref{lambda min boot} and $(x+y)^q\leq x^q+y^q$ give
\begin{align}
&2c_d^q
\bigl(\mathbb E^b\|(\Sigma_n^b)^{-1/2}D^b\|^p\bigr)^{1/(p+1)}
\\
&\quad\leq 2(2en)^{1/(p+1)}
\left[
 (\sqrt{2}\,c_d C_\Sigma M_{1,1}^b)^q
 p^{3q/2}n^{-\gamma q/2}
 +
 (\sqrt{2}\,c_d C_\Sigma M_{2,1}^b)^q
 p^q n^{-(\gamma-1/2)q}
\right].
\end{align}
Since $p+1=\log n\geq4$, we have
$(2en)^{1/(p+1)}\leq e^{3/2}$ and
$n^{-a q}=e^a n^{-a}$. Using also $x^q\leq\max\{1,x\}$,
we obtain the stated bound with
\begin{equation}
\label{M_3i boot def}
\begin{split}
    &M_{3,1}^\boot = 259  (\sqrt{2} C_\Sigma C_Q C_{1, \xi})^3W_{\max} \sqrt{d} \eqsp,\\ 
    &M_{3,2}^b
=2e^{3/2+\gamma}
  \max\{1,\sqrt{2}\,c_d C_\Sigma M_{2,1}^b\},\\
    &M_{3,3}^b
=2e^{3/2+\gamma/2}
  \max\{1,\sqrt{2}\,c_d C_\Sigma M_{1,1}^b\}, 
\end{split}   
\end{equation}
and 
\begin{equation}
\label{def:C_1_xi}
    C_{1,\xi}^3
= 8(1+\sqrt{2})\sqrt{6\rme}\,d^{3/2}C_\xi^3.
\end{equation}
and $M_{1,1}^{\boot}, M_{2,1}^{\boot}$ are defined in \eqref{eq:def_M_11_M_21_boot} and 
combining the above inequalities, we complete the proof.

\begin{remark}
We use \citep{shao2022berry} with $D = 0$ to prove \eqref{CLT boot lin stat proof} since we are not aware of Berry-Esseen results for non i.i.d. random vectors in dimension $d$ with precise constants and dependence on $d$. The result \citet{bentkus2004} may be applied for i.i.d. vectors only. 
\end{remark}

\subsection{From non-linear to linear statistics}
\label{sec: from nonlinear to linear boot}
In this section we prove \eqref{crude bound for nonlinear stat}. We start from the definition of an isoperimetric constant.  
Define 
\[
A^{\varepsilon} = \{x \in \rset^d : \rho_A(x) \leq \varepsilon\}
\quad \text{and} \quad
A^{-\varepsilon} = \{x \in A : B_{\varepsilon}(x) \subset A\},
\]
where $\rho_A(x) = \inf\limits_{y \in A} \|x - y\|$ is the distance between $A \subset \rset^d$ and $x \in \rset^d$, and 
\[
B_{\varepsilon}(x) = \{y \in \rset^d : \|x - y\| \leq \varepsilon\}.
\]
For some class $\mathcal A$ of subsets of $\rset^d$ we define its isoperimetric constant $a_d(\mathscr{A})$ (depending only on $d$ and $\mathscr{A}$)  as follows: for all $A \in \mathscr{A}$  and $\varepsilon > 0$,
\[
\mathbb{P} \{ Y \in A^{\varepsilon} \setminus A \} \leq a_d \varepsilon, \quad
\mathbb{P} \{ Y \in A \setminus A^{-\varepsilon} \} \leq a_d \varepsilon
\]
where  $Y$ follows the standard Gaussian distribution on $\rset^d$. \citet{ball_reverse_1993} has proved that
\begin{equation}
\label{ball}
    e^{-1} \sqrt{\ln d} \leq \sup_{A \in \mathscr{C}} \int_{\partial A} p(x) \, \mathrm{d}s \leq 4 d^{1/4},
\end{equation}
where $p(x)$ is the standard normal $d$-dimensional density and $ds$ is the surfac emeasure on the boundary $\partial A$ of $A$. Using \eqref{ball} one can show that for the class of convex sets 
\begin{equation}
\label{eq: isoperimetric const convex}
    e^{-1} \sqrt{\ln d} \leq a_d(\Conv(\rset^d)) \le   4 d^{1/4} \eqsp. 
\end{equation}
We take $c_d=\max\{1,a_d(\mathcal C(\mathbb R^d))\}$
as an isoperimetric constant for the class $\mathcal C(\mathbb R^d)$.
\begin{proposition}
\label{nonlinearapprox}
Let $\nu$ be a standard Gaussian measure in $\rset^d$. Then for any random vectors $X, Y$ taking values in $\rset^d$, and any $p \geq 1$,
  \begin{equation}
\label{eq:shao_zhang_bound}
\sup_{B \in \Conv(\rset^d)}|\PP(X + Y \in B) - \nu(B)| \le \sup_{B \in \Conv(\rset^d)}|\PP(X
\in B) - \nu(B)| + 2 c_d^{p/(p+1)} \PE^{1/(p+1)}[\|Y\|^p]\eqsp, 
\end{equation}
where $c_d$ is the isoperimetric constant of class $\Conv(\rset^d)$.
\end{proposition}
\begin{proof}
Let $\varepsilon > 0$. Define
$\rho(B) = \PP(X + Y \in B) - \nu(B)$. 
Let $B$ be such that $\rho(B) \geq 0$.  By Markov's inequality
\begin{multline}
\rho(B) \le \PP(X + Y \in B, |Y| \le \varepsilon) + \frac{1}{\varepsilon^{p}} \PE [\|Y\|^p] - \nu(B) \\
    \le \sup_{A}|\PP(X \in A) - \nu(A)| + \nu(B^\varepsilon \setminus B) + \frac{1}{\varepsilon^{p}} \PE [\|Y\|^p].  
\end{multline}   
Choosing 
\begin{equation}
\label{eps choice}
    \varepsilon = \frac{1}{c_d^{1/(p+1)}}\PE^{1/(p+1)}[\|Y\|^p]
\end{equation}
we obtain
\begin{equation}
\rho(B) \le \sup_{B }|\PP(X \in B) - \nu(B)| + 2 c_d^{p/(p+1)} \PE^{1/(p+1)}[\|Y\|^p] \eqsp. 
\end{equation}
Assume now that $\rho(B)<0$. Since
$\{X\in B^{-\varepsilon},\,\|Y\|<\varepsilon\}
\subseteq\{X+Y\in B\}$, we have
\begin{align}
-\rho(B)
&\leq \nu(B)-\mathbb P(X\in B^{-\varepsilon})
      +\mathbb P(\|Y\|\geq\varepsilon)\\
&\leq \sup_{A\in\mathcal C(\mathbb R^d)}
      |\mathbb P(X\in A)-\nu(A)|
      +\nu(B\setminus B^{-\varepsilon})
      +\varepsilon^{-p}\mathbb E[\|Y\|^p].
\end{align}
Using $\nu(B\setminus B^{-\varepsilon})\leq c_d\varepsilon$
and choosing $\varepsilon$ as in \eqref{eps choice}, we obtain the same bound
for $-\rho(B)$. Taking the supremum over $B$ completes the proof.
\end{proof}

\subsection{Bounds for $D^\boot$}
Recall that the term $D^\boot$ defined in \eqref{eq:D_boot-def}, has a form: 
\begin{equation}
\label{eq:D_boot-def-appendix}
\begin{split}
        D^\boot &= -\frac{1}{\sqrt{n}}\sum_{i=1}^{n-1}(w_i-1) Q_{i}\biggl( G(\theta_{i-1}^{\boot}-\thetas) + g(\theta_{i-1}^{\boot}, \xi_i) + H(\theta_{i-1}^{\boot})\biggr) \\&- \frac{1}{\sqrt{n}}\sum_{i=1}^{n-1}Q_i\biggl(H(\theta_{i-1}^{\boot}) +g(\theta_{i-1}^{\boot}, \xi_i) -  H(\theta_{i-1}) -g(\theta_{i-1}, \xi_i)\biggr)\eqsp. 
\end{split}
\end{equation}
To prove \Cref{GAR bootstrap_main}, we need to obtain a high-probability bound for the non-linear statistic $D^\boot$.
To this end, we first derive a bound on $\PE^{1/p}[\|D^\boot\|^p]$, where the expectation is taken with respect to the joint distribution of the bootstrap weights and the data $\Xi^{n-1}$.
We then apply Markov's inequality to convert this moment bound into a high-probability bound.
\begin{proposition} \label{prop:prob-D-boot-bound}
Assume \Cref{ass:L-smooth}-  \Cref{ass:step_size_new_boot}. Then it holds for any $2\leq p \leq \log n$ that 
\begin{equation}
\label{eq:D-boot-joint-bound}
\PE^{1/p}[\norm{D^\boot}^p] \leq M_{1,1}^{\boot}\rme^{1/p}p^{3/2}n^{-\gamma/2} +  M_{2,1}^{\boot}\rme^{1/p}p n^{1/2-\gamma}\eqsp,
\end{equation}
where the constants are given by 
\begin{equation}
\label{eq:def_M_11_M_21_boot}
\begin{split} 
&M_{1,1}^{\boot} =4C_Q\max(L_1,L_2)\frac{\max(\sqrt{K_2},\sqrt{K_1})\sqrt{c_0}(W_{\max}+1)}{\sqrt{2}(1-\gamma)}\eqsp,\\
& M_{2,1}^{\boot}=  3C_QL_H\frac{c_0\max(K_2,K_1)(W_{\max}+1)}{2(1-\gamma)} \eqsp,
\end{split}
\end{equation}
 and $K_1, K_2$ are defined in \eqref{eq:def_K_1}, \eqref{eq:def_K_2}, respectively.
Moreover, there is a set $\Omega_0 \in \F_{n-1} = \sigma(\xi_1,\ldots,\xi_{n-1})$, such that $\PP(\Omega_0) \geq 1-1/n$, and on $\Omega_0$ it holds that
\begin{equation}
\label{eq:D-boot-high-prob-bound}
\{\PEb[\norm{D^\boot}^{p}]\}^{1/p} \leq M_{1,1}^{\boot}\rme^{1/p}p^{3/2}n^{1/p-\gamma/2} +  M_{2,1}^{\boot}\rme^{1/p}p n^{1/2+1/p-\gamma}\eqsp.
\end{equation}
\end{proposition}
\begin{proof}
We first show \eqref{eq:D-boot-joint-bound}. We split 
\[
D^\boot = D^\boot_{1} + D^\boot_{2}\eqsp,
\]
where 
\begin{equation}
\begin{split}
D^\boot_{1} &= -\frac{1}{\sqrt{n}}\sum_{i=1}^{n-1}(w_i-1) Q_{i}\bigl( G(\theta_{i-1}^{\boot}-\thetas) + g(\theta_{i-1}^{\boot}, \xi_i)\bigr) - \frac{1}{\sqrt{n}}\sum_{i=1}^{n-1}Q_i\bigl(g(\theta_{i-1}^{\boot}, \xi_i) -g(\theta_{i-1}, \xi_i)\bigr)\eqsp, \\
D^\boot_{2} &= -\frac{1}{\sqrt{n}}\sum_{i=1}^{n-1}(w_i-1) Q_{i} H(\theta_{i-1}^{\boot}) - \frac{1}{\sqrt{n}}\sum_{i=1}^{n-1}Q_i\bigl(H(\theta_{i-1}^{\boot})-  H(\theta_{i-1})\bigr)\eqsp.
\end{split}
\end{equation}
Applying Minkowski's inequality together with \Cref{lem:bound_p_moment_martingale_components} and \Cref{lem:bound_p_moment_for_H_component} we get \eqref{eq:D-boot-joint-bound}.

To proof \eqref{eq:D-boot-high-prob-bound} we consider 
\begin{equation}
    \Omega_0 =\{\{\PEb[\norm{D^\boot}^{p}]\}^{1/p} \leq M_{1,1}^{\boot}\rme^{1/p}p^{3/2}n^{1/p-\gamma/2} +  M_{2,1}^{\boot}\rme^{1/p}p n^{1/2+1/p-\gamma}\}\eqsp.
\end{equation}
Note that by Markov's inequality 
\begin{align}
    &\PP(\Omega_0^c) \leq \frac{\PE[\{\PEb[\norm{D^\boot}^{p}]\}]}{n(M_{1,1}^{\boot}\rme^{1/p}p^{3/2}n^{-\gamma/2} +  M_{2,1}^{\boot}\rme^{1/p}p n^{1/2-\gamma})^p}\\& = 
    \frac{\PE[\norm{D^\boot}^{p}]}{n(M_{1,1}^{\boot}\rme^{1/p}p^{3/2}n^{-\gamma/2} +  M_{2,1}^{\boot}\rme^{1/p}p n^{1/2-\gamma})^p} \leq \frac{1}{n}\eqsp.
\end{align}

\end{proof}

\label{sec: Dboot bounds}
\begin{lemma}
\label{lem:bound_p_moment_martingale_components}
Assume \Cref{ass:L-smooth}-\Cref{ass:step_size_new_boot}. Then for any $2\leq p \leq \log n$ it holds
\begin{equation}
    \PE^{1/p}[\norm{D^\boot_{1}}^p]\leq M_{1,1}^{\boot}\rme^{1/p}p^{3/2}n^{-\gamma/2} \eqsp, 
\end{equation}
where 
\begin{equation}
\label{def:M_11_boot}
    M_{1,1}^{\boot} = 4C_Q\max(L_1,L_2)\frac{\max(\sqrt{K_2},\sqrt{K_1})\sqrt{c_0}(W_{\max}+1)}{\sqrt{2}(1-\gamma)}\eqsp,
\end{equation}
and $K_1, K_2$ are defined in \eqref{eq:def_K_1}, \eqref{eq:def_K_2}, respectively.
\end{lemma}
\begin{proof}
    We split \( D^\boot_{1} \) into four parts, where each part is a sum of martingale differences.
    Note that $\{Q_i(g(\theta_{i-1},\xi_i)-g(\thetas,\xi_i))\}_{i=1}^{n}$ is a martingale difference sequence with respect to $(\mathcal F_i)_{i=0}^{n-1}$.
    Then applying Burholder's inequality \citep[Theorem 8.6]{oskekowski2012sharp} together with Minkowski's inequality and \Cref{lem:bound_Q_i_and_Sigma_n}, we obtain that 
    \begin{align}   
    &\PE^{1/p}\bigl[\norm{\sum_{i=1}^{n-1}Q_i\bigl(g(\theta_{i-1},\xi_i)-g(\thetas,\xi_i) \bigr)}^{p}\bigr] \\ 
    &\qquad \qquad \leq p\bigl(\PE^{2/p}\bigl[\bigl( \sum_{i=1}^{n-1}\norm{Q_i\bigl(g(\theta_{i-1},\xi_i)-g(\thetas,\xi_i) \bigr)}^2\bigr)^{p/2}\bigr]\bigr)^{1/2}
    \\ 
    &\qquad \qquad \leq C_Qp\bigl(\PE^{2/p}\bigl[\bigl( \sum_{i=1}^{n-1}\norm{g(\theta_{i-1},\xi_i)-g(\thetas,\xi_i)}^2\bigr)^{p/2}\bigr]\bigr)^{1/2}\\ 
    &\qquad \qquad \leq C_Qp\bigl(\sum_{i=1}^{n-1}\PE^{2/p}[\norm{g(\theta_{i-1},\xi_i)-g(\thetas,\xi_i) }^{p}]\bigr)^{1/2}\eqsp.
    \end{align}
Finally, using \Cref{ass:noise_decomposition} and \Cref{lem: high_prob_last_iter}, we obtain 
\begin{align}
&\PE^{1/p}\bigl[\norm{\sum_{i=1}^{n-1}Q_i\bigl(g(\theta_{i-1},\xi_i)-g(\thetas,\xi_i) \bigr)}^{p}\bigr] \leq pC_QL_2\bigl(\sum_{i=1}^{n-1}\PE^{2/p}[\norm{\theta_{i-1} -\thetas}^{p}]\bigr)^{1/2}\\ 
& \qquad \qquad \leq  C_QL_2(\rme )^{1/p}p^{3/2}\frac{\sqrt{K_2}}{\sqrt{2}}\bigl(\sum_{i=0}^{n-2}\alpha_i\bigr)^{1/2} \\
& \qquad \qquad \leq  C_QL_2(\rme )^{1/p}p^{3/2}\frac{\sqrt{K_2}}{\sqrt{2}}\bigl(c_0\frac{(k_0+n-2)^{1-\gamma}-(k_0-1)^{1-\gamma}}{1-\gamma}\bigr)^{1/2}\eqsp.
\end{align}
Since $k_0 \geq 1$ and $(k_0+n-2)^{1-\gamma}-(k_0-1)^{1-\gamma}\leq n^{1-\gamma}$ we complete the proof for $$
\PE^{1/p}\bigl[\norm{\sum_{i=1}^{n-1}Q_i\bigl(g(\theta_{i-1},\xi_i)-g(\thetas,\xi_i) \bigr)}^{p}\bigr] \eqsp. 
$$
The proof for other three terms is analogous, since each of the terms 
\begin{align}
\{Q_i\bigl(g(\theta_{i-1}^{\boot},\xi_i)-g(\thetas,\xi_i) \bigr)\}_{i=1}^{n-1}\eqsp,\{(w_i-1)Q_i\bigl(g(\theta_{i-1}^{\boot},\xi_i)-g(\thetas,\xi_i) \bigr)\}_{i=1}^{n-1}\eqsp,
\{(w_i-1)Q_iG(\theta_{i-1}^{\boot}-\thetas)\}_{i=1}^{n-1} \eqsp,
\end{align}
are martingale difference sequences with respect to
$(\widetilde{\mathcal F}_i)_{i=0}^{n-1}$ (see definition in \eqref{eq: extended filtration}). 
We finish the proof applying Minkowski's inequality.
\end{proof}

\begin{lemma}
\label{lem:bound_p_moment_for_H_component}
Assume \Cref{ass:L-smooth}- \Cref{ass:step_size_new_boot}.  Then for any $p\geq 2$ it holds
\begin{align}
    &\PE^{1/p}[\norm{D_2^{\boot}}^p]\leq M_{2,1}^{\boot}\rme^{1/p}p n^{1/2-\gamma}\eqsp,\\
\end{align}
\begin{equation}
\label{def:M_21_boot}
     M_{2,1}^{\boot}= 3C_QL_H\frac{c_0\max(K_2,K_1)(W_{\max}+1)}{2(1-\gamma)}\eqsp,
\end{equation}
and $K_1, K_2$ are defined in \eqref{eq:def_K_1}, \eqref{eq:def_K_2}, respectively.
\end{lemma}
\begin{proof}
 Using Minkowski's inequality, we get 
 \begin{equation}
     \label{eq: D_2 boot bound}
 \begin{split}
          \PE^{1/p}[\norm{D_2^{\boot}}^p] &\leq \frac{1}{\sqrt{n}}\PE^{1/p}[\norm{\sum_{i=1}^{n-1}Q_i H(\theta_{i-1})}^p]\\& +\frac{1}{\sqrt{n}}\PE^{1/p}[\norm{\sum_{i=1}^{n-1}(w_i-1) Q_{i}\biggl(H(\theta_{i-1}^{\boot})\biggr)}^p]\\
     &+ \frac{1}{\sqrt{n}}\PE^{1/p}[\norm{\sum_{i=1}^{n-1}Q_iH(\theta_{i-1}^{\boot})}^p]\eqsp. 
     \end{split}
\end{equation}
We will now consider each term separately.
   Using Minkowski's inequality together with \Cref{lem:H_theta_bound}, we obtain 
   \begin{equation}
   \begin{split}
       &\frac{1}{\sqrt{n}}\PE^{1/p}\biggl[\norm{\sum_{i=1}^{n-1}Q_iH(\theta_{i-1})}^{p}\biggr] \leq \frac{C_QL_H}{\sqrt{n}}\sum_{i=0}^{n-2}\PE^{1/p}\biggl[\norm{\theta_{i}-\thetas}^{2p}\biggr] \\&\qquad \qquad\leq \frac{C_QL_Hp}{\sqrt{n}}(\rme )^{1/p}K_2\sum_{i=0}^{n-2}\alpha_i \\&\qquad \qquad \leq \frac{C_QL_Hp}{\sqrt{n}}(\rme )^{1/p}K_2\biggl(c_0\frac{(k_0+n-2)^{1-\gamma}-(k_0-1)^{1-\gamma}}{1-\gamma} \biggr)\eqsp.
       \end{split}
   \end{equation}
   Since $k_0 \geq 1$ and $(k_0+n-2)^{1-\gamma}-(k_0-1)^{1-\gamma}\leq n^{1-\gamma}$ we complete the proof for the first term in the r.h.s. of \eqref{eq: D_2 boot bound}. The proof for other two terms is analogous.
\end{proof}

\section{Proof of Theorem \ref{cor:berry-esseen}}
\label{sec:Gar_real_world_sigma_infty}
The result follows immediately from the triangle inequality together with \Cref{th:bound_kolmogorov_dist_pr_sigma_n} and \Cref{lem:bound_kolmogorov_dist_sigma_n_sigma_infty}. Therefore, it remains to prove \Cref{lem:bound_kolmogorov_dist_sigma_n_sigma_infty} to complete the proof.
\subsection{Proof of Lemma \ref{lem:bound_kolmogorov_dist_sigma_n_sigma_infty}}
\label{sec: proof of difference between cov}
By definition of $\Sigma_n$ and $\Sigma_\infty$ we may write
\begin{align}
    \label{repr:Lstar_minus_Ln}
    \Sigma_n - \Sigma_{\infty} = \underbrace{\frac{1}{n}\sum_{t=1}^{n-1} (Q_t - G^{-1})\noisecov G^{-\top} + \frac{1}{n} \sum_{t=1}^{n-1} G^{-1} \noisecov (Q_t - G^{-1})^\top}_{D_1} + \\ + \underbrace{\frac{1}{n} \sum_{t=1}^{n-1} (Q_t - G^{-1}) \noisecov (Q_t - G^{-1})^{\top}}_{D_2} - \frac{1}{n} \Sigma_{\infty} \eqsp. 
\end{align}
The following lemma is an analogue of \citep[pp. 26-30]{wu2024statistical}. 

\begin{lemma} The following identities hold
\begin{equation}
    \label{repr:qtminusa}
    Q_i - G^{-1} = S_i - G^{-1} G_{i:n-1}^{(\alpha)} ,~ S_i = \sum_{j=i+1}^{n-1} (\alpha_i - \alpha_j) G_{i+1:j-1}^{(\alpha)} \eqsp,  
\end{equation}
and
\begin{equation}
    \label{repr:sumqtminusa}
    \sum_{i=1}^{n-1} (Q_i - G^{-1}) = -G^{-1} \sum_{j=1}^{n-1} G_{1:j}^{(\alpha)} \eqsp,
\end{equation}
where 
\begin{equation}
     G_{i:j}^{(\alpha)} = \prod_{k=i}^{j}(I-\alpha_kG)
\end{equation}
\end{lemma}

 \begin{proof}
To prove \eqref{repr:sumqtminusa} we first change the order of summation and then use the  properties of the telescopic sums we get
\begin{align}
\sum_{i=1}^{n-1} Q_i &= \sum_{i=1}^{n-1} \alpha_i \sum_{j=i}^{n-1}\prod_{k=i+1}^{j}(\Id-\alpha_k G) = \sum_{j=1}^{n-1} \sum_{i = 1}^j \alpha_i \prod_{k=i+1}^{j}(\Id-\alpha_k G) \\
& = \sum_{j=1}^{n-1} \sum_{i = 1}^j G^{-1} (\prod_{k = i+1}^{j} -  \prod_{k = i}^{j}) (\Id - \alpha_k G)  = G^{-1} \sum_{j=1}^{n-1} (\Id - \prod_{k = 1}^j (\Id - \alpha_k G)) \eqsp.
\end{align}
The proof of \eqref{repr:qtminusa} could be obtained by the following arguments. Note that
\begin{align}
\alpha_i G Q_i & = Q_i - (\Id - \alpha_i G) Q_i = \\
& = \alpha_i \Id + \alpha_i \sum_{j=i+1}^{n-1}\prod_{k=i+1}^{j}(\Id-\alpha_k G) - \alpha_i \sum_{j=i+1}^{n-1} \prod_{k=i}^{j-1}(\Id - \alpha_k G) - \alpha_i \prod_{k=i}^{n-1} (\Id - \alpha_k G)\eqsp. 
\end{align}  
It remains to note that
$$
\prod_{k=i+1}^{j}(\Id-\alpha_k G) - \prod_{k=i}^{j-1}(\Id - \alpha_k G) = (\alpha_i - \alpha_j) G \prod_{k=i+1}^{j-1} (\Id - \alpha_k G) \eqsp.
$$
The last two equations imply \eqref{repr:qtminusa}.
\end{proof}

\begin{lemma}
\label{bound_G_S}
It holds that 
\begin{enumerate}[(a)]
\item 
\begin{equation}
    \norm{S_i}\leq C_{S}(i + k_0)^{\gamma-1}\eqsp,
\end{equation}
where 
$$
C_{S} = 2c_0\exp\biggl\{\frac{\mu c_0}{k_0^{\gamma}}\biggr\}\biggl(2^{\gamma/(1-\gamma)}\frac{1}{\mu c_0} + (\frac{1}{\mu c_0})^{1/(1-\gamma)}\Gamma(\frac{1}{1-\gamma})\biggr)\eqsp.
$$
\item 
\begin{equation}
    \sum_{i=1}^{n-1}\norm{G_{i:n-1}^{(\alpha)}}^2\leq \frac{1}{1 - (1 - c_0\mu(n+k_0-2)^{-\gamma})^2}
\end{equation}
\item 
\begin{equation}
    \norm{\sum_{i=1}^{n-1}G_{i:n-1}^{(\alpha)}}\leq \frac{k_0^\gamma n^\gamma}{c_0\mu}
\end{equation}
\end{enumerate}
\end{lemma}
\begin{proof}
For simplicity we define $m_i^j = \sum_{k=i}^j (k+k_0)^{-\gamma}$.
Note that 
\begin{equation}
    \norm{\sum_{j=i+1}^{n-1} (\alpha_i - \alpha_j) G_{i+1:j-1}^{(\alpha)}} \leq \sum_{j=i}^{n-2} \frac{c_0}{(j+k_0+1)^\gamma}\biggl(\biggl(\frac{j+k_0+1}{i+k_0}\biggr)^\gamma - 1\biggr) \exp\{-\mu c_0m_{i+1}^{j}\}
\end{equation}
    Following the proof of \citep[Lemma A.5]{wu2024statistical}, we have 
    \begin{equation}
        \biggl(\frac{j+k_0+1}{i+k_0}\biggr)^\gamma - 1 \leq (i+k_0)^{\gamma-1}\biggl(1+ (1-\gamma)m_{i}^{j}\biggr)^{\gamma/(1-\gamma)}
    \end{equation}
    Hence, we obtain 
    \begin{align}
        \norm{S_i} &\leq c_0(i+k_0)^{\gamma-1}\sum_{j=i}^{n-2} \frac{1}{(j+k_0+1)^\gamma}\biggl(1+ (1-\gamma)m_{i}^{j}\biggr)^{\gamma/(1-\gamma)} \exp\{-\mu c_0m_{i+1}^{j}\}\\ &\leq
        c_0(i+k_0)^{\gamma-1}\sum_{j=i}^{n-2} \frac{1}{(j+k_0)^\gamma}\biggl(1+ (1-\gamma)m_{i}^{j}\biggr)^{\gamma/(1-\gamma)}\exp\{\mu c_0(k_0+i)^{-\gamma}\} \exp\{-\mu c_0m_{i}^{j}\}\\&\leq 
         c_0\exp\{\frac{\mu c_0}{k_0^{\gamma}}\}(i+k_0)^{\gamma-1}\sum_{j=i}^{n-2} (m_i^j-m_i^{j-1})\biggl(1+ (1-\gamma)m_{i}^{j}\biggr)^{\gamma/(1-\gamma)} \exp\{-\mu c_0m_{i}^{j}\}\\ &\leq 
         2c_0\exp\{\frac{\mu c_0}{k_0^{\gamma}}\}(i+k_0)^{\gamma-1}\int_{0}^{+\infty}\biggl(1+ (1-\gamma)m\biggr)^{\gamma/(1-\gamma)} \exp\{-\mu c_0m\}\rmd m \\ &\leq 
         2c_0\exp\{\frac{\mu c_0}{k_0^{\gamma}}\}(i+k_0)^{\gamma-1}\biggl(2^{\gamma/(1-\gamma)}\frac{1}{\mu c_0} + (\frac{1}{\mu c_0})^{1/(1-\gamma)}\Gamma(\frac{1}{1-\gamma})\biggr)\eqsp.
    \end{align}
    Note that 
\begin{align}
    \norm{\sum_{i=1}^{n-1}G_{i:n-1}^{(\alpha)}}&\leq \sum_{i=1}^{n-1}\prod_{k=i}^{n-1}(1-\alpha_k\mu) = \sum_{i=1}^{n-1}\prod_{k=i}^{n-1}\alpha_{i-1}^{-1}\alpha_{i-1}(1-\alpha_k\mu) \\&\leq\frac{(k_0+n-2)^\gamma}{c_0\mu}\sum_{i=1}^{n-1}\biggl(\prod_{k=i}^{n-1}(1-\alpha_k\mu)-\prod_{k=i-1}^{n-1}(1-\alpha_k\mu)\biggr)\leq \frac{k_0^{\gamma}n^{\gamma}}{\mu c_0}\eqsp,
\end{align}
where in the last inequality we use that $(k_0+n-2)^\gamma \leq (k_0n)^{\gamma}$.
    Bound for $\sum_{i=1}^{n-1}\norm{G_{i:n-1}^{(\alpha)}}^2$ is obtained similarly to $ \norm{\sum_{i=1}^{n-1}G_{i:n-1}^{(\alpha)}}$.
\end{proof}

To finish the proof of Lemma \ref{lem:bound_kolmogorov_dist_sigma_n_sigma_infty} we need to bound $D_1, D_2$. By \eqref{repr:sumqtminusa} we obtain
\begin{align}
\normop{\frac{1}{n}\sum_{i=1}^{n-1} (Q_i - G^{-1})\noisecov G^{-\top}} &= \normop{-\frac{1}{n} G^{-1} \sum_{j=1}^{n-1} G_{1:j}^{(\alpha)}\noisecov G^{-\top}} \\ &= \normop{n^{-1} \Sigma_{\infty}  \sum_{j=1}^{n-1} G_{1:j}^{(\alpha)}} \leq n^{-1} \normop{\Sigma_{\infty}} \cdot \normop{\sum_{j=1}^{n-1} G_{1:j}^{(\alpha)}} \eqsp. 
    \end{align}
    It remains to apply \Cref{lem:bound_Q_i_and_Sigma_n} which gives
    \begin{align}
        \normop{\frac{1}{n}\sum_{i=1}^{n-1} (Q_i - G^{-1})\noisecov G^{-\top}}  \leq \normopadapt{\Sigma_{\infty}}C_Q \frac{k_0^\gamma n^{\gamma-1}}{c_0}
    \end{align}
    Hence, 
    \begin{align}
        \label{lemma:D1_bound}
        \normop{D_1} \leq 2 \normopadapt{\Sigma_{\infty}}C_Q \frac{k_0^\gamma n^{\gamma-1}}{c_0}
    \end{align}
To bound $D_2$ we use \eqref{repr:qtminusa} which gives
    \begin{align}
        \label{lemma:D2_expanded_repr}
        &n^{-1}\sum_{i=1}^{n-1}(Q_i - G^{-1}) \noisecov (Q_i - G^{-1})^\top \\
        &\qquad = n^{-1}\sum_{i=1}^{n-1}\bigl(S_i - G^{-1} \prod_{k=i}^{n-1} (\Id - \alpha_k G)\bigr) \noisecov \bigl(S_i - G^{-1} \prod_{k=i}^{n-1} (\Id - \alpha_k G)\bigr)^\top \\
        &\qquad = \underbrace{n^{-1} \sum_{i=1}^{n-1} S_i \noisecov S_i^\top}_{D_{21}} + \underbrace{n^{-1}\sum_{i=1}^{n-1}G^{-1} \prod_{k=i}^{n-1} (\Id - \alpha_k G) \noisecov G^{-\top} \prod_{k=i}^{n-1} (\Id - \alpha_k G)^\top}_{D_{22}} \\ 
        &\qquad\qquad \qquad  -\underbrace{n^{-1} \sum_{i=1}^{n-1} G^{-1} \prod_{k=i}^{n-1} (\Id - \alpha_k G) \cdot \noisecov S_i^\top}_{D_{23}} - \underbrace{n^{-1}\sum_{i=1}^{n-1} S_i \noisecov G^{-\top} \prod_{k=i}^{n-1} (\Id - \alpha_k G)^{\top}}_{D_{24}}\eqsp.
    \end{align}
    To bound $D_{21}$ we use  \Cref{bound_G_S}, and obtain 
    \begin{align}
        \label{lemma:D21_bound}
        \normop{D_{21}} = \normop{n^{-1} \sum_{i=1}^{n-1} S_i \noisecov S_i^\top} &\leq n^{-1} \sum_{i=1}^{n-1} \normop{\noisecov} \normop{S_i}^2 \\  &\leq
        n^{-1}\normop{\noisecov} C_{S}^2\sum_{i=1}^{n-1} (i + k_0)^{2(\gamma-1)}\\ &\leq n^{-1}\normop{\noisecov} C_{S}^2 \frac{(n + k_0-1)^{2\gamma-1}-k_0^{2\gamma-1}}{2\gamma-1}
        \\ &\leq \normop{\noisecov} C_{S}^2 \frac{n^{2(\gamma-1)}}{2\gamma-1}
    \end{align}
    The bound for $D_{22}$ follows from  \Cref{bound_G_S}
    \begin{align} 
        \label{lemma:D22_bound}
        \normop{D_{22}} &= \normop{n^{-1}\sum_{i=1}^{n-1} \prod_{k=i}^{n-1} (\Id - \alpha_k G) G^{-1} \noisecov G^{-\top} \prod_{k=i}^{n-1} (\Id - \alpha_k G)^\top} \leq n^{-1} \normop{\Sigma_\infty} \sum_{i=1}^{n-1} \normop{G_{i:n-1}^{(\alpha)}}^2 \\ 
        &\leq 
        n^{-1} \frac{\normop{\Sigma_\infty}}{ 2c_0\mu(n+k_0-2)^{-\gamma} - c_0^2\mu^2(n+k_0-2)^{-2\gamma}} \leq 
         \normop{\Sigma_\infty} k_0^{\gamma}\frac{n^{\gamma-1}}{c_0\mu}\eqsp.
    \end{align}
    Since $D_{23} = D_{24}^\top$, we concentrate on $\normop{D_{24}}$. \Cref{bound_G_S} immediately imply 
    \begin{align}
        \normop{D_{24}} & \leq n^{-1} \normop{\noisecov G^{-\top}} \sum_{i=1}^{n-1} \normop{S_i} \normop{\prod_{k=i}^{n-1} (\Id - \alpha_k G)^{\top}} \\ &\leq n^{-1} \normop{\noisecov}\frac{1}{\mu}C_{S}\sum_{i=1}^{n-1}(i+k_0)^{\gamma-1}\prod_{k=i}^{n-1}(1-\mu\frac{c_0}{(k+k_0)^\gamma}) \\ &\leq n^{-1} \normop{\noisecov}\frac{1}{\mu}C_{S}\sum_{i=1}^{n-1}(i+k_0)^{2\gamma-1}(i+k_0)^{-\gamma}\prod_{k=i+1}^{n-1}(1-\mu\frac{c_0}{(k+k_0)^\gamma})\\ &\leq \normop{\noisecov}C_{S}k_0^{2\gamma-1}\frac{n^{2(\gamma-1)}}{\mu^2c_0}
    \end{align}
    Combining all inequalities above, we obtain
    \begin{equation}
    \label{eq:difference_sigma_n_infty_appendix}
        \norm{\Sigma_n-\Sigma_{\infty}} \leq C_{\infty}'n^{\gamma-1}\eqsp,
    \end{equation}
    where
    \begin{equation}
    \label{eq:def_C_infty_prime}
        C_{\infty}' = (\frac{k_0^{\gamma}}{c_0\mu} +2C_{Q} \frac{k_0^\gamma}{c_0} +1)\normopadapt{\Sigma_{\infty}} +  (C_{S}^2 \frac{1}{2\gamma-1}+
        2C_{S}\frac{k_0^{2\gamma-1}}{\mu^2c_0})\normop{\noisecov}\eqsp.
    \end{equation}
To finish the proof it remains to apply Lemma \ref{Pinsker}, since
\begin{equation}
\label{eq:def_C_infty}
3/2\|\Sigma_n^{-1/2}\Sigma_{\infty}\Sigma_n^{-1/2}-I\|_{\mathsf{F}}\leq C_{\infty}n^{\gamma-1}\eqsp, \text{ where } C_{\infty}=3/2 \sqrt{d} C_\Sigma^2 C_{\infty}'\eqsp.
\end{equation}

\section{Lower bounds}
\label{sec:proof_lower_bound}
In the following computations we provide a lower bound on the quantity $\bigl| \frac{1}{n}\sum_{j=1}^{n-1}Q_{j}^2 - 1 \bigr|$, provided that the number of observations $n$ is large enough. For simplicity in this bound we consider $k_0 = 1$. We first note that 
\begin{align}
\frac{1}{n}\sum_{j=1}^{n-1}Q_{j}^2 - 1 = \frac{1}{n}\sum_{j=1}^{n-1}(Q_j - 1)(Q_j + 1) - \frac{1}{n} = \frac{T_1}{n} + \frac{T_2}{n}\eqsp, 
\end{align}
where 
\[
T_1 = \sum_{j=1}^{n-1}(Q_j - 1)^2\eqsp, \quad T_2 =  2\sum_{j=1}^{n-1}(Q_j - 1) - 1\eqsp,
\]
and treat the terms $T_1$ and $T_2$ separately. Using the identity \eqref{repr:sumqtminusa}, we get, since $G = 1$, that 
\[
\sum_{j=1}^{n-1}(Q_j - 1) = -\sum_{j=1}^{n-1}\prod_{\ell=1}^{j}(1 - \alpha_{\ell})\eqsp.
\]
Fix $J$ such that $\alpha_\ell\le 1/2$ for all $\ell\ge J$.
For $j\ge J$,
\[
\left|\prod_{\ell=1}^j(1-\alpha_\ell)\right|
\le
\left|\prod_{\ell=1}^{J-1}(1-\alpha_\ell)\right|
\exp\left\{-\sum_{\ell=J}^j\alpha_\ell\right\}.
\]
Since $\sum_{\ell=J}^j\alpha_\ell\asymp j^{1-\gamma}$ as $j\to\infty$,
\eqref{repr:sumqtminusa} gives
\[
|T_2|\le 1+2\sum_{j=1}^{\infty}
\left|\prod_{\ell=1}^j(1-\alpha_\ell)\right|<\infty.
\]
 
 Here we notice that, applying \eqref{repr:qtminusa}, 
\[
Q_i - 1 = S_i - \prod_{\ell=i}^{n-1}(1-\alpha_{\ell})\eqsp, \quad S_i = \sum_{j=i+1}^{n-1} (\alpha_i - \alpha_j) \prod_{\ell=i+1}^{j-1}(1 - \alpha_{\ell})\eqsp.
\]
Since $T_1$ is a sum of squares, we may restrict the sum to
$j\ge J$ to obtain a lower bound.
\begin{equation}
\label{eq:lower_bound_decomposition}
T_1 \geq \sum_{j=J}^{n-1}S_j^2 - 2 \sum_{j=J}^{n-1} S_j \prod_{\ell=j}^{n-1}(1-\alpha_{\ell}) + \sum_{j=J}^{n-1} \prod_{\ell=j}^{n-1} (1-\alpha_{\ell})^2\eqsp.
\end{equation}
Due to item (a) from \Cref{bound_G_S}, it holds that $|S_j| \leq C_{S} / (j+1)^{1-\gamma}$. Hence, similarly to  the proof of \Cref{lem:bound_kolmogorov_dist_sigma_n_sigma_infty} we can show that 
\[
\frac{1}{n} |\sum_{j=J}^{n-1}S_j^2|\leq C_s^2n^{2(\gamma-1)}/(2\gamma-1)\eqsp,
\]
and 
\[
\frac{1}{n} |\sum_{j=J}^{n-1} S_j \prod_{\ell=j}^{n-1}(1-\alpha_{\ell})| \leq C_S n^{2(\gamma-1)}/c_0\eqsp.
\]
Now, we proceed with the last term in \eqref{eq:lower_bound_decomposition}, and provide a lower bound on the last remaining component of $T_{1}$ in \eqref{eq:lower_bound_decomposition}, that is, 
\[
T_{3} = \sum_{j=J}^{n-1} \prod_{\ell=j}^{n-1} (1-\alpha_{\ell})^2\eqsp.
\]

Since $\alpha_\ell\le 1/2$ for $\ell\ge J$, we get, using an elementary inequality $1 - x \geq \exp\{-2x\}$, valid for $0 \leq x \leq 1/2$, we get that 
\begin{align}
T_3
&\ge \sum_{j=J}^{n-1}
    \exp\left\{-4\sum_{\ell=j}^{n-1}\alpha_\ell\right\}\\
&\ge \sum_{j=1}^{n-1}
    \exp\left\{-4\sum_{\ell=j}^{n-1}\alpha_\ell\right\}
    -(J-1)\\
&\geq \sum_{j=1}^{n-1} \exp\biggl\{ - \frac{4c_0}{1-\gamma}(n^{1-\gamma} - j^{1-\gamma}) \biggr\}-(J-1) \\
&=\exp\biggl\{ - \frac{4c_0}{1-\gamma}n^{1-\gamma}\biggr\}\sum_{j=1}^{n-1} \exp\biggl\{ \frac{4c_0}{1-\gamma} j^{1-\gamma} \biggr\}-(J-1)
\end{align}
Now we get that 
\begin{align}
\sum_{j=1}^{n-1} \exp\biggl\{ \frac{4c_0}{1-\gamma} j^{1-\gamma} \biggr\} 
&\geq \int_{0}^{n - 1} \exp\biggl\{ \frac{4c_0}{1 - \gamma} y^{1 - \gamma} \biggr\} \, \rmd y \\
&= (n-1) \int_{0}^{1}\exp\biggl\{ \frac{4c_0}{1 - \gamma} ((n-1)z)^{1 - \gamma} \biggr\} \,\rmd z\eqsp.
\end{align}
 Now we proceed with Laplace approximation (see e.g. \citep{fedoryuk1977metod} or \citep{olver1997asymptotics}) for the inner integral: 
\begin{align}
    \int_{0}^{1}\exp\biggl\{ \frac{4c_0}{1 - \gamma} ((n-1)z)^{1 - \gamma} \biggr\} \,\rmd z = \exp\biggl\{ \frac{4c_0}{1 - \gamma} (n-1)^{1 - \gamma} \biggr\}\frac{(n-1)^{\gamma-1}}{4c_0}\bigl[1 + \mathcal{O}(n^{\gamma-1})\bigr]
\end{align}
Since $n^{1-\gamma}-(n-1)^{1-\gamma} \leq 1$ and $\frac{n-1}{n}\geq 1/2$ for $n \geq 2$, we get 
\[
\frac{T_3}{n}
=
\frac{1}{n}\sum_{j=J}^{n-1}\prod_{\ell=j}^{n-1}(1-\alpha_\ell)^2
\ge
\frac{1}{4c_0}\exp\!\left\{-\frac{8c_0}{1-\gamma}\right\}
\frac{1}{(n-1)^{1-\gamma}}
+ O\!\left(n^{2(\gamma-1)}\right).
\]
Hence, we conclude that for $n$ large enough, 
\[
|\sigma^2_{n,\gamma} - 1| > \frac{C_1(\gamma,c_0)}{n^{1-\gamma}}\eqsp,
\]
and the statement follows. For two centered univariate Gaussian laws with distinct variances,
the density with the smaller variance exceeds the other on an interval,
so the total variation distance is attained on a convex set. To prove the second part, it remains to apply the lower bound on the total variation distance between Gaussian random vectors given in \citep[Theorem~1.1]{Devroye2018}.

\subsection{Numerical demonstration}
\label{sec:numerical_demonstartion}
In order to illustrate numerically the tightness of bounds provided in \Cref{prop:lower_bounds}, we consider the following simple experiment. We consider the statistics
\[
|\sigma^2_{n,\gamma}-1| \cdot n^{1-\gamma}\eqsp, \quad n \in \{2^{10}, \ldots, 2^{27}\}\eqsp.
\]
We illustrate numerically the tightness of our bound in the Figure \ref{fig:results-expe} below by calculating 
\[
n^{1-\gamma} \cdot |\sigma^2_{n,\gamma}-1|
\]
for different values of $\gamma \in \{0.5,\ldots,0.9\}$ and $n$. Here we fix the values of parameter $k_0 = 1$ and $c_0 = 1$. Code to reproduce the plot is provided in \url{https://anonymous.4open.science/r/gaussian_approximation_sgd-DBDD/}. 
\begin{figure}
\centering
\includegraphics[width=0.9\linewidth]{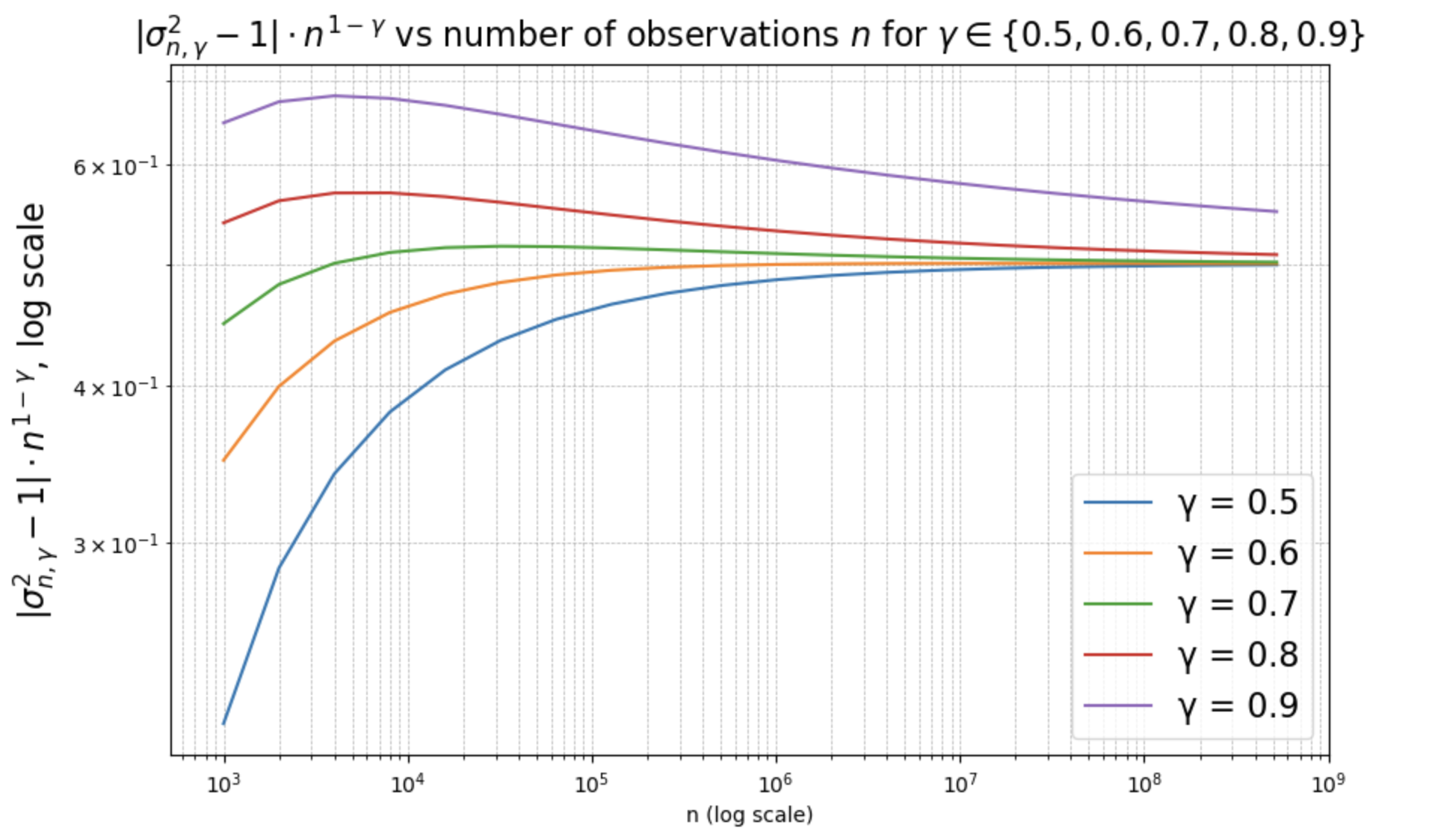}
\caption{Numerical verification of the lower bound given in \Cref{prop:lower_bounds}}
\label{fig:results-expe}
\end{figure}

\section{Results for the last iterate}
\subsection{Last iterate bound}
\label{sec:last_iterate}

To prove Theorem \ref{th:bound_kolmogorov_dist_pr_sigma_n}, we need a bound on the $2p$-th moment for $p\geq 1$ of the last iterate. Our approach is based on induction: first, we establish the result for $p = 1$, and then we show how to proceed from $2(p - 1)$ to $2p$.

\begin{lemma}
\label{lem:bound_last_iter_second_moment}
Assume \Cref{ass:L-smooth}, \Cref{ass:hessian_Lipschitz_ball}, \Cref{ass:noise_decomposition}($2$), and  \Cref{ass:step_size}. Then for any $k \in \nset$ it holds that 
\begin{equation}
\label{eq:bound_last_iter_second_moment}
\PE[\norm{\theta_k - \thetas}^{2}] \leq C_{1} \exp\biggl\{-\frac{\mu c_0}{4} (k+k_0)^{1-\gamma}\biggr\} \bigl[ \norm{\theta_0-\thetas}^2 + \sigma_2^2\bigr] + C_{2} \sigma_2^2 \alpha_{k}  \eqsp,
\end{equation}
where $\sigma_2^2$ is defined in \Cref{ass:noise_decomposition}($2$), and the constants $C_1$ and $C_2$ are given by 
\begin{equation}
\label{eq:const_C_1_C_2_def}
\begin{split}
C_{1} &= \exp\biggl\{\frac{3\mu c_0}{4(1-\gamma)}k_0^{1-\gamma}\biggr\}\biggl(\left(1 + L_2^{-2}\right) \exp\biggl\{\frac{6c_0^2L_2^2}{2\gamma-1}\biggr\} + \frac{2c_{0}^2}{2\gamma-1}\biggr) \eqsp,\\
C_{2} &= \frac{2^{1+\gamma}}{\mu}\eqsp.
\end{split}
\end{equation}
\end{lemma}
\begin{proof}
From \eqref{eq:sgd_recursion_main} and \cref{ass:noise_decomposition} it follows that 
    \begin{equation}
        \norm{\theta_k-\thetas}^2 = \norm{\theta_{k-1}-\thetas}^2-2\alpha_k\langle \theta_{k-1}-\thetas , \nabla f(\theta_{k-1})+\zeta_k\rangle + \alpha_k^2\norm{\nabla f(\theta_{k-1})+\zeta_k}^2.
    \end{equation}
    Using \Cref{ass:L-smooth} and \Cref{ass:noise_decomposition}($2$), we obtain
    \begin{align}
    2\alpha_k\PE[\langle \theta_{k-1}-\thetas , \nabla f(\theta_{k-1})+\zeta_k\rangle|\mathcal{F}_{k-1}]
    &= 2\alpha_k\langle \theta_{k-1}-\thetas , \nabla f(\theta_{k-1})-\nabla f(\thetas) \rangle.     
    \end{align}
    Using \Cref{ass:noise_decomposition}($2$) and \Cref{ass:L-smooth}, we get 
    \begin{align}
        \PE[\norm{\nabla f(\theta_{k-1})+\zeta_k}^2|\mathcal{F}_{k-1}] 
        &= \norm{\nabla f(\theta_{k-1})-\nabla f(\thetas)}^2 + \PE[\norm{\eta(\xi_k) + g(\theta_{k-1}, \xi_k)}^2|\mathcal{F}_{k-1}] \\ &\leq  L_1 \pscal{\nabla f(\theta_{k-1}) - \nabla f(\thetas)}{\theta_{k-1}- \thetas} + 2L_2^2 \norm{\theta_{k-1}-\thetas}^2 + 2\sigma_2^2\eqsp.
    \end{align}

  Combining the above inequalities, we obtain 
   \begin{equation}
   \label{eq:recursion_second_moment}
       \PE [\norm{\theta_k-\thetas}^2]\leq (1-\mu \alpha_k (2 - \alpha_k L_1) +2 \alpha_k^2 L_2^2))\PE[\norm{\theta_{k-1}-\thetas}^2]+2\alpha_k^2\sigma_2^2.
    \end{equation}
    By applying the recurrence \eqref{eq:recursion_second_moment}, we obtain that 
    \begin{align}
         \PE[\norm{\theta_k-\thetas}^2] 
         \leq A_{1,k} \norm{\theta_0-\thetas}^2 + 2 \sigma_2^2 A_{2,k}\eqsp, 
    \end{align}
where we have set 
\begin{equation}
\label{eq:A_1_2_k_def}
\begin{split}
A_{1,k} &= \prod_{i=1}^k(1- (3/2)\alpha_i\mu+2 \alpha_i^2 L_2^2)\eqsp, \\
A_{2,k} &= \sum_{i=1}^k\prod_{j=i+1}^k(1-(3/2) \alpha_j\mu+2\alpha_j^2 L_2^2)\alpha_i^2\eqsp.
\end{split}
\end{equation}
Using the elementary bound $1+t \leq \rme^{t}$ for any $t\in \rset$, we get 
\begin{equation}
    A_{1,k} \leq \exp\biggl\{-(3/2) \mu\sum_{i=1}^k\alpha_i \biggr\}\exp\biggl\{2L_2^2\sum_{i=1}^k\alpha_i^2\biggr\}\eqsp.
\end{equation}
Using \Cref{lem:bounds_on_sum_step_sizes}, we obtain 
\begin{equation}
\label{eq:A_1_k_bound}
A_{1,k} \leq c_{1} \exp\biggl\{-\frac{3\mu c_0}{4(1-\gamma)} (k+k_0)^{1-\gamma}\biggr\}\eqsp,
\end{equation}
where we have set 
\begin{equation}
\label{eq:c_1_small}
c_1 = \exp\biggl\{\frac{2c_0^2L_2^2}{2\gamma-1} + \frac{3\mu c_0}{4(1-\gamma)}k_0^{1-\gamma}\biggr\}\eqsp.
\end{equation}
Now we estimate $A_{2,k}$. Let $k_1$ be the largest index $k$ such that $4 \alpha_k^2 L_2^2 \geq \alpha_k\mu$ and set $k_1=0$ if the defining set is empty. Then, for $i > k_1$, we have that 
\[
1- (3/2) \alpha_i \mu+2 \alpha_i^2 L_2^2 \leq 1 - \alpha_i \mu \eqsp.
\]
Thus, using the definition of $A_{2,k}$ in \eqref{eq:A_1_2_k_def}, we obtain that 
\begin{align}
    A_{2,k}  \leq \sum_{i=1}^k \alpha_i^2 \prod_{j=i+1}^k(1-\alpha_j\mu) +  \sum_{i=1}^{k_1}\alpha_i^2 \biggl\{ \prod_{j=i+1}^{k_1}(1+2\alpha_j^2 L_2^2) \biggr\} \biggl\{ \prod_{j=k_1+1}^{k}(1-\alpha_j\mu) \biggr\}\eqsp.
\end{align}
Note that
\begin{align}
    \sum_{i=1}^{k_1}\alpha_i^2 \prod_{j=i+1}^{k_1}(1+2 \alpha_j^2L_2^2) 
    &=\frac{1}{ 2L_2^2}\sum_{i=1}^{k_1}\biggl(\prod_{j=i}^{k_1}(1+2 \alpha_j^2 L_2^2)-\prod_{j=i+1}^{k_1}(1+2 \alpha_j^2L_2^2)\biggr) \\
    &\leq \frac{1}{2L_2^2}\prod_{j=1}^{k_1}(1+2\alpha_j^2 L_2^2) \leq \frac{1}{2L_2^2}\exp\biggl\{2L_2^2\sum_{j=1}^{k_1}\alpha_j^2\biggr\}.
\end{align}
Note, that for $k\leq k_1$, $\alpha_k \geq \mu/(4L_2^2)$, hence, we have 
\begin{equation}
    \prod_{j=k_1+1}^k(1-\alpha_j\mu) \leq \exp\biggl\{-\mu\sum_{i=1}^k\alpha_i\biggr\}\exp\biggl\{\mu\sum_{i=1}^{k_1}\alpha_i\biggr\} \leq \exp\biggl\{-\mu\sum_{i=1}^k\alpha_i\biggr\}\exp\biggl\{4L_2^2 \sum_{i=1}^{k_1}\alpha_i^2\biggr\}\eqsp.
\end{equation}
Moreover, for any $m\in\{0,\ldots ,k\}$, we obtain 
\begin{align}
&\sum_{i=1}^k \alpha_i^2 \prod_{j=i+1}^k(1-\alpha_j\mu) = \sum_{i=1}^m\prod_{j=i+1}^k(1-\alpha_j\mu)\alpha_i^2 + \sum_{i=m+1}^k\prod_{j=i+1}^k(1-\alpha_j\mu)\alpha_i^2 \\
&\qquad \qquad \leq \prod_{j=m+1}^k(1-\alpha_j\mu)\sum_{i=1}^{m}\alpha_i^2 +\alpha_m\sum_{i=m+1}^k\prod_{j=i+1}^k(1-\alpha_j\mu)\alpha_i \\
&\qquad \qquad \leq \exp\biggl\{-\mu\sum_{j=m+1}^k\alpha_j\biggr\}\sum_{i=1}^{m}\alpha_i^2 + \frac{\alpha_m}{\mu}\sum_{i=m+1}^k\biggl(\prod_{j=i+1}^k(1-\alpha_j\mu)-\prod_{j=i}^k(1-\alpha_j\mu)\biggr) \\
&\qquad \qquad  \leq \exp\biggl\{-\mu\sum_{j=m+1}^k\alpha_j\biggr\}\sum_{i=1}^{m}\alpha_i^2 + \frac{\alpha_m}{\mu}\biggl(1-\prod_{j=m+1}^k(1-\alpha_j\mu)\biggr) \\
&\qquad \qquad \leq  \exp\biggl\{-\mu\sum_{j=m+1}^k\alpha_j\biggr\}\sum_{i=1}^{m}\alpha_i^2 + \frac{\alpha_m}{\mu}\eqsp.
\end{align}
Thus, setting $m= \lfloor k/2 \rfloor $, and using the definition of $A_{2,k}$ in \eqref{eq:A_1_2_k_def}, we obtain that 
\begin{align}
A_{2,k} &\leq \exp\biggl\{\frac{-\mu c_0}{2(1-\gamma)}((k+k_0)^{1-\gamma}-(\lfloor k/2\rfloor+k_0)^{1-\gamma}))\biggr\} \frac{c_0^2}{2\gamma-1} + \frac{c_0}{\mu(k_0+\lfloor k/2 \rfloor)^{\gamma}} \\& \qquad \qquad \qquad \qquad+ c_2 \exp\biggl\{-\frac{\mu c_0}{2(1-\gamma)}(k+k_0)^{1-\gamma}\biggr\}\eqsp,
\end{align}
where we have set
\begin{equation}
\label{eq:c_2_small}
c_2 = \frac{1}{2L_2^2} \exp\biggl\{\frac{6c_0^2 L_2^2}{2\gamma-1}+ \frac{\mu c_0}{2(1-\gamma)}k_0^{1-\gamma}\biggr\}\eqsp.
\end{equation}
Using that $\lfloor k/2 \rfloor \leq k/2$ together with the elementary inequality 
\[
\frac{x^{\beta}}{\beta} - \frac{(x/2)^{\beta}}{\beta} \geq \frac{x^{\beta}}{2}\eqsp,
\]
which is valid for $\beta \in (0,1]$, and $\frac{c_0}{\mu(k_0+\lfloor k/2 \rfloor)^{\gamma}} \leq \frac{2^{\gamma} c_0}{\mu (k+k_0)^{\gamma}}$, we obtain that 
\begin{align}
A_{2,k} &\leq \exp\biggl\{-\frac{\mu c_0}{4} (k+k_0)^{1-\gamma}\biggr\}\exp\biggl\{\frac{\mu c_0}{2(1-\gamma)}k_0^{1-\gamma}\biggr\}\frac{c_0^2}{2\gamma-1} + \frac{2^{\gamma} c_0}{\mu (k+k_0)^{\gamma}} \\& \qquad \qquad \qquad \qquad + c_2 \exp\biggl\{\frac{-\mu c_0}{2(1-\gamma)}(k+k_0)^{1-\gamma}\biggr\}\eqsp.
\end{align}
Combining the bounds for $A_{1,k}$ and $A_{2,k}$, we obtain that 
\begin{align}
\PE[\norm{\theta_k-\thetas}^2] 
&\leq c_{1} \exp\biggl\{-\frac{3\mu c_0}{4(1-\gamma)} (k+k_0)^{1-\gamma}\biggr\} \norm{\theta_0 - \thetas}^2 \\
&+ \exp\biggl\{-\frac{\mu c_0}{4} (k+k_0)^{1-\gamma}\biggr\} \frac{2c_0^2 \sigma_2^2}{2\gamma-1}\exp\biggl\{\frac{\mu c_0}{2(1-\gamma)}k_0^{1-\gamma}\biggr\}+  \frac{2^{1+\gamma} c_0 \sigma_2^2}{\mu (k+k_0)^{\gamma}} \\
&+ 2 c_2 \sigma_2^2 \exp\biggl\{\frac{-\mu c_0}{2(1-\gamma)}(k+k_0)^{1-\gamma}\biggr\} \\
&\leq C_{1} \exp\biggl\{-\frac{\mu c_0}{4} (k+k_0)^{1-\gamma}\biggr\} \bigl[ \norm{\theta_0-\thetas}^2 + \sigma_2^2\bigr] + C_{2} \sigma_2^2\alpha_{k}  \eqsp,
\end{align}
where we have set constants $C_1$ and $C_2$ using the definitions of $c_1$ and $c_2$ from \eqref{eq:c_1_small} and \eqref{eq:c_2_small}. 
\end{proof}

The first term in the bound from \Cref{lem:bound_last_iter_second_moment} decays exponentially with $k$, which implies that $\mathbb{E}^{1/2}[\|\theta_k - \theta^*\|^2] \lesssim \sqrt{\alpha_k}$. Below, we state this result with an explicit constant, as this specific form of the bound for the last iteration is needed for the induction step.

\begin{corollary}
\label{cor:second_moment_bound}
Under the assumptions of \Cref{lem:bound_last_iter_second_moment},  it holds that 
\begin{equation}
\PE[\norm{\theta_k-\thetas}^2]\leq D_1 (\norm{\theta_0-\thetas}^2 + \sigma_2^2) \alpha_k\eqsp,
\end{equation}
where 
\[
D_1 = C_1 (1/c_0 + C_{2}) \biggl(\frac{4\gamma}{(1-\gamma)\mu c_0 \rme}\biggr)^{\gamma/(1-\gamma)}\eqsp.
\]
\end{corollary}
\begin{proof}
Set
\[
C_3 =
\left(\frac{4\gamma}{(1-\gamma)\mu c_0\rme}\right)
^{\gamma/(1-\gamma)}.
\]
For every $t>0$, we have
\[
\exp\{-\mu c_0t^{1-\gamma}/4\}
\leq C_3t^{-\gamma}.
\]
Evaluating this inequality at $t=k_0$ and using the
definition of $C_1$, we obtain
\[
C_1C_3
\geq k_0^\gamma
\exp\left\{
\frac{\mu c_0k_0^{1-\gamma}}{4}
\left(\frac{3}{1-\gamma}-1\right)
\right\}
\geq 1.
\]
Consequently,
\[
\frac{C_1C_3}{c_0}+C_2
\leq C_1C_3(1/c_0+C_2)=D_1.
\]
Combining these inequalities with
\Cref{lem:bound_last_iter_second_moment}
proves the claim.
\end{proof}
Now we provide bound for p-moment of last iterate.
\begin{proposition}  \label{prop:2p-moment-bound} 
Let $p\geq2$ be an integer. Assume \Cref{ass:L-smooth}, \Cref{ass:hessian_Lipschitz_ball}, \Cref{ass:noise_decomposition}($2p$), and \Cref{ass:step_size}. Then for any $k \in \nset$ it holds that
\begin{equation}
\PE[\norm{\theta_k - \thetas}^{2p}] \leq C_{2p,1}\exp\biggl\{-\frac{p\mu c_0}{4}(k+k_0)^{1-\gamma}\biggr\}(\norm{\theta_0-\thetas}^{2p} + \sigma_{2p}^{2p}) + C_{2p,2}\sigma_{2p}^{2p}\alpha_k^p\eqsp, 
\end{equation}
where 
\begin{equation}
    C_{2p,1} = 2^{2p-1}(D_{2(p-1)}C_4^pc_0^p + 1)c_4\eqsp,
\end{equation}
\begin{equation}
    C_{2p,2} =2^{2p-1}D_{2(p-1)}C_4^p\frac{2^{1+\gamma p}}{\mu pc_0} \eqsp,
\end{equation}
constants $D_{2(p-1)}$ are defined in \eqref{eq:const_D_2p_def}, and  
\begin{equation}
\begin{split}
    &C_4^p = (4c_0^{1/2}2^{\gamma/2}+2^{\gamma}+ 4c_0)^p\\
    &c_4=
\left(
\exp\left\{
\frac{8c_0^2p^2(L_1+L_2)^2}{2\gamma-1}
\exp\{5pc_0(L_1+L_2)\}
\right\}+1
\right)
\exp\left\{
\frac{p\mu c_0}{1-\gamma}k_0^{1-\gamma}
\right\}
\frac{1}{\min\{\gamma(p+1)-1,2\}}.
\end{split}
\end{equation}
\end{proposition}
\begin{proof}

We prove the statement 
by induction on the integer $p\geq2$.
Set $D_2:=D_1$, where $D_1$ is defined in \Cref{cor:second_moment_bound}.
We first assume that $\theta_0=\theta^\star$ and then
provide a result for an arbitrary initial condition. Assume that for any $t\leq p-1$ and all $k \in \nset$ we proved that 
\begin{equation}
\label{eq:induction_assumption_p_moment}
\PE[\norm{\theta_k - \thetas}^{2t}] \leq D_{2t} \sigma_{2t}^{2t} \alpha_{k}^{t}\eqsp,
\end{equation}
and the sequence of constants $\{D_{2t}\}$ is non-decreasing in $t$. 
By Lyapunov's inequality, \Cref{ass:noise_decomposition}($2p$) implies \Cref{ass:noise_decomposition}($2t$)
with the same noise bound $\sigma_{2p}$, for every
integer $1\leq t\leq p-1$.
Applying the induction hypothesis with this bound gives
\begin{equation}
\PE[\|\theta_k-\theta^\star\|^{2t}]
\leq D_{2t}\sigma_{2p}^{2t}\alpha_k^t.
\end{equation}

For any $k\in \nset$ we denote $\delta_k= \norm{\theta_k - \thetas}$. Using \eqref{eq:sgd_recursion_main}, we get 
\begin{align}
\delta_k^{2p} 
&= \left(\delta_{k-1}^2 - 2\alpha_k\langle\theta_{k-1}-\thetas, \nabla f(\theta_{k-1}) + \zeta_k \rangle + \alpha_k^2\norm{\nabla f(\theta_{k-1}) + \zeta_k}^2\right)^p \\
&= \underset{\mathclap{\substack{i+j+l=p;\\ i,j,l \in \{0,\ldots p\}}}}{\sum}\frac{p!}{i!j!l!}\delta_{k-1}^{2i}(-2\alpha_k\langle\theta_{k-1}-\thetas, \nabla f(\theta_{k-1})+\zeta_k\rangle)^j\alpha_k^{2l}\norm{\nabla f(\theta_{k-1}) + \zeta_k}^{2l}.
\end{align}
Now we bound each term in the sum above.
\begin{enumerate}
        \item First, for $i=p, j=0, l=0$, the corresponding term in the sum equals $\delta_{k-1}^{2p}$.
        \item Second, for $i=p-1, j=1, l=0$, we obtain, applying \Cref{ass:L-smooth}, that 
        \begin{align}
        2p\alpha_k \PE [\langle\theta_{k-1}-\thetas, \nabla f(\theta_{k-1}) +\zeta_k \rangle\delta_{k-1}^{2(p-1)}|\F_{k-1}] 
        &= 2p\alpha_k\langle\theta_{k-1}-\thetas, \nabla f(\theta_{k-1})- \nabla f(\thetas)  \rangle\delta_{k-1}^{2(p-1)} \\
        &\geq 2 p \mu \alpha_k \delta_{k-1}^{2 p}\eqsp.
        \end{align}
        \item Third, for $l \geq 1$ or $j \geq 2$ (that is,  $2l+j \geq 2$), we use Cauchy-Schwartz inequality  
    \begin{equation}
        |\langle\theta_{k-1}-\thetas, \nabla f(\theta_{k-1})+\zeta_k\rangle)^j| \leq \norm{\theta_{k-1}-\thetas}^j\norm{\nabla f(\theta_{k-1})+\zeta_k}^j\eqsp,
    \end{equation}
    moreover, applying \Cref{ass:L-smooth} and \Cref{ass:noise_decomposition}($2p$) together with the Lyapunov inequality, we get 
    \begin{align}
        \PE[\norm{\nabla f(\theta_{k-1})+\zeta_k}^{2l+j}|\F_{k-1}] 
        &= \PE[\norm{\nabla f(\theta_{k-1})+g(\theta_{k-1}, \xi_k) + \eta(\xi_k)}^{2l+j}|\F_{k-1}] \\
        &\leq 2^{2l+j-1}((L_1+L_2)^{2l+j}\delta_{k-1}^{2l+j} + \sigma_{2p}^{2l+j})\eqsp. 
    \end{align}
\end{enumerate}
    Combining inequalities above, we get 
    \begin{multline}
    \label{eq:delta_k_2p_decomposition}
        \PE[\delta_k^{2p}|\F_{k-1}]\leq \biggl(1 -2p\mu\alpha_k + \underset{\mathclap{\substack{i+j+l=p;\\ i,j,l \in \{0,\ldots p\}: \\j+2l\geq 2 }}} {\sum}\frac{p!}{i!j!l!}\alpha_k^{j+2l}2^{2l+2j-1}(L_1+L_2)^{2l+j}\biggr) \delta_{k-1}^{2p} \\  + \underset{\mathclap{\substack{i+j+l=p;\\ i,j,l \in \{0,\ldots p\}: \\j+2l\geq 2 }}} {\sum}\frac{p!}{i!j!l!}\delta_{k-1}^{2i+j}\alpha_k^{j+2l}2^{2l+2j-1}\sigma_{2p}^{2l+j}\eqsp.
    \end{multline}
    Consider the first term above, and note that
    \begin{align}
    \label{eq:bound_sum_j_2l_p_moment}
    &\qquad \qquad\qquad \qquad \qquad \qquad\underset{\mathclap{\substack{i+j+l=p;\\ i,j,l \in \{0,\ldots p\}: \\j+2l\geq 2 }}} {\sum}\frac{p!}{i!j!l!}\alpha_k^{j+2l}2^{2l+2j-1}(L_1+L_2)^{2l+j} 
     \\ 
    &\leq 2 \alpha_k^2 (L_1+L_2)^2 \biggl(\underset{\mathclap{\substack{i+j+l=p;\\ i,j,l \in \{0,\ldots p\}: \\l\geq 1 }}} {\sum}\frac{p!}{i!j!l!}(4\alpha_k(L_1+L_2))^j(4\alpha_k^2(L_1+L_2)^2)^{l-1} + 4\underset{\mathclap{\substack{i+j+l=p;\\ i,j,l \in \{0,\ldots p\};\\l=0;j\geq2 }}} {\sum}\frac{p!}{i!j!}(4\alpha_k(L_1+L_2))^{j-2}\biggr) \\& \qquad \qquad\qquad \qquad \qquad  \leq 4p^2\alpha_k^2(L_1+L_2)^2
\bigl(1+2\alpha_k(L_1+L_2)\bigr)^{2p-2}
\eqsp.
    \end{align}
    Hence, 
    \begin{align}
     \PE[\delta_k^{2p}] \leq \left(1 -2p\mu\alpha_k + 4 p^2\alpha_k^2 (L_1+L_2)^2(1+2\alpha_k(L_1+L_2))^{2p-2} \right)\PE[\delta_{k-1}^{2p}] + T_1\eqsp,
    \end{align}
    where we have defined 
    \begin{align}
    T_1 &= \underset{\mathclap{\substack{i+j+l=p;\\ i,j,l \in \{0,\ldots p\}; \\j+2l\geq 2}}} {\sum}\frac{p!}{i!j!l!} \PE[\delta_{k-1}^{2i+j}] \alpha_k^{j+2l}2^{2l+2j-1}\sigma_{2p}^{2l+j}\eqsp.
    \end{align}
    For the last term we apply H\"older's inequality together with  induction assumption \eqref{eq:induction_assumption_p_moment} and  $(k+k_0-1)^{-\gamma}\leq 2^{\gamma}(k+k_0)^{-\gamma}$ and obtain
    \begin{align}
    T_1 &\leq  \underset{\mathclap{\substack{i+j+l=p;\\ i,j,l \in \{0,\ldots p\}; \\j+2l\geq 2}}} {\sum}\frac{p!}{i!j!l!}\PE^{1/2}[\delta_{k-1}^{2(i+\lfloor j/2 \rfloor)}]\PE^{1/2}[\delta_{k-1}^{2(i+\lceil j/2 \rceil)}]\alpha_k^{j+2l}2^{2l+2j-1}\sigma_{2p}^{2l+j} \\
    & \leq D_{2(p-1)}\frac{(4c_0^{1/2}2^{\gamma/2}+2^{\gamma}+ 4c_0)^p}{2}c_0^p\sigma_{2p}^{2p}(k+k_0)^{-\gamma (p+1)}\eqsp.
    \end{align}
    Hence, combining the above bounds, we obtain that 
    \begin{equation}
    \label{eq:p_moment_last_iterate_requrence}
   \PE[\delta_k^{2p}]
\leq
\Bigl(
1-2p\mu\alpha_k
+4p^2\alpha_k^2(L_1+L_2)^2
 (1+2\alpha_k(L_1+L_2))^{2p-2}
\Bigr)
\PE[\delta_{k-1}^{2p}]
+
D_{2(p-1)}C_4^pc_0^p\sigma_{2p}^{2p}
(k_0+k)^{-\gamma(p+1)}.
   \end{equation}
    where we have defined 
    \begin{equation}
    \label{eq:C_4_const_def}
    C_4^p = (4c_0^{1/2}2^{\gamma/2}+2^{\gamma}+ 4c_0)^p\eqsp.
    \end{equation}
    Note that 
    \begin{equation}
        1 - 2p\mu\alpha_k + 4 p^2\alpha_k^2 (L_1+L_2)^2(1+2\alpha_k(L_1+L_2))^{2p-2}
> (1-p\mu\alpha_k)^2 \geq 0\eqsp.
    \end{equation}
    Unrolling the recurrence \eqref{eq:p_moment_last_iterate_requrence}, we get 
    \begin{equation}
       \PE[\delta_k^{2p}] \leq A_{2,k}'D_{2(p-1)}C_4^pc_0^p\sigma_{2p}^{2p}, 
    \end{equation}
    where we have set 
    \begin{equation}
    \label{eq:A_2_k'_def}
        A_{2,k}' = \sum_{t=1}^k\prod_{i=t+1}^k(1 - 2p\mu\alpha_i + 4 p^2\alpha_i^2 (L_1+L_2)^2(1+2\alpha_i(L_1+L_2))^{2p-2})(t+k_0)^{-\gamma(p+1)}\eqsp.
    \end{equation}
    For simplicity, we define $C_{5}=4p^2(L_1+L_2)^2$. Let $k_1$ is the largest $k$ such that $\alpha_k^2C_5(1+2\alpha_k(L_1+L_2))^{2p-2} \geq p\mu\alpha_k$. Then, for $i > k_1$, we have
    \begin{equation}
        1 - 2p\mu\alpha_i + C_5 \alpha_i^2 (1+2\alpha_i(L_1+L_2))^{2p-2}\leq 1-p\mu\alpha_i \eqsp.
    \end{equation}

Write
\[
\rho_i=p\mu\alpha_i,\qquad
B_i=C_5\alpha_i^2
    (1+2\alpha_i(L_1+L_2))^{2p-2},
\qquad
H_p=\frac{1}{\gamma(p+1)-1}.
\]
Set $k_1=0$ if the set defining $k_1$ is empty,
and let $\kappa=\min\{k,k_1\}$.
For $1\leq i\leq k_1$, we have $\rho_i\leq B_i$.
For $i>k_1$, since $\mu\leq L_1+L_2$, we have
\[
4\rho_i^2\leq B_i<\rho_i,
\qquad\text{hence}\qquad \rho_i<\frac14.
\]
Using \Cref{lem:bounds_on_sum_step_sizes}\ref{eq:sum_alpha_k_p}, we obtain
\begin{align}
A'_{2,k}
&\leq
\sum_{t=\kappa+1}^{k}
\exp\left\{-\sum_{i=t+1}^{k}\rho_i\right\}
(t+k_0)^{-\gamma(p+1)}
+
H_p\exp\left\{
-\sum_{i=1}^{k}\rho_i+2\sum_{i=1}^{\kappa}B_i
\right\}.
\end{align}
Let $m=\lfloor k/2\rfloor$ and $r=\max\{m,\kappa\}$.
Splitting the first sum at $m$ gives
\begin{align}
A'_{2,k}
&\leq
H_p\exp\left\{-\sum_{i=m+1}^{k}\rho_i\right\}\\
&\quad+
(m+k_0)^{-\gamma p}
\sum_{t=r+1}^{k}
\exp\left\{-\sum_{i=t+1}^{k}\rho_i\right\}
(t+k_0)^{-\gamma}\\
&\quad+
H_p\exp\left\{
-\sum_{i=1}^{k}\rho_i+2\sum_{i=1}^{\kappa}B_i
\right\}.
\end{align}
For every index in the remaining sum, $\rho_t<1/4$,
and therefore $\rho_t\leq2(1-e^{-\rho_t})$.
Consequently,
\begin{align}
&\sum_{t=r+1}^{k}
\exp\left\{-\sum_{i=t+1}^{k}\rho_i\right\}
(t+k_0)^{-\gamma}\\
&\quad\leq
\frac{2}{p\mu c_0}
\sum_{t=r+1}^{k}(1-e^{-\rho_t})
\exp\left\{-\sum_{i=t+1}^{k}\rho_i\right\}\\
&\quad=
\frac{2}{p\mu c_0}
\left(
1-\exp\left\{-\sum_{i=r+1}^{k}\rho_i\right\}
\right)
\leq\frac{2}{p\mu c_0}.
\end{align}
Applying \Cref{lem:bounds_on_sum_step_sizes}\ref{eq:sum_alpha_k_p}, we get 
\[
\sum_{i=1}^k2C_{5}\alpha_i^2(1+2\alpha_i(L_1+L_2))^{2p-2} \leq 2C_5(1+2c_0(L_1+L_2))^{2p-2}\sum_{i=1}^{+\infty}\alpha_i^2\leq \exp\biggl\{5pc_0(L_1+L_2)\biggr\}\frac{2C_5c_0^2}{2\gamma-1}
\]

Applying \Cref{lem:bounds_on_sum_step_sizes}\ref{eq:simple_bound_sum_alpha_k}, we get 
\begin{align}
        &A_{2,k}' \leq \exp\biggl\{-\frac{p\mu c_0}{2(1-\gamma)}((k+k_0)^{1-\gamma}-(\lfloor k/2\rfloor+k_0)^{1-\gamma})\biggr\}\frac{1}{\gamma(p+1)-1} + \frac{2(\lfloor k/2\rfloor +k_0)^{-\gamma p}}{p\mu c_0}\\& + c_3\exp\biggl\{-\frac{p\mu c_0}{2(1-\gamma)}(k+k_0)^{1-\gamma}\biggr\},
\end{align}
where we have set
\begin{equation}
    c_3 = \exp\biggl\{\exp\biggl\{5pc_0(L_1+L_2)\biggr\}\frac{2C_5c_0^2}{2\gamma-1} +\frac{p\mu c_0}{2(1-\gamma)}k_0^{1-\gamma}\biggr \}\frac{1}{\gamma(p+1)-1}.
\end{equation}

Using that $\lfloor k/2 \rfloor \leq k/2$ together with the elementary inequality 
\[
\frac{x^{\beta}}{\beta} - \frac{(x/2)^{\beta}}{\beta} \geq \frac{x^{\beta}}{2}\eqsp,
\]
which is valid for $\beta \in (0,1]$, and $\frac{2}{\mu pc_0 (\lfloor k/2 \rfloor + k_0)^{\gamma p}} \leq \frac{2^{1+\gamma p}}{\mu pc_0(k+k_0)^{\gamma p}}$, we obtain that 
\begin{align}
        A_{2,k}' &\leq \exp\biggl\{-\frac{p\mu c_0}{4}(k+k_0)^{1-\gamma}\biggr\}\exp\biggl\{\frac{p\mu c_0}{2(1-\gamma)}k_0^{1-\gamma}\biggr\}\frac{1}{\gamma(p+1)-1} \\& + \frac{2^{1+\gamma p}}{\mu pc_0 (k+k_0)^{\gamma p}}+ c_3\exp\biggl\{-\frac{p\mu c_0}{2(1-\gamma)}(k+k_0)^{1-\gamma}\biggr\} \\& \leq c_4\exp\biggl\{-\frac{p\mu c_0}{4}(k+k_0)^{1-\gamma}\biggr\} + c_5(k+k_0)^{-\gamma p},
\end{align} 
where we have set 
\begin{align}
    &c_4=
\left(
\exp\left\{
\frac{8c_0^2p^2(L_1+L_2)^2}{2\gamma-1}
\exp\{5pc_0(L_1+L_2)\}
\right\}+1
\right)
\exp\left\{
\frac{p\mu c_0}{1-\gamma}k_0^{1-\gamma}
\right\}
\frac{1}{\min\{\gamma(p+1)-1,2\}}\\&
    c_5 = \frac{2^{1+\gamma p}}{\mu pc_0}
\end{align}
Finally, we get 
\begin{equation}
    \PE[\delta_k^{2p}]\leq C_{2p, 1}'\exp\biggl\{-\frac{p\mu c_0}{4}(k+k_0)^{1-\gamma}\biggr\}\sigma_{2p}^{2p} + C_{2p, 2}'\sigma_{2p}^{2p}\alpha_k^p,
\end{equation}
where 
\begin{align}
    &C_{2p, 1}' = D_{2(p-1)}C_4^pc_0^pc_4\\&
    C_{2p, 2}' = D_{2(p-1)}C_4^pc_5.
\end{align}

To provide the result for an arbitrary starting point
$\theta_0=\theta$, we consider the synchronously coupled
recursions
\begin{align}
\label{eq:def_coupling_recursions}
\theta_k
&=\theta_{k-1}
-\alpha_k\bigl(
\nabla f(\theta_{k-1})
+g(\theta_{k-1},\xi_k)+\eta(\xi_k)
\bigr),
\qquad \theta_0=\theta,
\\
\theta_k'
&=\theta_{k-1}'
-\alpha_k\bigl(
\nabla f(\theta_{k-1}')
+g(\theta_{k-1}',\xi_k)+\eta(\xi_k)
\bigr),
\qquad \theta_0'=\thetas.
\nonumber
\end{align}

For any $k\geq0$, denote
$\delta_k'=\|\theta_k-\theta_k'\|$.
For $k\geq1$, set
\begin{align}
\Delta_{k-1}
&=\theta_{k-1}-\theta_{k-1}',\\
V_k
&=\nabla f(\theta_{k-1})-\nabla f(\theta_{k-1}')
+g(\theta_{k-1},\xi_k)-g(\theta_{k-1}',\xi_k).
\end{align}
Since the additive noise cancels in the synchronous coupling,
\[
\theta_k-\theta_k'=\Delta_{k-1}-\alpha_kV_k.
\]
Hence, for an integer $p\geq1$, the multinomial formula gives
\begin{align}
(\delta_k')^{2p}
&=\bigl(
(\delta_{k-1}')^2
-2\alpha_k\langle\Delta_{k-1},V_k\rangle
+\alpha_k^2\|V_k\|^2
\bigr)^p\\
&=\sum_{\substack{i+j+l=p\\i,j,l\geq0}}
\frac{p!}{i!j!l!}
(\delta_{k-1}')^{2i}
\bigl(-2\alpha_k\langle\Delta_{k-1},V_k\rangle\bigr)^j
\alpha_k^{2l}\|V_k\|^{2l}.
\end{align}

Both $\theta_{k-1}$ and $\theta_{k-1}'$ are
$\mathcal F_{k-1}$-measurable, whereas $\xi_k$ is independent
of $\mathcal F_{k-1}$. Therefore, by
\Cref{ass:L-smooth} and
\Cref{ass:noise_decomposition}($2p$),
the conditional Minkowski and Lyapunov inequalities imply
that, for every $1\leq r\leq2p$,
\[
\left(
\PE[\|V_k\|^r\mid\mathcal F_{k-1}]
\right)^{1/r}
\leq (L_1+L_2)\delta_{k-1}'.
\]
Moreover, the centering condition on $g$ yields
\[
\PE[V_k\mid\mathcal F_{k-1}]
=
\nabla f(\theta_{k-1})-\nabla f(\theta_{k-1}').
\]

We now take conditional expectations in the multinomial
expansion and bound its terms.

\begin{enumerate}
\item
For $i=p$ and $j=l=0$, the corresponding term equals
$(\delta_{k-1}')^{2p}$.

\item
For $i=p-1$, $j=1$, and $l=0$, strong convexity gives
\begin{align}
&-2p\alpha_k(\delta_{k-1}')^{2p-2}
\PE[
\langle\Delta_{k-1},V_k\rangle
\mid\mathcal F_{k-1}
]
\\
&=
-2p\alpha_k(\delta_{k-1}')^{2p-2}
\left\langle
\Delta_{k-1},
\nabla f(\theta_{k-1})-\nabla f(\theta_{k-1}')
\right\rangle
\\
&\leq
-2p\mu\alpha_k(\delta_{k-1}')^{2p}.
\end{align}

\item
For the remaining terms, $j+2l\geq2$.
By the Cauchy--Schwarz inequality and the conditional
moment bound above,
\begin{align}
&\left|
\PE\left[
(\delta_{k-1}')^{2i}
\langle\Delta_{k-1},V_k\rangle^j
\|V_k\|^{2l}
\mid\mathcal F_{k-1}
\right]
\right|
\\
&\leq
(\delta_{k-1}')^{2i+j}
\PE[\|V_k\|^{j+2l}\mid\mathcal F_{k-1}]
\\
&\leq
(L_1+L_2)^{j+2l}
(\delta_{k-1}')^{2i+2j+2l}
\\
&=
(L_1+L_2)^{j+2l}
(\delta_{k-1}')^{2p},
\end{align}
where we used $j+2l\leq2p$ and $i+j+l=p$.
\end{enumerate}

Combining inequalities above, we obtain 

\begin{equation}
\label{eq:coupling_recurence}
    \PE[\delta_{k}'^{2p}|\F_{k-1}] \leq (1- 2p\mu\alpha_k+ \underset{\mathclap{\substack{i+j+l=p;\\ i,j,l \in \{0,\ldots p\};\\j+2l\geq 2}}}{\sum}\frac{p!}{i!j!l!}2^j\alpha_k^{j+2l}(L_1+L_2)^{j+2l})\delta_{k-1}'^{2p}
\end{equation}

Similar to \eqref{eq:bound_sum_j_2l_p_moment}, we have 
\begin{equation}
\sum_{\substack{i+j+l=p\\i,j,l\geq0\\j+2l\geq2}}
\frac{p!}{i!j!l!}\,
2^j\alpha_k^{j+2l}(L_1+L_2)^{j+2l}
\leq
2p^2\alpha_k^2(L_1+L_2)^2
(1+\alpha_k(L_1+L_2))^{2p-2}.
\end{equation}
Enrolling recurrence \eqref{eq:coupling_recurence}, we get
\begin{align}
    \PE[\delta_{k}'^{2p}] &\leq \exp\biggl\{-2p\mu\sum_{i=1}^k\alpha_i\biggr\}\exp\biggl\{ 2p^2(L_1+L_2)^2\sum_{i=1}^k\alpha_i^{2}(1+\alpha_i(L_1+L_2))^{2p-2}\biggr\}\norm{\theta_0-\thetas}^{2p}\\ &\leq c_6\exp\biggl\{-\frac{p\mu c_0}{1-\gamma}(k+k_0)^{1-\gamma}\biggr\}\norm{\theta_0-\thetas}^{2p}\eqsp,
\end{align}
where we have set
\begin{equation}
    c_6 = \exp\biggl\{\exp\biggl\{ 3pc_0(L_1+L_2)\biggr\}\frac{2p^2c_0^2(L_1+L_2)^2)}{2\gamma-1} + \frac{p\mu c_0}{1-\gamma}k_0^{1-\gamma}\biggr\}\eqsp.
\end{equation}
It remains to note that
\begin{align}
    \PE[\norm{\theta_k-\thetas}^{2p}]&\leq 2^{2p-1} \PE[\norm{\theta'_k-\thetas}^{2p}] + 2^{2p-1}\PE[\norm{\theta_k-\theta_k'}^{2p}] \\&\leq C_{2p,1} \exp\biggl\{-\frac{p\mu c_0}{4}(k+k_0)^{1-\gamma}\biggr\}(\norm{\theta_0-\thetas}^{2p} + \sigma_{2p}^{2p}) + C_{2p,2}\sigma_{2p}^{2p}\alpha_k^p\eqsp.
\end{align}
\end{proof}

For validity of induction in \Cref{prop:2p-moment-bound}, we need the following corollary.

\begin{corollary}
\label{cor:p_moment_bound}
Under the assumptions of \Cref{prop:2p-moment-bound},  it holds that 
\begin{equation}
\PE[\norm{\theta_k-\thetas}^{2p}]\leq D_{2p} (\norm{\theta_0-\thetas}^{2p} + \sigma_{2p}^{2p}) \alpha_k^p\eqsp,
\end{equation}
where 
\begin{equation}
\label{eq:const_D_2p_def}
D_{2p} = C_{2p,1} (1/c_0^p + C_{2p, 2}) \biggl(\frac{4\gamma}{(1-\gamma)\mu  c_0 \rme}\biggr)^{\gamma p/(1-\gamma)}\eqsp.
\end{equation}
\end{corollary}
\begin{proof}
Define $C_5 = (\frac{4\gamma}{(1-\gamma)\mu c_0 \rme})^{\gamma p/(1-\gamma)}$, then  $\exp\{-\mu pc_0(k+k_0)^{1-\gamma} / 4\}\leq C_5 (k+k_0)^{-p\gamma}$.

By the definition of $C_5$,
\[
C_5\geq
k_0^{\gamma p}
\exp\left\{-\frac{p\mu c_0}{4}k_0^{1-\gamma}\right\}.
\]
Together with the definition of $c_4$, this yields
\[
c_4C_5
\geq
\frac{2k_0^{\gamma p}}{\min\{\gamma(p+1)-1,2\}}
\exp\left\{
p\mu c_0k_0^{1-\gamma}
\left(\frac{1}{1-\gamma}-\frac14\right)
\right\}
\geq1.
\]
Consequently,
\[
C_{2p,1}C_5\geq2^{2p-1}>1,
\]
which proves the claimed bound.

To verify the monotonicity required in the induction,
note that $C_4\geq2^\gamma$ and hence
\[
D_{2p}
\geq
2^{2p-1}D_{2(p-1)}C_4^pc_4C_5
\geq
2^{(2+\gamma)p-1}D_{2(p-1)}
>
D_{2(p-1)}.
\]
\end{proof}

\begin{corollary}
\label{cor:fourth_moment_bound_last_iterate}
Assume \Cref{ass:L-smooth}, \Cref{ass:hessian_Lipschitz_ball}, \Cref{ass:noise_decomposition}($4$) and \Cref{ass:step_size}. Then for any $k \in \nset$ it holds that
\begin{equation}
\PE[\norm{\theta_k - \thetas}^{4}] \leq C_{4,1}\exp\biggl\{-\frac{2\mu c_0}{4}k^{1-\gamma}\biggr\}(\norm{\theta_0-\thetas}^{4} + \sigma_{4}^{4}) + C_{4,2}\sigma_{4}^{4}\alpha_k^2\eqsp, 
\end{equation}
with 
\begin{equation}
    C_{4,1} = 2^{3}\biggl(C_1(1/c_0+C_2)\biggl(\frac{4\gamma}{(1-\gamma)\mu c_0 \rme}\biggr)^{\gamma/(1-\gamma)}\bigl(4c_0^{1/2}2^{\gamma/2}+2^{\gamma}+ 4c_0\bigr)^2c_0^2+1\biggr)c_{2,4}
\end{equation}
and 
\begin{equation}
    C_{4,2} =2^{3}C_1(1/c_0+C_2)\biggl(\frac{4\gamma}{(1-\gamma)\mu c_0 \rme}\biggr)^{\gamma/(1-\gamma)}\bigl(4c_0^{1/2}2^{\gamma/2}+2^{\gamma}+ 4c_0\bigr)^2c_{2,5}.
\end{equation}
Here $C_1$ and $C_2$ are defined in \Cref{lem:bound_last_iter_second_moment} and 
\begin{align}
    &c_{2,4} = \biggl(\exp\biggl\{\exp\biggl\{10c_0(L_1+L_2)\biggr\}\frac{32c_0^2(L_1+L_2)^2}{2\gamma-1}\biggr\} + 1\biggr)\exp\biggl\{\frac{2\mu c_0}{1-\gamma}k_0^{1-\gamma}\biggr\}\frac{1}{3\gamma-1}\eqsp, \\&
    c_{2,5} = \frac{2^{1+2\gamma}}{2\mu c_0}\eqsp.
\end{align}
\end{corollary}
\begin{proof}
The proof follows directly from \Cref{prop:2p-moment-bound} and \Cref{cor:second_moment_bound}.
\end{proof}


\subsection{High probability bounds on the last iterate}
\label{sec:high_prob_last_iterate}
In this section, we establish a high-probability bound for the last iterate, which is instrumental in controlling the non-linear statistic $D^{\boot}$. Our analysis adapts the approach of \citep[Theorem 9]{madden2024high}, which relies on the assumption that the noise variables $\zeta_k$ are sub-Gaussian. This result, in turn, generalizes the previous results of \citep{harvey2019tight}, where the authors assumed that both the additive noise component $\eta(\xi)$ and state-dependent component $g(\theta,\xi)$ are uniformly bounded. 

\begin{lemma}
\label{lem: high_prob_last_iter_boot}
Assume \Cref{ass:L-smooth}, \Cref{ass:bound_noise},\Cref{ass:bound_bootstap_weights}, \Cref{ass:step_size_new_boot}. For every fixed $k\ge0$ and every $\delta\in(0,1)$, it holds with probability at least $1-\delta$ that, 
\begin{equation}
    \norm{\theta_k^b - \thetas}^2 \leq \alpha_k K_1\log\biggl(\frac{\rme }{\delta}\biggr)\eqsp,
\end{equation}
where 
\begin{equation}
    \label{eq:def_K_1}
    K_1 = \frac{k_0^\gamma}{c_0}\norm{\theta_0-\thetas}^2 +\frac{16W_{\max}^2C_\xi^2}{\mu W_{\min}}(2d+1)
\end{equation}
\end{lemma}
\begin{proof}
Unrolling the recurrence \eqref{eq:sgd_bootstrap}, we have
\begin{align}
    \norm{\theta_k^b-\thetas}^2 &= \norm{\theta_{k-1}^b-\thetas}^2 - 2\alpha_kw_k\langle  F(\theta_{k-1}^b, \xi_k), \theta_{k-1}^b-\thetas \rangle + \alpha_k^2w_k^2\norm{F(\theta_{k-1}^b, \xi_k)}^2 \\& \leq
    \norm{\theta_{k-1}^b-\thetas}^2 - 2\alpha_kw_k\langle   \nabla f(\theta_{k-1}^b), \theta_{k-1}^b-\thetas \rangle - 2\alpha_kw_k\langle g(\theta_{k-1}^b, \xi_k) + \eta(\xi_k), \theta_{k-1}^b -\thetas \rangle \\& \qquad \qquad +  \alpha_k^2w_k^2\norm{ \nabla f(\theta_{k-1}^b) + g(\theta_{k-1}^b, \xi_k) + \eta(\xi_k)}^2  \eqsp.
\end{align}
Using \Cref{ass:L-smooth}, \Cref{ass:bound_noise}, and \Cref{ass:bound_bootstap_weights}, we obtain 
\begin{align}
     \norm{\theta_k^b-\thetas}^2 &\leq (1-2\alpha_k\mu W_{\min} + 2\alpha_k^2L_1^2W_{\max}^2)\norm{\theta_{k-1}^b-\thetas}^2 \\& \qquad \qquad- 2\alpha_kw_k\langle g(\theta_{k-1}^b, \xi_k)+ \eta(\xi_k), \theta_{k-1}^b-\thetas \rangle  + 2\alpha_k^2w_k^2\norm{g(\theta_{k-1}^b, \xi_k) + \eta(\xi_k)}^2\\&\leq (1-2\alpha_k\mu W_{\min} + 2\alpha_k^2L_1^2W_{\max}^2)\norm{\theta_{k-1}^b-\thetas}^2 \\& \qquad \qquad- 2\alpha_kw_k\langle g(\theta_{k-1}^b, \xi_k)+ \eta(\xi_k), \theta_{k-1}^b-\thetas \rangle  + 2\alpha_k^2W_{\max}^2\norm{g(\theta_{k-1}^b, \xi_k) + \eta(\xi_k)}^2
\end{align}
 Using  \Cref{ass:step_size_new_boot}, we have
\begin{multline}
\label{eq:last_iter_reccurrence_high_prob}
    \norm{\theta_k^b-\thetas}^2 \leq
    (1-\mu\alpha_kW_{\min})\norm{\theta_{k-1}^b-\thetas}^2 - 2\alpha_kw_k\langle g(\theta_{k-1}^b,\xi_k)+\eta(\xi_k), \theta_{k-1}^b-\thetas \rangle \\
     + 2\alpha_k^2W_{\max}^2\norm{g(\theta_{k-1}^b, \xi_k) + \eta(\xi_k)}^2\eqsp.
\end{multline}
Now we introduce the quantities
\begin{align}
X_k = \alpha_k^{-1}\norm{\theta_k^b-\thetas}^2\eqsp, \quad Y_{k-1} = -2w_k\langle g(\theta_{k-1}^b,\xi_k)+\eta(\xi_k), \theta_{k-1}^b-\thetas \rangle\eqsp, \quad Z_{k-1} = 2\alpha_kW_{\max}^2\norm{g(\theta_{k-1}^b, \xi_k) + \eta(\xi_k)}^2\eqsp.
\end{align}
Using \eqref{eq:last_iter_reccurrence_high_prob}, we obtain 
\begin{equation}
\label{eq:recurrence_x_y_z}
X_k \leq\alpha_k^{-1} \alpha_{k-1}(1-\mu W_{\min}\alpha_k)X_{k-1} + Y_{k-1} + Z_{k-1}\eqsp.
\end{equation}
Note that 
\begin{align}
\label{eq:bound_contracting_multiplier_last_iter}
     \frac{\alpha_{k-1}}{\alpha_k}(1-\mu W_{\min}\alpha_k) &= \biggl(\frac{k_0+k}{k_0+k-1}\biggr)^{\gamma} -\frac{\mu W_{\min}c_0}{(k_0+k-1)^{\gamma}}\\&\leq 1 + \frac{c_0(\gamma/c_0)}{k_0+k-1}-\frac{\mu W_{\min}c_0}{(k_0+k-1)^{\gamma}}\\& = 1 - \alpha_{k-1}\biggr( \mu W_{\min} -\frac{(\gamma/c_0)}{(k_0+k-1)^{1-\gamma}}\biggl)\eqsp.
\end{align}
Since $k_0\geq \biggl(\frac{2\gamma}{c_0\mu W_{\min}}\biggr)^{1/(1-\gamma)}$, recurrence \eqref{eq:recurrence_x_y_z} and the above identities yield
\begin{equation}
X_k \leq (1 - \mu W_{\min} \alpha_{k-1} / 2) X_{k-1} + Y_{k-1} + Z_{k-1}\eqsp.
\end{equation}
Using \Cref{ass:bound_noise} and \Cref{ass:bound_bootstap_weights}, we have for any $\lambda\in\rset$
\begin{align}
\PE[\exp\{\lambda Y_{k-1}\}|\widetilde{\F}_{k-1}] = \PE[\exp\{-2\lambda w_k \langle\theta_{k-1}^b-\thetas, g(\theta_{k-1}^b,\xi_k)+\eta(\xi_k)\rangle\}|\widetilde{\F}_{k-1}]\leq \exp\{\lambda^2\beta_{k-1}^2X_{k-1}/2\}\eqsp,
\end{align}
where we set $\beta_{k-1}^2 = 4W_{\max}^2C_{\xi}^2\alpha_{k-1}$. Using \Cref{ass:bound_noise} and \Cref{lem:sub_gauss_prop}, we get for any $\lambda \in [0, r_{k-1}^{-1}]$
\begin{align}
\PE[\exp\{\lambda Z_{k-1}\}|\widetilde{\F}_{k-1}] = \PE[\exp\{2 \lambda  \alpha_k W_{\max}^2\norm{g(\theta_{k-1}^b, \xi_k) + \eta(\xi_k)}^2\}|\widetilde{\F}_{k-1}] \leq \exp\{\lambda r_{k-1}\}\eqsp,
\end{align}
where $r_{k-1} = 8d\alpha_{k}W_{\max}^2C_{\xi}^2$ and $\widetilde{\F}_{k-1}$ is defined in \eqref{eq: extended filtration}. Then, applying \citep[Theorem 9]{madden2024high}, for any $k \geq 0$ it holds with probability at least $1-\delta$ for $\delta \in (0,1)$ that 
\begin{equation}
X_k \leq P_k \log\biggl(\frac{\rme}{\delta}\biggr), 
\end{equation}
where $\{P_\ell\}_{\ell \in \nset}$ is any sequence of positive real numbers, satisfying
\begin{equation}
\label{eq:full_rec_for_P_i}
P_{i+1}^2 \geq ((1-\frac{\mu W_{\min}}{2}\alpha_{i})P_i + 2r_{i})P_{i+1} + \beta_i^2P_i\eqsp, \qquad P_0\geq X_0 \eqsp.
\end{equation}
Note in particular, that the sequence $\{P_\ell\}_{\ell \in \nset}$ given by the recurrence
\begin{equation}
P_{i+1} = (1-\frac{\mu W_{\min}}{2}\alpha_{i})P_i + \tau_i\eqsp, \quad \text{ where } \tau_i = 2r_i + \frac{\beta_i^2}{1-\frac{\mu W_{\min}}{2}\alpha_{i}}\eqsp, \text{ and } P_0 = X_0\eqsp,
\end{equation}
satisfies \eqref{eq:full_rec_for_P_i}. Hence, unraveling the recursion, we have
\begin{equation}
P_{k+1} = \prod_{i=0}^k(1-\frac{\mu W_{\min}}{2}\alpha_{i})P_0 +\sum_{i=0}^k\prod_{j=i+1}^k(1-\frac{\mu W_{\min}}{2}\alpha_{j}) \tau_i\eqsp.
\end{equation}
Using \Cref{lem:summ_alpha_k} and $\alpha_{i+1} \leq \alpha_i$, we have
\begin{equation}
P_{k+1} \leq \frac{k_0^\gamma}{c_0}\norm{\theta_0-\thetas}^2 + \frac{32dW_{\max}^2C_\xi^2}{\mu W_{\min}} +\frac{8W_{\max}^2C_\xi^2}{(1-\frac{\mu W_{\min} c_0}{2k_0^\gamma})\mu W_{\min}}\eqsp.
\end{equation}
To complete the proof, it remains to apply the bound on $k_0$ given by \Cref{ass:step_size_new_boot}.
\end{proof}

\begin{corollary}
\label{cor: last_iter_boot_p_moment}
Under the assumptions of \Cref{lem: high_prob_last_iter_boot}
for every fixed $ k\in\{1,\ldots n\}$ and any $p\geq 2$ it holds 
\begin{equation}
     \PE^{2/p}[\norm{\theta_k^b - \thetas}^{p}] \leq p \alpha_k(\rme )^{2/p} K_1/2 \eqsp,
\end{equation}
where $K_1$ is defined in \eqref{eq:def_K_1}.
\end{corollary}
\begin{proof}
Note that from \Cref{lem: high_prob_last_iter_boot} for $\forall k\in\{1,\ldots n\}$ and for any $t\geq 0$ it holds 
    \begin{equation}
    \PP\bigl[\norm{\theta_k^b - \thetas}^2 \geq t\bigr]\leq f(t)\eqsp,
\end{equation}
where 
\[
f(t) = \rme \exp\biggl\{-\frac{t}{K_1\alpha_k}\biggr\}\eqsp.
\]
Then, we have 
\begin{align}
    \PE[\norm{\theta_k^b - \thetas}^{p}] &= \int_{0}^{+\infty}\PP[\norm{\theta_k^b - \thetas}^{p} > u]\rmd u \leq\int_{0}^{+\infty}\rme \exp\biggl\{-\frac{u^{2/p}}{K_1\alpha_k}\biggr\}\rmd u \\&=\rme (p/2)\biggl(K_1\alpha_k\biggr)^{p/2}\int_{0}^{+\infty}\rme^{-x}x^{p/2-1}\rmd x \leq\rme \biggl((p/2)K_1\alpha_k\biggr)^{p/2}\eqsp,
\end{align}
where in the last inequality we use that $\Gamma(p/2) \leq (p/2)^{p/2 -1}$ (see \citep[Theorem 1.5]{AndersonGammaFunc}).
\end{proof}

\begin{lemma}
\label{lem: high_prob_last_iter}
Assume \Cref{ass:L-smooth}, \Cref{ass:bound_noise}, \Cref{ass:step_size_new_boot}. Then for every fixed $k$ and any $\delta \in (0,1)$ with probability at least $1-\delta$  it holds 
\begin{equation}
    \norm{\theta_k - \thetas}^2 \leq \alpha_k K_2 log\biggl(\frac{\rme }{\delta}\biggr)\eqsp, 
\end{equation}
where 
\begin{equation}
    \label{eq:def_K_2}
    K_2 =  \frac{k_0^\gamma}{c_0}\norm{\theta_0-\thetas}^2 +\frac{16C_\xi^2}{\mu}(2d+1)
\end{equation}
Moreover, it holds 
for every fixed  $k\in\{1,\ldots n\}$ and any $p\geq 2$ that 
\begin{equation}
     \PE^{2/p}[\norm{\theta_k - \thetas}^{p}]\leq p\alpha_k (\rme )^{2/p}K_2/2\eqsp.
\end{equation}
\end{lemma}
\begin{proof}
    The proof is similar to the proof of \Cref{lem: high_prob_last_iter_boot} and \Cref{cor: last_iter_boot_p_moment}. 
\end{proof}

\section{Technical bounds}
\label{sec:append_technical}
We begin this section with useful technical bounds on sums of coefficients
\[
\sum_{i=m}^{k}\alpha_{i}^{p}\eqsp,
\]
where the step sizes $\alpha_i$ have a form 
\[
\alpha_i = \frac{c_0}{(k_0+i)^{\gamma}}\eqsp, \quad 1/2 < \gamma < 1\eqsp, \quad k_0 \geq 1\eqsp.
\]
We also bound other related quantities, which are instrumental to our further analysis. 

\begin{lemma}
\label{lem:bounds_on_sum_step_sizes}
    Assume \Cref{ass:step_size}. Then 
    \begin{enumerate}[(a)]
        \item \label{eq:sum_alpha_k_p} for any $p\geq 2$, it holds that
        \begin{equation}
        \sum_{i=1}^{k}\alpha_i^p \leq\frac{c_0^p}{p\gamma-1}\eqsp,
        \end{equation}
        \item \label{eq:simple_bound_sum_alpha_k}
        for any $m\in\{0, \ldots, k\}$, it holds that 
        \begin{equation}
        \sum_{i=m+1}^k\alpha_i \geq  \frac{c_0}{2(1-\gamma)}((k+k_0)^{1-\gamma}-(m+k_0)^{1-\gamma})\eqsp,
        \end{equation}
    \end{enumerate}
\end{lemma}
\begin{proof}
    To prove \ref{eq:sum_alpha_k_p}, observe that 
    \begin{equation}
        \sum_{i=1}^{k}\alpha_i^p \leq  c_{0}^p \int_{1}^{+\infty}\frac{\rmd x}{x^{p\gamma}} \leq \frac{c_0^p}{p\gamma-1}\eqsp,
    \end{equation}
    To prove \ref{eq:simple_bound_sum_alpha_k}, note that for any $i\geq 1$ and $k_0 \geq 1$, we have $2(i+k_0)^{-\gamma} \geq (i+k_0-1)^{-\gamma}$. Hence,
    \begin{equation}
        \sum_{i=m+1}^k\alpha_i \geq \frac{1}{2}\sum_{i=m}^{k-1}\alpha_i \geq \frac{c_0}{2}\int_{m+k_0}^{k+k_0}\frac{\rmd x}{x^{\gamma}} = 
        \frac{c_0}{2(1-\gamma)}((k+k_0)^{1-\gamma}-(m+k_0)^{1-\gamma})\eqsp.
    \end{equation}
\end{proof}

\begin{lemma}[Lemma~24 in \citep{durmus2021stability}]
\label{lem:summ_alpha_k}
Let $b > 0$ and $\{\alpha_k\}_{k \geq 0}$ be a non-increasing sequence such that $\alpha_0 \leq 1/b$. Then
\[
\sum_{j=0}^{k} \alpha_j \prod_{l=j+1}^{k} (1 - \alpha_l b) = \frac{1}{b} \left\{1  - \prod_{l=0}^{k} (1 - \alpha_l b) \right\}
\]
\end{lemma}
\begin{proof}
Proof of this statement is given in \citep{durmus2021stability}.
\end{proof}

\begin{lemma}
\label{lem:bound_sum_exponent}
    For any $A >0$, any $0 \leq i \leq n-1$,  and any $\gamma \in (1/2, 1)$ it holds
   \begin{equation}
        \sum_{j=i}^{n-1}\exp\biggl\{-A(j^{1-\gamma} - i^{1-\gamma})\biggr\} \leq
        \begin{cases}
            1 + \exp\bigl\{\frac{1}{1-\gamma}\bigr\}\frac{1}{A^{1/(1-\gamma)}(1-\gamma)}\Gamma(\frac{1}{1-\gamma})\eqsp, &\text{ if } Ai^{1-\gamma} \leq \frac{1}{1-\gamma} \text{ and } i \geq 1\eqsp;\\
            1 + \frac{1}{A(1-\gamma)^2}i^\gamma\eqsp,  &\text{ if } Ai^{1-\gamma} >\frac{1}{1-\gamma} \text{ and } i \geq 1\eqsp;\\
            1 + \frac{1}{A^{1/(1-\gamma)}(1-\gamma)}\Gamma(\frac{1}{1-\gamma})\eqsp, &\text{ if } i=0 \eqsp. 
        \end{cases}
    \end{equation}
\end{lemma}
\begin{proof}
Note that 
\begin{align}
\sum_{j=i}^{n-1}\exp\biggl\{-A(j^{1-\gamma} - i^{1-\gamma})\biggr\} 
&\leq 1 +\exp\biggl\{A i^{1-\gamma}\biggr\}\int_{i}^{+\infty}\exp\biggl\{-Ax^{1-\gamma} \biggr\}\rmd x \\
&= 1 + \exp\biggl\{A i^{1-\gamma}\biggr\}\frac{1}{A^{1/(1-\gamma)}(1-\gamma)}\int_{Ai^{1-\gamma}}^{+\infty}\rme^{-u} u^{\frac{1}{1-\gamma}-1}\rmd u
\end{align}
      Applying \citep[Theorem 4.4.3]{gabcke2015neue}, we get 
    \begin{equation}
    \int_{Ai^{1-\gamma}}^{+\infty}\rme^{-u} u^{\frac{1}{1-\gamma}-1} \rmd u\leq
        \begin{cases}
            \Gamma(\frac{1}{1-\gamma})\eqsp, &\text{ if } Ai^{1-\gamma} < \frac{1}{1-\gamma};\\
            \frac{1}{1-\gamma}\exp\{-Ai^{1-\gamma}\} A^{\gamma/(1-\gamma)}i^\gamma\eqsp, &\text{ otherwise.}
        \end{cases}
    \end{equation}
    Combining inequities above, for $i \geq 1$ we obtain 
    \begin{equation}
        \sum_{j=i}^{n-1}\exp\biggl\{-A(j^{1-\gamma} - i^{1-\gamma})\biggr\} \leq
        \begin{cases}
            1 + \exp\bigl\{\frac{1}{1-\gamma}\bigr\}\frac{1}{A^{1/(1-\gamma)}(1-\gamma)}\Gamma(\frac{1}{1-\gamma})\eqsp, &\text{ if } Ai^{1-\gamma} < \frac{1}{1-\gamma};\\
            1 + \frac{1}{A(1-\gamma)^2}i^\gamma\eqsp, &\text{ otherwise.}
        \end{cases}
        \eqsp,
    \end{equation}
    and for $i=0$, we have
    \begin{equation}
        \sum_{j=0}^{n-1}\exp\biggl\{-A j^{1-\gamma} \biggr\} \leq 1 + \frac{1}{A^{1/(1-\gamma)}(1-\gamma)}\Gamma\biggl(\frac{1}{1-\gamma}\biggr)\eqsp.
    \end{equation}
\end{proof}

\subsection{Proof of \Cref{lem:bound_Q_i_and_Sigma_n}}
\label{subsec:proof_bound_Q_i_Sigma_n}
\textbf{Version of \Cref{lem:bound_Q_i_and_Sigma_n} with explicit constants.}
\textit{
Assume \Cref{ass:L-smooth}, \Cref{ass:step_size} and suppose that $\lambda_{\min}(\Sigma_\xi)>0$. Then for any $i \in \{0, \ldots, n-1\}$ it holds that
\begin{equation}
\lambda_{\max}(Q_i) \leq C_Q\eqsp,
\end{equation}
where the constant $C_Q$ is given by \begin{equation}
\label{eq:const_C_Q}
    C_Q =\biggl[1+\max\biggl(\exp\biggl\{\frac{1}{1-\gamma}\biggr\}\biggl(\frac{2(1-\gamma)}{\mu c_0}\biggr)^{1/(1-\gamma)}\frac{1}{1-\gamma}\Gamma(\frac{1}{1-\gamma}), \frac{2}{\mu c_0(1-\gamma)}\biggr)\biggr]c_0\eqsp.
\end{equation}
Moreover,
\begin{equation}
\label{eq:Q_i_lower_bound}
\lambda_{\min}(Q_i) \geq C_Q^{\min}\eqsp, \text{ and } \norm{\Sigma_n^{-1/2}} \leq C_{\Sigma}\eqsp,
\end{equation}
where the matrix $\Sigma_n$ is defined in \eqref{eq:sigma_n_def}, and 
\begin{equation}
\label{def:C_Q_min}
    C_Q^{\min} = \frac{1}{L_1}(1-(1-\alpha_i L_1)^{n-i})\eqsp,
\end{equation}
\begin{equation}
\label{eq:def_C_Sigma}
C_{\Sigma} = \frac{\sqrt{3}L_1}{(1-\exp\{-\frac{ c_0L_1}{2(k_0+1)}\})\sqrt{\lambda_{\min}(\Sigma_{\xi})}}\eqsp.
\end{equation}
}
\begin{proof}
Note that using \Cref{lem:bounds_on_sum_step_sizes}\ref{eq:simple_bound_sum_alpha_k}, for $i \geq 0$, it holds that 
\begin{align}
\lambda_{\max}(Q_i) 
&\leq \alpha_i\sum_{j=i}^{n-1}\prod_{k=i+1}^{j}(1-\alpha_k \mu) \leq \alpha_i\sum_{j=i}^{n-1}\exp\biggl\{-\mu\sum_{k=i+1}^j\alpha_k\biggr\} \\
&\leq \alpha_i\sum_{j=i+k_0}^{n-1+k_0}\exp\biggl\{-\frac{\mu c_0}{2(1-\gamma)}(j^{1-\gamma}-(i+k_0)^{1-\gamma})\biggr\}\eqsp.
\end{align} 
Using \Cref{lem:bound_sum_exponent}, we complete the first part with the constant $C_Q$ defined in \eqref{eq:const_C_Q}. In order to prove \eqref{eq:Q_i_lower_bound}, we note that 
\begin{equation}
     \lambda_{\min}(Q_i) \geq \alpha_i\sum_{j=i}^{n-1}(1-\alpha_i L_1)^{j-i} =\frac{1}{L_1}(1-(1-\alpha_i L_1)^{n-i})\eqsp.
\end{equation}
Then, for $1 \leq i \leq \lfloor n/2\rfloor$, we have
\begin{equation}
\begin{aligned}
    \lambda_{\min}(Q_i)
    &\geq \frac{1}{L_1}
        \bigl(1-(1-\alpha_i L_1)^{n/2}\bigr) \\
    &\geq \frac{1}{L_1}
        \bigl(1-\exp\{-\alpha_i L_1 n/2\}\bigr) \\
    &\geq \frac{1}{L_1}
        \left(1-\exp\left\{-\frac{c_0L_1}{2(k_0+1)}\right\}\right)
        \eqsp,
\end{aligned}
\end{equation}
where the last inequality follows from
\[
    \alpha_i n \geq \alpha_n n
    \geq \frac{c_0 n}{k_0+n}
    \geq \frac{c_0}{1+k_0}
    \eqsp.
\]
Combining the previous inequalities, we get
\begin{equation}
\begin{aligned}
    \lambda_{\min}(\Sigma_n)
    &\geq \lambda_{\min}\left(
        \frac{1}{n}\sum_{i=1}^{\lfloor n/2\rfloor}
        Q_i\Sigma_{\xi}Q_i^\top
    \right) \\
    &\geq
    \frac{\lfloor n/2\rfloor}{n}
    \frac{\lambda_{\min}(\Sigma_{\xi})}{L_1^2}
    \left(1-\exp\left\{-\frac{c_0L_1}{2(k_0+1)}\right\}\right)^2 \\
    &\geq
    \frac{\lambda_{\min}(\Sigma_{\xi})}{3L_1^2}
    \left(1-\exp\left\{-\frac{c_0L_1}{2(k_0+1)}\right\}\right)^2
    \eqsp,
\end{aligned}
\end{equation}
where we used $\lfloor n/2\rfloor/n \geq 1/3$ for every integer
$n \geq 2$. Thus, \eqref{eq:Q_i_lower_bound} follows.
\end{proof}

\begin{lemma}
\label{lem:H_theta_bound}
    Assume \Cref{ass:L-smooth} and \Cref{ass:hessian_Lipschitz_ball}. Then
    \begin{equation}
        \norm{H(\theta)}\leq L_{H} \norm{\theta-\thetas}^2,
    \end{equation}
    where $L_{H}= \max(L_3,2L_1/\beta)$. Moreover, for every $\theta,\theta'\in\mathbb R^d$,
\[
\|H(\theta)-H(\theta')\|
\leq
\frac{L_H}{2}
\bigl(\|\theta-\theta^\star\|+\|\theta'-\theta^\star\|\bigr)
\|\theta-\theta'\|.
\]
\end{lemma}
\begin{proof}
    Using \Cref{ass:hessian_Lipschitz_ball} and the definition of $H(\theta)$ in \eqref{eq:H_theta_def}, we get
    \begin{equation}
        \| H(\theta) \| \mathds{1}(\norm{\theta-\thetas}\leq \beta)\leq L_3\norm{\theta-\thetas}^2.
    \end{equation}
   Since $\mu \Id \preccurlyeq \nabla^2f(\theta) \preccurlyeq L_1 \Id$, we also obtain  
    \begin{equation}
        \norm{H(\theta)} \mathds{1}(\norm{\theta-\thetas} > \beta) \leq 2L_1\mathds{1}(\norm{\theta-\thetas} > \beta)\norm{\theta-\thetas} \leq \frac{2L_1}{\beta}\norm{\theta-\thetas}^2.
    \end{equation}
Furthermore, \Cref{ass:L-smooth} and \Cref{ass:hessian_Lipschitz_ball} imply that, for every $\theta\in\mathbb R^d$,
\[
\|\nabla^2 f(\theta)-G\|
\leq L_H\|\theta-\theta^\star\|.
\]
Indeed, inside the ball of radius $\beta$ this follows
from \Cref{ass:hessian_Lipschitz_ball}, whereas outside this ball,
\[
\|\nabla^2 f(\theta)-G\|
\leq 2L_1
\leq \frac{2L_1}{\beta}\|\theta-\theta^\star\|.
\]
Integrating along the segment from $\theta$ to $\theta'$, we obtain
\begin{align}
\|H(\theta)-H(\theta')\|
&\leq
\|\theta-\theta'\|
\int_0^1
\|\nabla^2f(\theta'+t(\theta-\theta'))-G\|\,dt\\
&\leq
L_H\|\theta-\theta'\|
\int_0^1
\bigl((1-t)\|\theta'-\theta^\star\|
+t\|\theta-\theta^\star\|\bigr)\,dt\\
&=
\frac{L_H}{2}
\bigl(\|\theta-\theta^\star\|+\|\theta'-\theta^\star\|\bigr)
\|\theta-\theta'\|.
\end{align}
This concludes the proof.
\end{proof}

\subsection{Properties of sub-Gaussian random vectors}
In this section we derive some auxiliary results on sub-Gaussian random vectors. Following \citep{vershynin:2018} and \citep{jin2019short}, we rely on the following definition.

\begin{definition}
A random vector $X \in \rset^{d}$ with $\PE[X] = 0$ is called sub-Gaussian with variance proxy $\sigma^2 > 0$, if for any vector $v \in \rset^{d}$,
\[
\PE[\exp\{\langle v, X\rangle \}] \leq \exp\{\norm{v}^2\sigma^2/2\}\eqsp.
\]
In this case, we write $X \in \SG(\sigma^2)$.
\end{definition}

\begin{lemma}
\label{lem:sub_gauss_prop}
Let $X \in \SG(\sigma^2)$ be a $d$-dimensional sub-Gaussian vector. Then for any  $\lambda \in [0, 1/C_X]$,
\begin{equation}
\label{eq:subgaus_equiv}
\PE[\exp\{\lambda \norm{X}^2\}] \leq \exp\{\lambda C_X\}\eqsp,
\end{equation}
where $C_X = 4 d \sigma^2$.
\end{lemma}
\begin{proof}
    Let $Y \sim \mathcal{N}(0, \Id_d)$ be a random vector independent of $X$. Then for fixed $x \in \rset^{d}$, we have 
    \begin{equation}
        \exp\{\lambda \norm{x}^2\} = \PE[\exp\{\langle x, \sqrt{2\lambda }Y\rangle\}]\eqsp.
    \end{equation}
    Hence, we have for $\lambda \in[0, 1/(2\sigma^2))$:
    \begin{equation}
    \PE[\exp\{\lambda \norm{X}^2\}] = \PE[\exp\{\langle X, \sqrt{2\lambda}Y\rangle\}] \leq \PE[\exp\{\lambda \sigma^2\norm{Y}^2\}] \leq (1-2\lambda \sigma^2)^{-d/2}\eqsp.
    \end{equation}
    Note that
    \begin{equation}
        -1/2\log (1-a) \leq a \text{ for } a \in[0, 1/2] \eqsp.
    \end{equation}
    Then, we have for $\lambda \in [0, 1/(4\sigma^2)]$:
    \begin{equation}
        \PE[\exp\{\lambda \norm{X}^2\}] \leq \exp\{2d\lambda \sigma^2\}\eqsp.
    \end{equation}
    Hence, \eqref{eq:subgaus_equiv} holds with $C_X = 4 d \sigma^2$.
\end{proof}

\subsection{Gaussian comparison lemma}
There are quite a few works devoted to the comparison of Gaussian measures with different covariance matrices and means. Among others, we note \citep{BarUly86}, \citep{Bernolli2019}, and \citep{Devroye2018}. We use the result from \citep[Theorem 1.1]{Devroye2018}, which provides a comparison in terms of the total variation distance. Recall that the total variation distance between probability measures $\mu$ and $\nu$ on a measurable space $(\Xset,\Xsigma)$ is defined as
\[
\mathsf{d}_{\mathsf{TV}}(\mu, \nu) = \sup_{B \in \Xsigma} |\mu(B) - \nu(B)|\eqsp.
\]
With a slight abuse of notation, when $X$ and $Y$ are random vectors with distributions $\mu$ and $\nu$, respectively, we write $\mathsf{d}_{\mathsf{TV}}(X,Y)$ instead of $\mathsf{d}_{\mathsf{TV}}(\mu,\nu)$. The following lemma holds:

\begin{lemma}
\label{Pinsker}
Let $\Sigma_1$ and $\Sigma_2$ be positive definite covariance matrices in $\rset^{d \times d}$. Let $X \sim \mathcal{N}(0, \Sigma_1)$ and $Y \sim \mathcal{N}(0, \Sigma_2)$. Then
\begin{equation}
\mathsf{d}_{\mathsf{TV}}(X, Y) \le \frac{3}{2} \|\Sigma_2^{-1/2} \Sigma_1 \Sigma_2^{-1/2} - \Id_d \|_{\mathsf{F}} \eqsp.
\end{equation}
\end{lemma}
Recall that our primary aim in this paper is to obtain convergence bounds in the convex distance
\[
\kolmogorov(X, Y) = \sup_{B \in \Conv(\rset^{d})}\left|\P\bigl(X \in B\bigr) - \P(Y \in B)\right|\eqsp,
\]
where $\Conv(\rset^{d})$ is a collection of convex sets on $\rset^{d}$. We can immediately obtain the result for convex distance from \Cref{Pinsker}, since
\[
\kolmogorov(X, Y) \leq \mathsf{d}_{\mathsf{TV}}(X, Y)\eqsp. 
\]

\section{Numerical Experiments}
\label{sec:experiments}
To evaluate the behavior of the multiplier bootstrap procedure for constructing confidence sets, we conducted a series of numerical experiments on synthetic linear and logistic regression problems. The overall experimental procedure is identical for both models, and we highlight only the model-specific differences below. Our experimental pipeline follows the main steps outlined in \citep{JMLR:v19:17-370}. Our experiments were conducted  on a single Intel Xeon Gold 6248R CPU (48 cores, 3.0–4.0 GHz), 768 GB RAM, and 240 GB SSD storage, without GPU accelerators. Code to reproduce the experiments is provided in \url{https://github.com/marina-shesha/Gaussian-Approximation-and-Multiplier-Bootstrap-for-Stochastic-Gradient-Descent}.

\subsection{Experimental Setup}

For each of 1024 trajectories, we generate a dataset $(X, y)$ with sample sizes  
$N \in \{10000, 20000, 30000, 40000, 50000\}$ and $d = 5$ features.  
For both models, the true parameter vector is fixed as  
\[
\theta_{\text{true}} = [1.0, 1.0, -0.1, -0.1, -0.1].
\]

For each trajectory we run the SGD algorithm with step sizes
\begin{equation}
\alpha_n = \frac{200}{(20000 + n)^{0.85}}\eqsp,
\end{equation}
and compute the Polyak--Ruppert averaged estimator $\bar{\theta}_n$.

To assess coverage probabilities, we employ a multiplier bootstrap procedure. For each trajectory, we generate $N_{\text{boot}} = 256$ bootstrap trajectories. In each bootstrap run, the step sizes $\alpha_n$ are multiplicatively perturbed by independent random variables drawn from a Beta distribution:
\begin{equation}
\tilde{\alpha}_n = \alpha_n\left(1 + \frac{w_n - \PE[w_n]}{\sqrt{\var{w_n}}}\right)\eqsp, \quad \text{ where } w_n \sim \text{Beta}(0.5, 2)\eqsp.
\end{equation}

For each trajectory length $N$, we construct confidence intervals for the one–dimensional functional given by a random projection of the parameter target vector. Specifically, for each trajectory we draw a unit vector $u \in \mathbb{S}^{d-1}$ (fixed with all trajectories and its bootstrap replicates) and form confidence intervals for the scalar target parameter. The coverage probabilities for this scalar target parameter are then estimated using three approaches:

\begin{enumerate}
    \item \textbf{Empirical Quantiles:} We compute the empirical quantiles of the $N_{\text{boot}}$ bootstrap replicates and check whether the target parameter belongs to the determined interval.
    
    \item \textbf{Standard Deviation--Based Confidence Intervals:} We construct confidence intervals based on the sample standard deviation of the bootstrap replicates and verify whether target parameter belongs to the determined interval.
    
    \item \textbf{Overlapping Batch Mean (OBM) Estimator \citep{meketon1984overlapping,flegal2010batch}:} For the trajectory with $N = 50000$, we estimate the asymptotic variance as follows:
    \begin{equation}
        \label{eq:OBM_estimator}
        \hat{\sigma}^2_\theta(u) = 
        \frac{b_n}{n - b_n + 1} 
        \sum_{t = 0}^{n - b_n} 
        \left( (\bar{\theta}_{b_n,t} - \bar{\theta}_n)^\top u \right)^2,
        \qquad 
        \bar{\theta}_{b_n,t} = \frac{1}{b_n}\sum_{\ell=t}^{t+b_n-1} \theta_\ell.
    \end{equation}
    Using this estimated asymptotic variance, we construct confidence intervals and check whether the target parameter lies within them. The procedure is repeated for several values of the batch size $b_n$, and we select the one that provides the best coverage performance for the $0.95$ confidence interval.
\end{enumerate}

The coverage probabilities are obtained by averaging the indicator of interval inclusion over all 1024 trajectories.

\subsection{Linear Regression}

The feature vectors $X_i \in \mathbb{R}^p$ are sampled uniformly from $[-1, 1]$, and the response is generated according to
\[
Y_i = X_i^\top \theta_{\text{true}} + \epsilon, 
\qquad 
\epsilon \sim \mathcal{N}(0, 0.02^2).
\]
Then the target parameter is equal to $\theta_{\text{true}}$.
The results are summarized in \Cref{table:linreg}.

\begin{table}[h!]
\centering
\caption{Comparison of Different Estimation Methods for Linear Regression Problem}
\label{table:linreg}

\begin{subtable}[t]{\textwidth}
\centering
\caption{Empirical Quantiles}
\begin{tabular}{c|ccc}
\hline
Trajectory length & 0.95 & 0.90 & 0.80 \\
\hline
10000 & 0.964844 & 0.936523 & 0.855469 \\
20000 & 0.958008 & 0.919922 & 0.852539 \\
30000 & 0.961914 & 0.927734 & 0.839844 \\
40000 & 0.963867 & 0.924805 & 0.834961 \\
50000 & 0.957031 & 0.920898 & 0.829102 \\
\hline
\end{tabular}
\end{subtable}

\vspace{1em}

\begin{subtable}[t]{\textwidth}
\centering
\caption{Standard Deviation-Based Confidence Intervals}
\begin{tabular}{c|ccc}
\hline
Trajectory length & 0.95 & 0.90 & 0.80 \\
\hline
10000 & 0.967773 & 0.943359 & 0.859375 \\
20000 & 0.963867 & 0.927734 & 0.852539 \\
30000 & 0.964844 & 0.932617 & 0.844727 \\
40000 & 0.967773 & 0.934570 & 0.837891 \\
50000 & 0.959961 & 0.924805 & 0.825195 \\
\hline
\end{tabular}
\end{subtable}

\vspace{1em}

\begin{subtable}[t]{\textwidth}
\centering
\caption{Overlapping Batch Mean Estimator}
\begin{tabular}{c|ccc}
\hline
Batch size ($b_n$), Trajectory length & 0.95 & 0.90 & 0.80 \\
\hline
1700, 50000 & 0.890625 & 0.818359 & 0.712891 \\
\hline
\end{tabular}
\end{subtable}

\end{table}

\subsection{Logistic Regression}

Here, $X_i \in \mathbb{R}^p$ are sampled uniformly from $[-1, 3]$, and responses follow the distribution
\[
\PP(Y_i = 1) = \sigma(X_i^\top \theta_{\text{true}})\eqsp, \quad \PP(Y_i = -1) = 1 - \PP(Y_i = 1)\eqsp.
\]
where $\sigma(\cdot)$ denotes the sigmoid function.  
We minimize the $L_2$-regularized logistic loss with regularization parameter $\lambda = 10^{-4}$.  
\begin{equation}
    \label{eq:l2_logloss}
    \theta_{\text{star}} = \underset{\theta\in\rset^d}{\arg\min}\biggl \{\PE\biggl[-\log\biggl( \frac{1}{1 + \exp\{-YX^\top\theta\}}\biggr)\biggr] + \lambda\norm{\theta}^2\biggr\}
\end{equation}
Since there is no closed-form solution, we estimate the target parameter $\theta_{\text{star}}$ by running one long SGD trajectory of length $10^6$ and computing Polyak-Ruppert averaged estimator along this trajectory.
The results are summarized in \Cref{table:logreg}.

\begin{table}[h!]
\centering
\caption{Comparison of Different Estimation Methods for Logistic Regression Problem}
\label{table:logreg}

\begin{subtable}[t]{\textwidth}
\centering
\caption{Empirical Quantiles}
\begin{tabular}{c|ccc}
\hline
Trajectory length & 0.95 & 0.90 & 0.80 \\
\hline
10000 & 0.947266 & 0.901367 & 0.806641 \\
20000 & 0.940430 & 0.909180 & 0.800781 \\
30000 & 0.943359 & 0.903320 & 0.798828 \\
40000 & 0.943359 & 0.894531 & 0.814453 \\
50000 & 0.945312 & 0.888672 & 0.785156 \\
\hline
\end{tabular}
\end{subtable}

\vspace{1em}

\begin{subtable}[t]{\textwidth}
\centering
\caption{Standard Deviation-Based Confidence Intervals}
\begin{tabular}{c|ccc}
\hline
Trajectory length & 0.95 & 0.90 & 0.80 \\
\hline
10000 & 0.961914 & 0.915039 & 0.833984 \\
20000 & 0.955078 & 0.921875 & 0.826172 \\
30000 & 0.958984 & 0.916016 & 0.822266 \\
40000 & 0.952148 & 0.906250 & 0.825195 \\
50000 & 0.955078 & 0.913086 & 0.815430 \\
\hline
\end{tabular}
\end{subtable}

\vspace{1em}

\begin{subtable}[t]{\textwidth}
\centering
\caption{Overlapping Batch Mean Estimator}
\begin{tabular}{c|ccc}
\hline
Batch size ($b_n$), Trajectory length & 0.95 & 0.90 & 0.80 \\
\hline
2500, 50000 & 0.920898 & 0.856445 & 0.743164 \\
\hline
\end{tabular}
\end{subtable}
\end{table}
\subsection{Discussion}
Both empirical quantiles and confidence intervals based on standard deviation are variants of the bootstrap procedure, and in our experiments they provide coverage very close to nominal levels across the entire trajectory length. In contrast, the Overlapping Batch Mean estimator systematically underestimates the coverage even when the trajectory length is rather large ($N = 50000$.). Moreover, the procedure itself is very sensitive to the choice of batch size $b_n$, which moderate heuristics for this problem available \citep{flegal2010batch}.

\end{document}